\documentclass[11pt,a4paper]{article} 
\usepackage[hyperref]{acl2021}
\usepackage{times}
\usepackage{latexsym}

\usepackage{amsthm}
\usepackage{amsmath,amssymb}         
\usepackage{stmaryrd}    
\usepackage{float}

\usepackage{xspace} 
\usepackage{color} 
\usepackage{epsfig}    
\graphicspath{{.}{./figures/}} 

\usepackage{macros/tikzfig}
\usepackage{keycommand} 

\usepackage[ruled,vlined]{algorithm2e}

\newkeycommand{\boxmap}[style=small box][1]{\,\tikz{\node[style=\commandkey{style}] (x) {$#1$};\draw(0,-1.25)--(x)--(0,1.25);}\,}
\newkeycommand{\dboxmap}[style=dbox][1]{\,\tikz{\node[style=\commandkey{style}] (x) {$#1$};\draw[boldedge] (0,-1.25)--(x)--(0,1.25);}\,}

\newkeycommand{\wirelabel}[style=white label]{\,\tikz{\node[style=\commandkey{style}]{};}\,\xspace}

\newkeycommand{\pointmap}[style=point][1]{\,\tikz{\node[style=\commandkey{style}] (x) at (0,-0.3) {$#1$};\node [style=none] (3) at (0, 0.9) {};\node [style=none] (3) at (0, -0.9) {};\draw (x)--(0,0.7);}\,}

\newkeycommand{\point}[style=point,estyle=][1]{\,\tikz{\node[style=\commandkey{style}] (x) at (0,-0.05) {$#1$};\draw [\commandkey{estyle}] (x)--(0,1.05);}\,}

\newkeycommand{\copointmap}[style=copoint][1]{\,\tikz{\node[style=\commandkey{style}] (x) at (0,0.3) {$#1$};\draw (x)--(0,-0.7);}\,}

\newkeycommand{\dpoint}[style=point][1]{\,\tikz{\node[style=\commandkey{style},doubled] (x) at (0,-0.3) {$#1$};\draw[boldedge] (x)--(0,0.7);}\,}

\newkeycommand{\kpoint}[style=kpoint,estyle=][1]{\,\tikz{\node[style=\commandkey{style}] (x) at (0,-0.05) {$#1$};\draw [\commandkey{estyle}] (x)--(0,1.05);}\,}

\newkeycommand{\kpointgray}[style=gray kpoint,estyle=][1]{\,\tikz{\node[style=\commandkey{style}] (x) at (0,-0.05) {$#1$};\draw [\commandkey{estyle}] (x)--(0,1.05);}\,}

\newkeycommand{\typedkpoint}[style=kpoint,edgestyle=][2]{%
\,\begin{tikzpicture}
  \begin{pgfonlayer}{nodelayer}
    \node [style=none] (0) at (0, 1) {};
    \node [style=\commandkey{style}] (1) at (0, -0.75) {$#2$};
    \node [style=right label] (2) at (0.25, 0.5) {$#1$};
  \end{pgfonlayer}
  \begin{pgfonlayer}{edgelayer}
    \draw [\commandkey{edgestyle}] (1) to (0.center);
  \end{pgfonlayer}
\end{tikzpicture}\,}

\newkeycommand{\typedkpointdag}[style=kpoint adjoint,edgestyle=][2]{%
\,\begin{tikzpicture}
  \begin{pgfonlayer}{nodelayer}
    \node [style=none] (0) at (0, -1) {};
    \node [style=\commandkey{style}] (1) at (0, 0.75) {$#2$};
    \node [style=right label] (2) at (0.25, -0.5) {$#1$};
  \end{pgfonlayer}
  \begin{pgfonlayer}{edgelayer}
    \draw [\commandkey{edgestyle}] (0.center) to (1);
  \end{pgfonlayer}
\end{tikzpicture}\,}

\newkeycommand{\bistate}[style=kpoint][1]{\,\begin{tikzpicture}
  \begin{pgfonlayer}{nodelayer}
    \node [style=\commandkey{style}, minimum width=1 cm, inner sep=2pt] (0) at (0, -0.25) {$#1$};
    \node [style=none] (1) at (-0.75, 0) {};
    \node [style=none] (2) at (0.75, 0) {};
    \node [style=none] (3) at (-0.75, 0.88) {};
    \node [style=none] (4) at (0.75, 0.88) {};
  \end{pgfonlayer}
  \begin{pgfonlayer}{edgelayer}
    \draw (1.center) to (3.center);
    \draw (2.center) to (4.center);
  \end{pgfonlayer}
\end{tikzpicture}\,}

\newkeycommand{\bistatebig}[style=kpoint][1]{\,\begin{tikzpicture}
  \begin{pgfonlayer}{nodelayer}
    \node [style=\commandkey{style}, minimum width=, minimum width=1.26 cm, inner sep=2pt] (0) at (0, -0.25) {$#1$};
    \node [style=none] (1) at (-1, 0) {};
    \node [style=none] (2) at (1, 0) {};
    \node [style=none] (3) at (-1, 0.88) {};
    \node [style=none] (4) at (1, 0.88) {};
  \end{pgfonlayer}
  \begin{pgfonlayer}{edgelayer}
    \draw (1.center) to (3.center);
    \draw (2.center) to (4.center);
  \end{pgfonlayer}
\end{tikzpicture}\,}

\newkeycommand{\dbistate}[style=dkpoint][1]{\,\begin{tikzpicture}
  \begin{pgfonlayer}{nodelayer}
    \node [style=\commandkey{style}, minimum width=1 cm, inner sep=2pt] (0) at (0, -0.25) {$#1$};
    \node [style=none] (1) at (-0.75, 0) {};
    \node [style=none] (2) at (0.75, 0) {};
    \node [style=none] (3) at (-0.75, 1.0) {};
    \node [style=none] (4) at (0.75, 1.0) {};
  \end{pgfonlayer}
  \begin{pgfonlayer}{edgelayer}
    \draw [boldedge] (1.center) to (3.center);
    \draw [boldedge] (2.center) to (4.center);
  \end{pgfonlayer}
\end{tikzpicture}\,}

\newkeycommand{\bistatebraket}[style1=kpoint,style2=kpointadj][2]{\,\begin{tikzpicture}
  \begin{pgfonlayer}{nodelayer}
    \node [style=\commandkey{style1}, minimum width=1 cm, inner sep=2pt] (0) at (0, -0.75) {$#1$};
    \node [style=none] (1) at (-0.75, -0.5) {};
    \node [style=none] (2) at (0.75, -0.5) {};
    \node [style=none] (3) at (-0.75, 0.5) {};
    \node [style=none] (4) at (0.75, 0.5) {};
    \node [style=\commandkey{style2}, minimum width=1 cm, inner sep=2pt] (5) at (0, 0.75) {$#2$};
  \end{pgfonlayer}
  \begin{pgfonlayer}{edgelayer}
    \draw (1.center) to (3.center);
    \draw (2.center) to (4.center);
  \end{pgfonlayer}
\end{tikzpicture}\,}

\newkeycommand{\weight}[style=dkpoint][1]{\,\begin{tikzpicture}
  \begin{pgfonlayer}{nodelayer}
    \node [style=\commandkey{style}] (0) at (0, -0.5) {$#1$};
    \node [style=upground] (1) at (0, 0.75) {};
  \end{pgfonlayer}
  \begin{pgfonlayer}{edgelayer}
    \draw [style=boldedge] (0) to (1);
  \end{pgfonlayer}
\end{tikzpicture}\,}

\newkeycommand{\dkpoint}[style=dkpoint][1]{\,\tikz{\node[style=\commandkey{style}] (x) at (0,-0.1) {$#1$};\draw[boldedge] (x)--(0,1);}\,}
\newkeycommand{\dkpointadj}[style=dkpointadj][1]{\,\tikz{\node[style=\commandkey{style}] (x) at (0,0.1) {$#1$};\draw[boldedge] (x)--(0,-1);}\,}
\newkeycommand{\dkpointtrans}[style=dkpointtrans][1]{\,\tikz{\node[style=\commandkey{style}] (x) at (0,0.1) {$#1$};\draw[boldedge] (x)--(0,-1);}\,}
\newkeycommand{\dkpointconj}[style=dkpointconj][1]{\,\tikz{\node[style=\commandkey{style}] (x) at (0,-0.1) {$#1$};\draw[boldedge] (x)--(0,1);}\,}

\newkeycommand{\blackkpoint}[style=black kpoint][1]{\,\tikz{\node[style=\commandkey{style}] (x) at (0,-0.1) {$#1$};\draw (x)--(0,1);}\,}
\newkeycommand{\blackkpointadj}[style=black kpointadj][1]{\,\tikz{\node[style=\commandkey{style}] (x) at (0,0.1) {$#1$};\draw (x)--(0,-1);}\,}
\newkeycommand{\blackdkpoint}[style=black dkpoint][1]{\,\tikz{\node[style=\commandkey{style}] (x) at (0,-0.1) {$#1$};\draw[boldedge] (x)--(0,1);}\,}
\newkeycommand{\blackdkpointadj}[style=black dkpointadj][1]{\,\tikz{\node[style=\commandkey{style}] (x) at (0,0.1) {$#1$};\draw[boldedge] (x)--(0,-1);}\,}

\newcommand{\trace}{\,\begin{tikzpicture}
  \begin{pgfonlayer}{nodelayer}
    \node [style=none] (0) at (0, -0.60) {};
    \node [style=upground] (1) at (0, 0.40) {};
  \end{pgfonlayer}
  \begin{pgfonlayer}{edgelayer}
    \draw [style=none] (0.center) to (1);
  \end{pgfonlayer}
\end{tikzpicture}\,}

\newcommand\discard\trace

\newkeycommand{\pointketbra}[style1=copoint,style2=point,estyle=][2]{\,%
\begin{tikzpicture}
  \begin{pgfonlayer}{nodelayer}
    \node [style=none] (0) at (0, -1.75) {};
    \node [style=\commandkey{style1}] (1) at (0, -0.75) {$#1$};
    \node [style=\commandkey{style2}] (2) at (0, 0.75) {$#2$};
    \node [style=none] (3) at (0, 1.75) {};
  \end{pgfonlayer}
  \begin{pgfonlayer}{edgelayer}
    \draw [\commandkey{estyle}] (0.center) to (1);
    \draw [\commandkey{estyle}] (2) to (3.center);
  \end{pgfonlayer}
\end{tikzpicture}\,}

\newkeycommand{\twocopointketbra}[style1=copoint,style2=copoint,style3=point][3]{\,%
\begin{tikzpicture}
  \begin{pgfonlayer}{nodelayer}
    \node [style=none] (0) at (-0.75, -1.75) {};
    \node [style=none] (1) at (0.75, -1.75) {};
    \node [style=\commandkey{style1}] (2) at (-0.75, -0.75) {$#1$};
    \node [style=\commandkey{style2}] (3) at (0.75, -0.75) {$#2$};
    \node [style=\commandkey{style3}] (4) at (0, 0.75) {$#3$};
    \node [style=none] (5) at (0, 1.75) {};
  \end{pgfonlayer}
  \begin{pgfonlayer}{edgelayer}
    \draw (0.center) to (2);
    \draw (1.center) to (3);
    \draw (4) to (5.center);
  \end{pgfonlayer}
\end{tikzpicture}\,}

\newkeycommand{\twopointketbra}[style1=copoint,style2=point,style3=point][3]{\,%
\begin{tikzpicture}
  \begin{pgfonlayer}{nodelayer}
    \node [style=none] (0) at (0, -1.75) {};
    \node [style=none] (1) at (-0.75, 1.75) {};
    \node [style=\commandkey{style1}] (2) at (0, -0.75) {$#1$};
    \node [style=\commandkey{style2}] (3) at (-0.75, 0.75) {$#2$};
    \node [style=\commandkey{style3}] (4) at (0.75, 0.75) {$#3$};
    \node [style=none] (5) at (0.75, 1.75) {};
  \end{pgfonlayer}
  \begin{pgfonlayer}{edgelayer}
    \draw (0.center) to (2);
    \draw (1.center) to (3);
    \draw (4) to (5.center);
  \end{pgfonlayer}
\end{tikzpicture}\,}

\newkeycommand{\pointbraket}[style1=point,style2=copoint,estyle=][2]{\,%
\begin{tikzpicture}
  \begin{pgfonlayer}{nodelayer}
    \node [style=\commandkey{style1}] (0) at (0, -0.625) {$#1$};
    \node [style=\commandkey{style2}] (1) at (0, 0.625) {$#2$};
  \end{pgfonlayer}
  \begin{pgfonlayer}{edgelayer}
    \draw [\commandkey{estyle}] (0) to (1);
  \end{pgfonlayer}
\end{tikzpicture}\,}

\newkeycommand{\kpointketbra}[style1=kpointdag,style2=kpoint,estyle=][2]{\,%
\begin{tikzpicture}
  \begin{pgfonlayer}{nodelayer}
    \node [style=none] (0) at (0, -1.9) {};
    \node [style=\commandkey{style1}] (1) at (0, -0.95) {$#1$};
    \node [style=\commandkey{style2}] (2) at (0, 0.95) {$#2$};
    \node [style=none] (3) at (0, 1.9) {};
  \end{pgfonlayer}
  \begin{pgfonlayer}{edgelayer}
    \draw [\commandkey{estyle}] (0.center) to (1);
    \draw [\commandkey{estyle}] (2) to (3.center);
  \end{pgfonlayer}
\end{tikzpicture}\,}

\newkeycommand{\kpointmap}[style1=kpoint,style2=map][2]{\,%
\begin{tikzpicture}
  \begin{pgfonlayer}{nodelayer}
    \node [style=\commandkey{style1}] (0) at (0, -1.1) {$#1$};
    \node [style=\commandkey{style2}] (1) at (0, 0.5) {$#2$};
    \node [style=none] (2) at (0, 1.75) {};
  \end{pgfonlayer}
  \begin{pgfonlayer}{edgelayer}
    \draw (0) to (1);
    \draw (1) to (2.center);
  \end{pgfonlayer}
\end{tikzpicture}%
\,}

\newkeycommand{\kpointmaptrans}[style1=kpointconj,style2=mapadj][2]{\,%
\begin{tikzpicture}
  \begin{pgfonlayer}{nodelayer}
    \node [style=\commandkey{style1}] (0) at (0, -1.1) {$#1$};
    \node [style=\commandkey{style2}] (1) at (0, 0.5) {$#2$};
    \node [style=none] (2) at (0, 1.75) {};
  \end{pgfonlayer}
  \begin{pgfonlayer}{edgelayer}
    \draw (0) to (1);
    \draw (1) to (2.center);
  \end{pgfonlayer}
\end{tikzpicture}%
\,}

\newkeycommand{\kpointmapadj}[style1=kpointadj,style2=map][2]{\,%
\begin{tikzpicture}
  \begin{pgfonlayer}{nodelayer}
    \node [style=\commandkey{style1}] (0) at (0, 1.1) {$#1$};
    \node [style=\commandkey{style2}] (1) at (0, -0.5) {$#2$};
    \node [style=none] (2) at (0, -1.75) {};
  \end{pgfonlayer}
  \begin{pgfonlayer}{edgelayer}
    \draw (0) to (1);
    \draw (1) to (2.center);
  \end{pgfonlayer}
\end{tikzpicture}%
\,}

\newkeycommand{\dkpointmap}[style1=dkpoint,style2=dmap][2]{\,%
\begin{tikzpicture}
  \begin{pgfonlayer}{nodelayer}
    \node [style=\commandkey{style1}] (0) at (0, -1.2) {$#1$};
    \node [style=\commandkey{style2}] (1) at (0, 0.4) {$#2$};
    \node [style=none] (2) at (0, 1.7) {};
  \end{pgfonlayer}
  \begin{pgfonlayer}{edgelayer}
    \draw[style=boldedge] (0) to (1);
    \draw[style=boldedge] (1) to (2.center);
  \end{pgfonlayer}
\end{tikzpicture}%
\,}

\newkeycommand{\dkpointmapx}[style1=dkpoint,style2=dmap][2]{\,%
\begin{tikzpicture}
  \begin{pgfonlayer}{nodelayer}
    \node [style=\commandkey{style1}] (0) at (0, -1.4) {$#1$};
    \node [style=\commandkey{style2}] (1) at (0, 0.6) {$#2$};
    \node [style=none] (2) at (0, 1.9) {};
  \end{pgfonlayer}
  \begin{pgfonlayer}{edgelayer}
    \draw[style=boldedge] (0) to (1);
    \draw[style=boldedge] (1) to (2.center);
  \end{pgfonlayer}
\end{tikzpicture}%
\,}

\newkeycommand{\boxcopointmapdag}[style1=map,style2=copoint][2]{\,%
\begin{tikzpicture}
  \begin{pgfonlayer}{nodelayer}
    \node [style=none] (0) at (0, -1.75) {};
    \node [style=\commandkey{style1}] (1) at (0, -0.5) {$#1$};
    \node [style=\commandkey{style2}] (2) at (0, 1) {$#2$};
  \end{pgfonlayer}
  \begin{pgfonlayer}{edgelayer}
    \draw (0.center) to (1);
    \draw (1) to (2);
  \end{pgfonlayer}
\end{tikzpicture}%
\,}

\newkeycommand{\kpointdagmap}[style1=map,style2=kpointdag][2]{\,%
\begin{tikzpicture}
  \begin{pgfonlayer}{nodelayer}
    \node [style=none] (0) at (0, -1.75) {};
    \node [style=\commandkey{style1}] (1) at (0, -0.5) {$#1$};
    \node [style=\commandkey{style2}] (2) at (0, 1.1) {$#2$};
  \end{pgfonlayer}
  \begin{pgfonlayer}{edgelayer}
    \draw (0.center) to (1);
    \draw (1) to (2);
  \end{pgfonlayer}
\end{tikzpicture}%
\,}

\newkeycommand{\sandwichmap}[style1=point,style2=map,style3=copoint][3]{\,%
\begin{tikzpicture}
  \begin{pgfonlayer}{nodelayer}
    \node [style=\commandkey{style1}] (0) at (0, -1.50) {$#1$};
    \node [style=\commandkey{style2}] (1) at (0, 0) {$#2$};
    \node [style=\commandkey{style3}] (2) at (0, 1.50) {$#3$};
  \end{pgfonlayer}
  \begin{pgfonlayer}{edgelayer}
    \draw (0) to (1);
    \draw (1) to (2);
  \end{pgfonlayer}
\end{tikzpicture}%
}

\newkeycommand{\kpointsandwichmap}[style1=kpoint,style2=map,style3=kpointdag][3]{\,%
\begin{tikzpicture}
  \begin{pgfonlayer}{nodelayer}
    \node [style=\commandkey{style1}] (0) at (0, -1.5) {$#1$};
    \node [style=\commandkey{style2}] (1) at (0, 0) {$#2$};
    \node [style=\commandkey{style3}] (2) at (0, 1.5) {$#3$};
  \end{pgfonlayer}
  \begin{pgfonlayer}{edgelayer}
    \draw (0) to (1);
    \draw (1) to (2);
  \end{pgfonlayer}
\end{tikzpicture}%
}

\newkeycommand{\sandwichmapconj}[style1=point,style2=mapconj,style3=copoint][3]{\,%
\begin{tikzpicture}
  \begin{pgfonlayer}{nodelayer}
    \node [style=\commandkey{style1}] (0) at (0, -1.50) {$#1$};
    \node [style=\commandkey{style2}] (1) at (0, 0) {$#2$};
    \node [style=\commandkey{style3}] (2) at (0, 1.50) {$#3$};
  \end{pgfonlayer}
  \begin{pgfonlayer}{edgelayer}
    \draw (0) to (1);
    \draw (1) to (2);
  \end{pgfonlayer}
\end{tikzpicture}%
}

\newkeycommand{\sandwichtwo}[style1=point,style2=map,style3=map,style4=copoint][4]{\,%
\begin{tikzpicture}
  \begin{pgfonlayer}{nodelayer}
    \node [style=\commandkey{style2}] (0) at (0, -1) {$#2$};
    \node [style=\commandkey{style3}] (1) at (0, 1) {$#3$};
    \node [style=\commandkey{style1}] (2) at (0, -2.6) {$#1$};
    \node [style=\commandkey{style4}] (3) at (0, 2.6) {$#4$};
  \end{pgfonlayer}
  \begin{pgfonlayer}{edgelayer}
    \draw (2) to (0);
    \draw (0) to (1);
    \draw (1) to (3);
  \end{pgfonlayer}
\end{tikzpicture}\,}

\newkeycommand{\longbraket}[style1=point,style2=copoint][2]{\,%
\begin{tikzpicture}
  \begin{pgfonlayer}{nodelayer}
    \node [style=\commandkey{style1}] (0) at (0, -2.6) {$#1$};
    \node [style=\commandkey{style2}] (1) at (0, 2.6) {$#2$};
  \end{pgfonlayer}
  \begin{pgfonlayer}{edgelayer}
    \draw (0) to (1);
  \end{pgfonlayer}
\end{tikzpicture}\,}

\newkeycommand{\onb}[style=point]{\ensuremath{\{\,\tikz{\node[style=\commandkey{style}] (x) at (0,-0.3) {$j$};\draw (x)--(0,0.7);}\,\}}\xspace}

\newkeycommand{\phase}[style=white phase dot][1]{\,\begin{tikzpicture}
    \begin{pgfonlayer}{nodelayer}
        \node [style=none] (0) at (0, 1) {};
        \node [style=\commandkey{style}] (2) at (0, -0) {$#1$}; 
        \node [style=none] (3) at (0, -1) {};
    \end{pgfonlayer}
    \begin{pgfonlayer}{edgelayer}
        \draw (2) to (0.center);
        \draw (3.center) to (2);
    \end{pgfonlayer}
\end{tikzpicture}\,}

\newkeycommand{\dphase}[style=white phase ddot][1]{\,\begin{tikzpicture}
    \begin{pgfonlayer}{nodelayer}
        \node [style=none] (0) at (0, 1.25) {};
        \node [style=\commandkey{style}] (2) at (0, -0) {$#1$};
        \node [style=none] (3) at (0, -1.25) {};
    \end{pgfonlayer}
    \begin{pgfonlayer}{edgelayer}
        \draw [boldedge] (2) to (0.center);
        \draw [boldedge] (3.center) to (2);
    \end{pgfonlayer}
\end{tikzpicture}\,}

\newkeycommand{\dphasepoint}[style=white phase ddot][1]{\,\begin{tikzpicture}[yshift=-3mm]
  \begin{pgfonlayer}{nodelayer}
    \node [style=none] (0) at (0, 1) {};
    \node [style=\commandkey{style}] (1) at (0, 0) {$#1$};
  \end{pgfonlayer}
  \begin{pgfonlayer}{edgelayer}
    \draw [style=boldedge] (0.center) to (1);
  \end{pgfonlayer}
\end{tikzpicture}\,}

\newkeycommand{\dphasepointgray}[style=gray phase ddot][1]{\,\begin{tikzpicture}[yshift=-3mm]
  \begin{pgfonlayer}{nodelayer}
    \node [style=none] (0) at (0, 1) {};
    \node [style=\commandkey{style}] (1) at (0, 0) {$#1$};
  \end{pgfonlayer}
  \begin{pgfonlayer}{edgelayer}
    \draw [style=boldedge] (0.center) to (1);
  \end{pgfonlayer}
\end{tikzpicture}\,}

\newkeycommand{\dphasecopoint}[style=white phase ddot][1]{\,\begin{tikzpicture}[yshift=3mm,yscale=-1]
  \begin{pgfonlayer}{nodelayer}
    \node [style=none] (0) at (0, 1) {};
    \node [style=\commandkey{style}] (1) at (0, 0) {$#1$};
  \end{pgfonlayer}
  \begin{pgfonlayer}{edgelayer}
    \draw [style=boldedge] (0.center) to (1);
  \end{pgfonlayer}
\end{tikzpicture}\,}

\newkeycommand{\dphasepointsm}[style=white phase ddot][1]{\,\begin{tikzpicture}[yshift=-3mm]
  \begin{pgfonlayer}{nodelayer}
    \node [style=none] (0) at (0, 1) {};
    \node [style=\commandkey{style}] (1) at (0, 0) {\footnotesize\!$#1$\!};
  \end{pgfonlayer}
  \begin{pgfonlayer}{edgelayer}
    \draw [style=boldedge] (0.center) to (1);
  \end{pgfonlayer}
\end{tikzpicture}\,}

\newkeycommand{\phasepoint}[style=white phase dot][1]{\,\begin{tikzpicture}[yshift=-3mm]
  \begin{pgfonlayer}{nodelayer}
    \node [style=none] (0) at (0, 1) {};
    \node [style=\commandkey{style}] (1) at (0, 0) {$#1$};
  \end{pgfonlayer}
  \begin{pgfonlayer}{edgelayer}
    \draw (0.center) to (1);
  \end{pgfonlayer}
\end{tikzpicture}\,}

\newkeycommand{\phasecopoint}[style=white phase dot][1]{\,\begin{tikzpicture}[yshift=3mm]
  \begin{pgfonlayer}{nodelayer}
    \node [style=none] (0) at (0, -1) {};
    \node [style=\commandkey{style}] (1) at (0, 0) {$#1$};
  \end{pgfonlayer}
  \begin{pgfonlayer}{edgelayer}
    \draw (0.center) to (1);
  \end{pgfonlayer}
\end{tikzpicture}\,}

\newkeycommand{\kpointbraket}[style1=kpoint,style2=kpoint adjoint,estyle=][2]{\,%
\begin{tikzpicture}
  \begin{pgfonlayer}{nodelayer}
    \node [style=\commandkey{style1}] (0) at (0, -0.735) {$#1$};
    \node [style=\commandkey{style2}] (1) at (0, 0.735) {$#2$};
  \end{pgfonlayer}
  \begin{pgfonlayer}{edgelayer}
    \draw [\commandkey{estyle}] (0) to (1);
  \end{pgfonlayer}
\end{tikzpicture}\,}

\newkeycommand{\kkpointbraket}[style1=kpoint,style2=kpoint adjoint,estyle=][2]{\,%
\begin{tikzpicture}
  \begin{pgfonlayer}{nodelayer}
    \node [style=\commandkey{style1}] (0) at (0, -0.685) {$#1$};
    \node [style=\commandkey{style2}] (1) at (0, 0.685) {$#2$};
  \end{pgfonlayer}
  \begin{pgfonlayer}{edgelayer}
    \draw [\commandkey{estyle}] (0) to (1);
  \end{pgfonlayer}
\end{tikzpicture}\,}

\newkeycommand{\kpointbraketx}[style1=kpoint,style2=kpoint adjoint,estyle=][2]{\,%
\begin{tikzpicture}
  \begin{pgfonlayer}{nodelayer}
    \node [style=\commandkey{style1}] (0) at (0, -0.885) {$#1$};
    \node [style=\commandkey{style2}] (1) at (0, 0.885) {$#2$};
  \end{pgfonlayer}
  \begin{pgfonlayer}{edgelayer}
    \draw [\commandkey{estyle}] (0) to (1);
  \end{pgfonlayer}
\end{tikzpicture}\,}

\tikzstyle{cdiag}=[matrix of math nodes, row sep=3em, column sep=3em, text height=1.5ex, text depth=0.25ex,inner sep=0.5em]
\tikzstyle{arrow above}=[transform canvas={yshift=0.5ex}]
\tikzstyle{arrow below}=[transform canvas={yshift=-0.5ex}]

\usepackage{tikz}
\usetikzlibrary{decorations.markings}
\usetikzlibrary{shapes.geometric}

\pgfdeclarelayer{edgelayer}
\pgfdeclarelayer{nodelayer}
\pgfsetlayers{edgelayer,nodelayer,main}

\tikzstyle{none}=[inner sep=0pt]
\definecolor{hexcolor0xff0000}{rgb}{1.000,0.000,0.000}
\definecolor{hexcolor0x000000}{rgb}{0.000,0.000,0.000}
\definecolor{hexcolor0x00ff00}{rgb}{0.000,1.000,0.000}
\definecolor{hexcolor0x000000}{rgb}{0.000,0.000,0.000}
\definecolor{hexcolor0xffff00}{rgb}{1.000,1.000,0.000}
\definecolor{hexcolor0xffffff}{rgb}{1.000,1.000,1.000}

\tikzstyle{rn}=[circle,fill=hexcolor0xff0000,draw=hexcolor0x000000,line width=0.8 pt]
\tikzstyle{gn}=[circle,fill=hexcolor0x00ff00,draw=hexcolor0x000000,line width=0.8 pt]
\tikzstyle{yn}=[circle,fill=hexcolor0xffff00,draw=hexcolor0x000000,line width=0.8 pt]
\tikzstyle{wn}=[circle,fill=hexcolor0xffffff,draw=hexcolor0x000000,line width=0.8 pt]
\tikzstyle{wnthick}=[circle,fill=hexcolor0xffffff,draw=hexcolor0x000000,line width=2.500]

\tikzstyle{simple}=[-,draw=hexcolor0x000000,line width=2.000]
\tikzstyle{arrow}=[-,draw=hexcolor0x000000,postaction={decorate},decoration={markings,mark=at position .5 with {\arrow{>}}},line width=2.000]
\tikzstyle{tick}=[-,draw=hexcolor0x000000,postaction={decorate},decoration={markings,mark=at position .5 with {\draw (0,-0.1) -- (0,0.1);}},line width=2.000]
\tikzstyle{halfthickness}=[-,draw=hexcolor0x000000,line width=0.500]
\tikzstyle{thick}=[-,draw=hexcolor0x000000,line width=2.500]
\tikzstyle{thicker}=[-,draw=hexcolor0x000000,line width=4.000]

\tikzstyle{env}=[copoint,regular polygon rotate=0,minimum width=0.2cm, fill=black]

\tikzstyle{probs}=[shape=semicircle,fill=white,draw=black,shape border rotate=180,minimum width=1.2cm]

\tikzstyle{every picture}=[baseline=-0.25em,scale=0.5]
\tikzstyle{dotpic}=[] 
\tikzstyle{diredges}=[every to/.style={diredge}]
\tikzstyle{math matrix}=[matrix of math nodes,left delimiter=(,right delimiter=),inner sep=2pt,column sep=1em,row sep=0.5em,nodes={inner sep=0pt},text height=1.5ex, text depth=0.25ex]

\tikzstyle{inline text}=[text height=1.5ex, text depth=0.25ex,yshift=0.5mm]
\tikzstyle{label}=[font=\footnotesize,text height=1.5ex, text depth=0.25ex,yshift=0.5mm]
\tikzstyle{left label}=[label,anchor=east,xshift=1.5mm]
\tikzstyle{right label}=[label,anchor=west,xshift=-1.5mm]

\tikzstyle{braceedge}=[decorate,decoration={brace,amplitude=2mm,raise=-1mm}]
\tikzstyle{small braceedge}=[decorate,decoration={brace,amplitude=1mm,raise=-1mm}]

\tikzstyle{doubled}=[line width=1.6pt] 
\tikzstyle{boldedge}=[doubled,shorten <=-0.17mm,shorten >=-0.17mm]
\tikzstyle{boldedgegray}=[doubled,gray,shorten <=-0.17mm,shorten >=-0.17mm]

\tikzstyle{semidoubled}=[line width=1.4pt] 
\tikzstyle{semiboldedgegray}=[semidoubled,gray,shorten <=-0.17mm,shorten >=-0.17mm]

\tikzstyle{boldedgedashed}=[very thick,dashed,shorten <=-0.17mm,shorten >=-0.17mm]
\tikzstyle{vboldedgedashed}=[doubled,dashed,shorten <=-0.17mm,shorten >=-0.17mm]
\tikzstyle{left hook arrow}=[left hook-latex]
\tikzstyle{right hook arrow}=[right hook-latex]
\tikzstyle{sembracket}=[line width=0.5pt,shorten <=-0.07mm,shorten >=-0.07mm]

\tikzstyle{causal edge}=[->,thick,gray]
\tikzstyle{causal nondir}=[thick,gray]
\tikzstyle{timeline}=[thick,gray, dashed]

\tikzstyle{cedge}=[<->,thick,gray!70!white]

\tikzstyle{empty diagram}=[draw=gray!40!white,dashed,shape=rectangle,minimum width=1cm,minimum height=1cm]
\tikzstyle{empty diagram small}=[draw=gray!50!white,dashed,shape=rectangle,minimum width=0.6cm,minimum height=0.5cm]

\tikzstyle{dot}=[inner sep=0mm,minimum width=2mm,minimum height=2mm,draw,shape=circle]
\tikzstyle{ddot}=[inner sep=0mm, doubled, minimum width=2.5mm,minimum height=2.5mm,draw,shape=circle]

\tikzstyle{black dot}=[dot,fill=black]
\tikzstyle{white dot}=[dot,fill=white,,text depth=-0.2mm]
\tikzstyle{green dot}=[white dot] 
\tikzstyle{gray dot}=[dot,fill=gray!40!white,,text depth=-0.2mm]
\tikzstyle{red dot}=[gray dot] 

\tikzstyle{black ddot}=[ddot,fill=black]
\tikzstyle{white ddot}=[ddot,fill=white]
\tikzstyle{gray ddot}=[ddot,fill=gray!40!white]

\tikzstyle{gray edge}=[gray!40!white]

\tikzstyle{small dot}=[inner sep=0.5mm,minimum width=0pt,minimum height=0pt,draw,shape=circle]

\tikzstyle{small black dot}=[small dot,fill=black]
\tikzstyle{small white dot}=[small dot,fill=white]
\tikzstyle{small gray dot}=[small dot,fill=gray!40!white]

\tikzstyle{causal dot}=[inner sep=0.4mm,minimum width=0pt,minimum height=0pt,draw=white,shape=circle,fill=gray!40!white]

\tikzstyle{phase dimensions}=[minimum size=5mm,font=\footnotesize,rectangle,rounded corners=2.5mm,inner sep=0.2mm,outer sep=-2mm]
\tikzstyle{dphase dimensions}=[minimum size=5mm,font=\footnotesize,rectangle,rounded corners=2.5mm,inner sep=0.2mm,outer sep=-2mm]

\tikzstyle{white phase dot}=[dot,fill=white,phase dimensions]
\tikzstyle{white phase ddot}=[ddot,fill=white,dphase dimensions]
\tikzstyle{green phase ddot}=[ddot,fill=green,dphase dimensions]

\tikzstyle{white rect ddot}=[draw=black,fill=white,doubled,minimum size=5mm,font=\footnotesize,rectangle,rounded corners=2.5mm,inner sep=0.2mm]
\tikzstyle{gray rect ddot}=[draw=black,fill=gray!40!white,doubled,minimum size=6mm,font=\footnotesize,rectangle,rounded corners=3mm]

\tikzstyle{gray phase dot}=[dot,fill=gray!40!white,phase dimensions]
\tikzstyle{gray phase ddot}=[ddot,fill=gray!40!white,dphase dimensions]
\tikzstyle{red phase ddot}=[ddot,fill=red,dphase dimensions]

\tikzstyle{grey phase dot}=[gray phase dot]
\tikzstyle{grey phase ddot}=[gray phase ddot]

\tikzstyle{small phase dimensions}=[minimum size=4mm,font=\tiny,rectangle,rounded corners=2mm,inner sep=0.2mm,outer sep=-2mm]
\tikzstyle{small dphase dimensions}=[minimum size=4mm,font=\tiny,rectangle,rounded corners=2mm,inner sep=0.2mm,outer sep=-2mm]

\tikzstyle{small gray phase dot}=[dot,fill=gray!40!white,small phase dimensions]
\tikzstyle{small gray phase ddot}=[ddot,fill=gray!40!white,small dphase dimensions]

\tikzstyle{small map}=[draw,shape=rectangle,minimum height=4mm,minimum width=4mm,fill=white]

\tikzstyle{cnot}=[fill=white,shape=circle,inner sep=-1.4pt]

\tikzstyle{asym hadamard}=[fill=white,draw,shape=NEbox,inner sep=0.6mm,font=\footnotesize,minimum height=4mm]
\tikzstyle{asym hadamard conj}=[fill=white,draw,shape=NWbox,inner sep=0.6mm,font=\footnotesize,minimum height=4mm]
\tikzstyle{asym hadamard dag}=[fill=white,draw,shape=SEbox,inner sep=0.6mm,font=\footnotesize,minimum height=4mm]

\tikzstyle{hadamard}=[fill=white,draw,inner sep=0.6mm,font=\footnotesize,minimum height=4mm,minimum width=4mm]
\tikzstyle{small hadamard}=[fill=white,draw,inner sep=0.6mm,minimum height=1.5mm,minimum width=1.5mm]
\tikzstyle{dhadamard}=[hadamard,doubled]
\tikzstyle{small dhadamard}=[small hadamard,doubled]
\tikzstyle{small dhadamard rotate}=[small hadamard,doubled,rotate=45]
\tikzstyle{antipode}=[white dot,inner sep=0.3mm,font=\footnotesize]

\tikzstyle{scalar}=[diamond,draw,inner sep=0.5pt,font=\small]
\tikzstyle{dscalar}=[diamond,doubled, draw,inner sep=0.5pt,font=\small]

\tikzstyle{small box}=[rectangle,inline text,fill=white,draw,minimum height=5mm,yshift=-0.5mm,minimum width=5mm,font=\small]
\tikzstyle{small gray box}=[small box,fill=gray!30]
\tikzstyle{medium box}=[rectangle,inline text,fill=white,draw,minimum height=5mm,yshift=-0.5mm,minimum width=10mm,font=\small]
\tikzstyle{square box}=[small box] 
\tikzstyle{medium gray box}=[small box,fill=gray!30]
\tikzstyle{semilarge box}=[rectangle,inline text,fill=white,draw,minimum height=5mm,yshift=-0.5mm,minimum width=12.5mm,font=\small]
\tikzstyle{large box}=[rectangle,inline text,fill=white,draw,minimum height=5mm,yshift=-0.5mm,minimum width=15mm,font=\small]
\tikzstyle{large gray box}=[small box,fill=gray!30]

\tikzstyle{Bayes box}=[rectangle,fill=black,draw, minimum height=3mm, minimum width=3mm]

\tikzstyle{gray square point}=[small box,fill=gray!50]

\tikzstyle{dphase box white}=[dhadamard]
\tikzstyle{dphase box gray}=[dhadamard,fill=gray!50!white]

\tikzstyle{point}=[regular polygon,regular polygon sides=3,draw,scale=0.75,inner sep=-0.5pt,minimum width=9mm,fill=white,regular polygon rotate=180]
\tikzstyle{copoint}=[regular polygon,regular polygon sides=3,draw,scale=0.75,inner sep=-0.5pt,minimum width=9mm,fill=white]
\tikzstyle{dpoint}=[point,doubled]
\tikzstyle{dcopoint}=[copoint,doubled]

\tikzstyle{wide copoint}=[fill=white,draw,shape=isosceles triangle,shape border rotate=90,isosceles triangle stretches=true,inner sep=0pt,minimum width=1.5cm,minimum height=6.12mm]
\tikzstyle{wide point}=[fill=white,draw,shape=isosceles triangle,shape border rotate=-90,isosceles triangle stretches=true,inner sep=0pt,minimum width=1.5cm,minimum height=6.12mm,yshift=-0.0mm]
\tikzstyle{wide point plus}=[fill=white,draw,shape=isosceles triangle,shape border rotate=-90,isosceles triangle stretches=true,inner sep=0pt,minimum width=1.74cm,minimum height=7mm,yshift=-0.0mm]

\tikzstyle{wide dpoint}=[fill=white,doubled,draw,shape=isosceles triangle,shape border rotate=-90,isosceles triangle stretches=true,inner sep=0pt,minimum width=1.5cm,minimum height=6.12mm,yshift=-0.0mm]
\tikzstyle{wide dcopoint}=[fill=white,doubled,draw,shape=isosceles triangle,shape border rotate=90,isosceles triangle stretches=true,inner sep=0pt,minimum width=1.5cm,minimum height=6.12mm,yshift=-0.0mm]

\tikzstyle{tinypoint}=[regular polygon,regular polygon sides=3,draw,scale=0.55,inner sep=-0.15pt,minimum width=6mm,fill=white,regular polygon rotate=180]

\tikzstyle{white point}=[point]
\tikzstyle{white dpoint}=[dpoint]
\tikzstyle{green point}=[white point] 
\tikzstyle{white copoint}=[copoint]
\tikzstyle{gray point}=[point,fill=gray!40!white]
\tikzstyle{gray dpoint}=[gray point,doubled]
\tikzstyle{red point}=[gray point] 
\tikzstyle{gray copoint}=[copoint,fill=gray!40!white]
\tikzstyle{gray dcopoint}=[gray copoint,doubled]

\tikzstyle{white point guide}=[regular polygon,regular polygon sides=3,font=\scriptsize,draw,scale=0.65,inner sep=-0.5pt,minimum width=9mm,fill=white,regular polygon rotate=180]

\tikzstyle{black point}=[point,fill=black,font=\color{white}]
\tikzstyle{black copoint}=[copoint,fill=black,font=\color{white}]

\tikzstyle{tiny gray point}=[tinypoint,fill=gray!40!white]

\tikzstyle{diredge}=[->]
\tikzstyle{ddiredge}=[<->]
\tikzstyle{rdiredge}=[<-]
\tikzstyle{thickdiredge}=[->, very thick]
\tikzstyle{pointer edge}=[->,very thick,gray]
\tikzstyle{pointer edge part}=[very thick,gray]
\tikzstyle{dashed edge}=[dashed]
\tikzstyle{thick dashed edge}=[very thick,dashed]
\tikzstyle{thick gray dashed edge}=[thick dashed edge,gray!40]
\tikzstyle{thick map edge}=[very thick,|->]

\makeatletter
\newcommand{\boxshape}[3]{%
\pgfdeclareshape{#1}{
\inheritsavedanchors[from=rectangle] 
\inheritanchorborder[from=rectangle]
\inheritanchor[from=rectangle]{center}
\inheritanchor[from=rectangle]{north}
\inheritanchor[from=rectangle]{south}
\inheritanchor[from=rectangle]{west}
\inheritanchor[from=rectangle]{east}
\backgroundpath{
\southwest \pgf@xa=\pgf@x \pgf@ya=\pgf@y
\northeast \pgf@xb=\pgf@x \pgf@yb=\pgf@y

\@tempdima=#2
\@tempdimb=#3

\pgfpathmoveto{\pgfpoint{\pgf@xa - 5pt + \@tempdima}{\pgf@ya}}
\pgfpathlineto{\pgfpoint{\pgf@xa - 5pt - \@tempdima}{\pgf@yb}}
\pgfpathlineto{\pgfpoint{\pgf@xb + 5pt + \@tempdimb}{\pgf@yb}}
\pgfpathlineto{\pgfpoint{\pgf@xb + 5pt - \@tempdimb}{\pgf@ya}}
\pgfpathlineto{\pgfpoint{\pgf@xa - 5pt + \@tempdima}{\pgf@ya}}
\pgfpathclose
}
}}

\boxshape{NEbox}{0pt}{5pt}
\boxshape{SEbox}{0pt}{-5pt}
\boxshape{NWbox}{5pt}{0pt}
\boxshape{SWbox}{-5pt}{0pt}
\boxshape{EBox}{-3pt}{3pt}
\boxshape{WBox}{3pt}{-3pt}
\makeatother

\tikzstyle{cloud}=[shape=cloud,draw,minimum width=1.5cm,minimum height=1.5cm]

\tikzstyle{map}=[draw,shape=NEbox,inner sep=2pt,minimum height=6mm,fill=white]
\tikzstyle{dashedmap}=[draw,dashed,shape=NEbox,inner sep=2pt,minimum height=6mm,fill=white]
\tikzstyle{mapdag}=[draw,shape=SEbox,inner sep=2pt,minimum height=6mm,fill=white]
\tikzstyle{mapadj}=[draw,shape=SEbox,inner sep=2pt,minimum height=6mm,fill=white]
\tikzstyle{maptrans}=[draw,shape=SWbox,inner sep=2pt,minimum height=6mm,fill=white]
\tikzstyle{mapconj}=[draw,shape=NWbox,inner sep=2pt,minimum height=6mm,fill=white]

\tikzstyle{langmap}=[draw,shape=NEbox,inner sep=2pt,minimum height=2.4mm,minimum width=3.2mm,fill=white]
\tikzstyle{langmaptrans}=[draw,shape=SWbox,inner sep=2pt,minimum height=2.4mm,minimum width=3.2mm,fill=white]

\tikzstyle{medium map}=[draw,shape=NEbox,inner sep=2pt,minimum height=6mm,fill=white,minimum width=7mm]
\tikzstyle{medium map dag}=[draw,shape=SEbox,inner sep=2pt,minimum height=6mm,fill=white,minimum width=7mm]
\tikzstyle{medium map adj}=[draw,shape=SEbox,inner sep=2pt,minimum height=6mm,fill=white,minimum width=7mm]
\tikzstyle{medium map trans}=[draw,shape=SWbox,inner sep=2pt,minimum height=6mm,fill=white,minimum width=7mm]
\tikzstyle{medium map conj}=[draw,shape=NWbox,inner sep=2pt,minimum height=6mm,fill=white,minimum width=7mm]
\tikzstyle{semilarge map}=[draw,shape=NEbox,inner sep=2pt,minimum height=6mm,fill=white,minimum width=9.5mm]
\tikzstyle{semilarge map trans}=[draw,shape=SWbox,inner sep=2pt,minimum height=6mm,fill=white,minimum width=9.5mm]
\tikzstyle{semilarge map adj}=[draw,shape=SEbox,inner sep=2pt,minimum height=6mm,fill=white,minimum width=9.5mm]
\tikzstyle{semilarge map dag}=[draw,shape=SEbox,inner sep=2pt,minimum height=6mm,fill=white,minimum width=9.5mm]
\tikzstyle{semilarge map conj}=[draw,shape=NWbox,inner sep=2pt,minimum height=6mm,fill=white,minimum width=9.5mm]
\tikzstyle{large map}=[draw,shape=NEbox,inner sep=2pt,minimum height=6mm,fill=white,minimum width=12mm]
\tikzstyle{large map conj}=[draw,shape=NWbox,inner sep=2pt,minimum height=6mm,fill=white,minimum width=12mm]
\tikzstyle{very large map}=[draw,shape=NEbox,inner sep=2pt,minimum height=6mm,fill=white,minimum width=17mm]

\tikzstyle{medium dmap}=[draw,doubled,shape=NEbox,inner sep=2pt,minimum height=6mm,fill=white,minimum width=7mm]
\tikzstyle{medium dmap dag}=[draw,doubled,shape=SEbox,inner sep=2pt,minimum height=6mm,fill=white,minimum width=7mm]
\tikzstyle{medium dmap adj}=[draw,doubled,shape=SEbox,inner sep=2pt,minimum height=6mm,fill=white,minimum width=7mm]
\tikzstyle{medium dmap trans}=[draw,doubled,shape=SWbox,inner sep=2pt,minimum height=6mm,fill=white,minimum width=7mm]
\tikzstyle{medium dmap conj}=[draw,doubled,shape=NWbox,inner sep=2pt,minimum height=6mm,fill=white,minimum width=7mm]
\tikzstyle{semilarge dmap}=[draw,doubled,shape=NEbox,inner sep=2pt,minimum height=6mm,fill=white,minimum width=9.5mm]
\tikzstyle{semilarge dmap trans}=[draw,doubled,shape=SWbox,inner sep=2pt,minimum height=6mm,fill=white,minimum width=9.5mm]
\tikzstyle{semilarge dmap adj}=[draw,doubled,shape=SEbox,inner sep=2pt,minimum height=6mm,fill=white,minimum width=9.5mm]
\tikzstyle{semilarge dmap dag}=[draw,doubled,shape=SEbox,inner sep=2pt,minimum height=6mm,fill=white,minimum width=9.5mm]
\tikzstyle{semilarge dmap conj}=[draw,doubled,shape=NWbox,inner sep=2pt,minimum height=6mm,fill=white,minimum width=9.5mm]
\tikzstyle{large dmap}=[draw,doubled,shape=NEbox,inner sep=2pt,minimum height=6mm,fill=white,minimum width=12mm]
\tikzstyle{large dmap conj}=[draw,doubled,shape=NWbox,inner sep=2pt,minimum height=6mm,fill=white,minimum width=12mm]
\tikzstyle{large dmap trans}=[draw,doubled,shape=SWbox,inner sep=2pt,minimum height=6mm,fill=white,minimum width=12mm]
\tikzstyle{large dmap adj}=[draw,doubled,shape=SEbox,inner sep=2pt,minimum height=6mm,fill=white,minimum width=12mm]
\tikzstyle{large dmap dag}=[draw,doubled,shape=SEbox,inner sep=2pt,minimum height=6mm,fill=white,minimum width=12mm]
\tikzstyle{very large dmap}=[draw,doubled,shape=NEbox,inner sep=2pt,minimum height=6mm,fill=white,minimum width=19.5mm]

\tikzstyle{muxbox}=[draw,shape=rectangle,minimum height=3mm,minimum width=3mm,fill=white]
\tikzstyle{dmuxbox}=[muxbox,doubled]

\tikzstyle{box}=[draw,shape=rectangle,inner sep=2pt,minimum height=6mm,minimum width=6mm,fill=white]
\tikzstyle{dbox}=[draw,doubled,shape=rectangle,inner sep=2pt,minimum height=6mm,minimum width=6mm,fill=white]
\tikzstyle{dmap}=[draw,doubled,shape=NEbox,inner sep=2pt,minimum height=6mm,fill=white]
\tikzstyle{dmapdag}=[draw,doubled,shape=SEbox,inner sep=2pt,minimum height=6mm,fill=white]
\tikzstyle{dmapadj}=[draw,doubled,shape=SEbox,inner sep=2pt,minimum height=6mm,fill=white]
\tikzstyle{dmaptrans}=[draw,doubled,shape=SWbox,inner sep=2pt,minimum height=6mm,fill=white]
\tikzstyle{dmapconj}=[draw,doubled,shape=NWbox,inner sep=2pt,minimum height=6mm,fill=white]

\tikzstyle{ddmap}=[draw,doubled,dashed,shape=NEbox,inner sep=2pt,minimum height=6mm,fill=white]
\tikzstyle{ddmapdag}=[draw,doubled,dashed,shape=SEbox,inner sep=2pt,minimum height=6mm,fill=white]
\tikzstyle{ddmapadj}=[draw,doubled,dashed,shape=SEbox,inner sep=2pt,minimum height=6mm,fill=white]
\tikzstyle{ddmaptrans}=[draw,doubled,dashed,shape=SWbox,inner sep=2pt,minimum height=6mm,fill=white]
\tikzstyle{ddmapconj}=[draw,doubled,dashed,shape=NWbox,inner sep=2pt,minimum height=6mm,fill=white]

\boxshape{sNEbox}{0pt}{3pt}
\boxshape{sSEbox}{0pt}{-3pt}
\boxshape{sNWbox}{3pt}{0pt}
\boxshape{sSWbox}{-3pt}{0pt}
\tikzstyle{smap}=[draw,shape=sNEbox,fill=white]
\tikzstyle{smapdag}=[draw,shape=sSEbox,fill=white]
\tikzstyle{smapadj}=[draw,shape=sSEbox,fill=white]
\tikzstyle{smaptrans}=[draw,shape=sSWbox,fill=white]
\tikzstyle{smapconj}=[draw,shape=sNWbox,fill=white]

\tikzstyle{dsmap}=[draw,dashed,shape=sNEbox,fill=white]
\tikzstyle{dsmapdag}=[draw,dashed,shape=sSEbox,fill=white]
\tikzstyle{dsmaptrans}=[draw,dashed,shape=sSWbox,fill=white]
\tikzstyle{dsmapconj}=[draw,dashed,shape=sNWbox,fill=white]

\boxshape{mNEbox}{0pt}{10pt}
\boxshape{mSEbox}{0pt}{-10pt}
\boxshape{mNWbox}{10pt}{0pt}
\boxshape{mSWbox}{-10pt}{0pt}
\tikzstyle{mmap}=[draw,shape=mNEbox]
\tikzstyle{mmapdag}=[draw,shape=mSEbox]
\tikzstyle{mmaptrans}=[draw,shape=mSWbox]
\tikzstyle{mmapconj}=[draw,shape=mNWbox]

\tikzstyle{mmapgray}=[draw,fill=gray!40!white,shape=mNEbox]
\tikzstyle{smapgray}=[draw,fill=gray!40!white,shape=sNEbox]

\makeatletter
\pgfdeclareshape{cornerpoint}{
\inheritsavedanchors[from=rectangle] 
\inheritanchorborder[from=rectangle]
\inheritanchor[from=rectangle]{center}
\inheritanchor[from=rectangle]{north}
\inheritanchor[from=rectangle]{south}
\inheritanchor[from=rectangle]{west}
\inheritanchor[from=rectangle]{east}
\backgroundpath{
\southwest \pgf@xa=\pgf@x \pgf@ya=\pgf@y
\northeast \pgf@xb=\pgf@x \pgf@yb=\pgf@y

\pgfmathsetmacro{\pgf@shorten@left}{\pgfkeysvalueof{/tikz/shorten left}}
\pgfmathsetmacro{\pgf@shorten@right}{\pgfkeysvalueof{/tikz/shorten right}}

\pgfpathmoveto{\pgfpoint{0.5 * (\pgf@xa + \pgf@xb)}{\pgf@ya - 5pt}}
\pgfpathlineto{\pgfpoint{\pgf@xa - 8pt + \pgf@shorten@left}{\pgf@yb - 1.5 * \pgf@shorten@left}}
\pgfpathlineto{\pgfpoint{\pgf@xa - 8pt + \pgf@shorten@left}{\pgf@yb}}
\pgfpathlineto{\pgfpoint{\pgf@xb + 8pt - \pgf@shorten@right}{\pgf@yb}}
\pgfpathlineto{\pgfpoint{\pgf@xb + 8pt - \pgf@shorten@right}{\pgf@yb - 1.5 * \pgf@shorten@right}}
\pgfpathclose
}
}

\pgfdeclareshape{cornercopoint}{
\inheritsavedanchors[from=rectangle] 
\inheritanchorborder[from=rectangle]
\inheritanchor[from=rectangle]{center}
\inheritanchor[from=rectangle]{north}
\inheritanchor[from=rectangle]{south}
\inheritanchor[from=rectangle]{west}
\inheritanchor[from=rectangle]{east}
\backgroundpath{
\southwest \pgf@xa=\pgf@x \pgf@ya=\pgf@y
\northeast \pgf@xb=\pgf@x \pgf@yb=\pgf@y

\pgfmathsetmacro{\pgf@shorten@left}{\pgfkeysvalueof{/tikz/shorten left}}
\pgfmathsetmacro{\pgf@shorten@right}{\pgfkeysvalueof{/tikz/shorten right}}

\pgfpathmoveto{\pgfpoint{0.5 * (\pgf@xa + \pgf@xb)}{\pgf@yb + 5pt}}
\pgfpathlineto{\pgfpoint{\pgf@xa - 8pt + \pgf@shorten@left}{\pgf@ya + 1.5 * \pgf@shorten@left}}
\pgfpathlineto{\pgfpoint{\pgf@xa - 8pt + \pgf@shorten@left}{\pgf@ya}}
\pgfpathlineto{\pgfpoint{\pgf@xb + 8pt - \pgf@shorten@right}{\pgf@ya}}
\pgfpathlineto{\pgfpoint{\pgf@xb + 8pt - \pgf@shorten@right}{\pgf@ya + 1.5 * \pgf@shorten@right}}
\pgfpathclose
}
}

\pgfdeclareshape{langpoint}{
\inheritsavedanchors[from=rectangle] 
\inheritanchorborder[from=rectangle]
\inheritanchor[from=rectangle]{center}
\inheritanchor[from=rectangle]{north}
\inheritanchor[from=rectangle]{south}
\inheritanchor[from=rectangle]{west}
\inheritanchor[from=rectangle]{east}
\backgroundpath{
\southwest \pgf@xa=\pgf@x \pgf@ya=\pgf@y
\northeast \pgf@xb=\pgf@x \pgf@yb=\pgf@y

\pgfmathsetmacro{\pgf@shorten@left}{\pgfkeysvalueof{/tikz/shorten left}}
\pgfmathsetmacro{\pgf@shorten@right}{\pgfkeysvalueof{/tikz/shorten right}}

\pgfpathmoveto{\pgfpoint{0.5 * (\pgf@xa + \pgf@xb)}{\pgf@ya - 2pt}}
\pgfpathlineto{\pgfpoint{\pgf@xa - 8pt}{\pgf@yb - 3 * \pgf@shorten@left + 5pt}} 
\pgfpathlineto{\pgfpoint{\pgf@xa - 8pt}{\pgf@yb -1pt}}
\pgfpathlineto{\pgfpoint{\pgf@xb + 8pt}{\pgf@yb -1pt}}
\pgfpathlineto{\pgfpoint{\pgf@xb + 8pt}{\pgf@yb - 3 * \pgf@shorten@left + 5pt}}
\pgfpathclose
}
}

\pgfdeclareshape{langcopoint}{
\inheritsavedanchors[from=rectangle] 
\inheritanchorborder[from=rectangle]
\inheritanchor[from=rectangle]{center}
\inheritanchor[from=rectangle]{north}
\inheritanchor[from=rectangle]{south}
\inheritanchor[from=rectangle]{west}
\inheritanchor[from=rectangle]{east}
\backgroundpath{
\southwest \pgf@xa=\pgf@x \pgf@ya=\pgf@y
\northeast \pgf@xb=\pgf@x \pgf@yb=\pgf@y

\pgfmathsetmacro{\pgf@shorten@left}{\pgfkeysvalueof{/tikz/shorten left}}
\pgfmathsetmacro{\pgf@shorten@right}{\pgfkeysvalueof{/tikz/shorten right}}

\pgfpathmoveto{\pgfpoint{0.5 * (\pgf@xa + \pgf@xb)}{\pgf@yb +0pt}}
\pgfpathlineto{\pgfpoint{\pgf@xa - 8pt}{\pgf@ya + 3 * \pgf@shorten@left - 5pt}} 
\pgfpathlineto{\pgfpoint{\pgf@xa - 8pt}{\pgf@ya + 1pt}}
\pgfpathlineto{\pgfpoint{\pgf@xb + 8pt}{\pgf@ya + 1pt}}
\pgfpathlineto{\pgfpoint{\pgf@xb + 8pt}{\pgf@ya + 3 * \pgf@shorten@left - 5pt}}
\pgfpathclose
}
}

\pgfdeclareshape{langrect}{
\inheritsavedanchors[from=rectangle] 
\inheritanchorborder[from=rectangle]
\inheritanchor[from=rectangle]{center}
\inheritanchor[from=rectangle]{north}
\inheritanchor[from=rectangle]{south}
\inheritanchor[from=rectangle]{west}
\inheritanchor[from=rectangle]{east}
\backgroundpath{
\southwest \pgf@xa=\pgf@x \pgf@ya=\pgf@y
\northeast \pgf@xb=\pgf@x \pgf@yb=\pgf@y

\pgfmathsetmacro{\pgf@shorten@left}{\pgfkeysvalueof{/tikz/shorten left}}
\pgfmathsetmacro{\pgf@shorten@right}{\pgfkeysvalueof{/tikz/shorten right}}

\pgfpathmoveto{\pgfpoint{\pgf@xa - 8pt}{\pgf@ya + 3 * \pgf@shorten@left - 5pt}} 
\pgfpathlineto{\pgfpoint{\pgf@xa - 8pt}{\pgf@ya + 1pt}}
\pgfpathlineto{\pgfpoint{\pgf@xb + 8pt}{\pgf@ya + 1pt}}
\pgfpathlineto{\pgfpoint{\pgf@xb + 8pt}{\pgf@ya + 3 * \pgf@shorten@left - 5pt}}
\pgfpathclose
}
}

\makeatother

\pgfkeyssetvalue{/tikz/shorten left}{0pt}
\pgfkeyssetvalue{/tikz/shorten right}{0pt}

\tikzstyle{kpoint common}=[draw,fill=white,inner sep=1pt,minimum height=4mm]

\tikzstyle{langstate}=[shape=langcopoint,shorten left=5pt,kpoint common,font=\footnotesize]
\tikzstyle{langeffect}=[shape=langpoint,shorten left=5pt,kpoint common,font=\footnotesize]
\tikzstyle{langstatedash}=[shape=langcopoint,dashed, shorten left=5pt,kpoint common,font=\footnotesize]
\tikzstyle{langeffectdash}=[shape=langpoint,dashed, shorten left=5pt,kpoint common,font=\footnotesize]
\tikzstyle{langbox}=[shape=langrect,shorten left=5pt,kpoint common,font=\footnotesize] 

\tikzstyle{kpoint}=[shape=cornerpoint,shorten left=5pt,kpoint common]
\tikzstyle{kpoint adjoint}=[shape=cornercopoint,shorten left=5pt,kpoint common]

\tikzstyle{kpoint conjugate}=[shape=cornerpoint,shorten right=5pt,kpoint common]
\tikzstyle{kpoint transpose}=[shape=cornercopoint,shorten right=5pt,kpoint common]
\tikzstyle{kpoint symm}=[shape=cornerpoint,shorten left=5pt,shorten right=5pt,kpoint common]

\tikzstyle{black kpoint}=[shape=cornerpoint,shorten left=5pt,kpoint common,fill=black,font=\color{white}]
\tikzstyle{black kpoint adjoint}=[shape=cornercopoint,shorten left=5pt,kpoint common,fill=black,font=\color{white}]
\tikzstyle{black kpointadj}=[shape=cornercopoint,shorten left=5pt,kpoint common,fill=black,font=\color{white}]

\tikzstyle{black dkpoint}=[shape=cornerpoint,shorten left=5pt,kpoint common,fill=black, doubled,font=\color{white}]
\tikzstyle{black dkpoint adjoint}=[shape=cornercopoint,shorten left=5pt,kpoint common,fill=black, doubled,font=\color{white}]
\tikzstyle{black dkpointadj}=[shape=cornercopoint,shorten left=5pt,kpoint common,fill=black, doubled,font=\color{white}]

\tikzstyle{kpointdag}=[kpoint adjoint]
\tikzstyle{kpointadj}=[kpoint adjoint]
\tikzstyle{kpointconj}=[kpoint conjugate]
\tikzstyle{kpointtrans}=[kpoint transpose]

\tikzstyle{big kpoint}=[kpoint, minimum width=1.2 cm, minimum height=8mm, inner sep=4pt, text depth=3mm]

\tikzstyle{wide kpoint}=[kpoint, minimum width=1 cm, inner sep=2pt]
\tikzstyle{wide kpointdag}=[kpointdag, minimum width=1 cm, inner sep=2pt]
\tikzstyle{wide kpointconj}=[kpointconj, minimum width=1 cm, inner sep=2pt]
\tikzstyle{wide kpointtrans}=[kpointtrans, minimum width=1 cm, inner sep=2pt]

\tikzstyle{gray kpoint}=[kpoint,fill=gray!50!white]
\tikzstyle{gray kpointdag}=[kpointdag,fill=gray!50!white]
\tikzstyle{gray kpointadj}=[kpointadj,fill=gray!50!white]
\tikzstyle{gray kpointconj}=[kpointconj,fill=gray!50!white]
\tikzstyle{gray kpointtrans}=[kpointtrans,fill=gray!50!white]

\tikzstyle{gray dkpoint}=[kpoint,fill=gray!50!white,doubled]
\tikzstyle{gray dkpointdag}=[kpointdag,fill=gray!50!white,doubled]
\tikzstyle{gray dkpointadj}=[kpointadj,fill=gray!50!white,doubled]
\tikzstyle{gray dkpointconj}=[kpointconj,fill=gray!50!white,doubled]
\tikzstyle{gray dkpointtrans}=[kpointtrans,fill=gray!50!white,doubled]

\tikzstyle{white label}=[draw,fill=white,rectangle,inner sep=0.7 mm]
\tikzstyle{gray label}=[draw,fill=gray!50!white,rectangle,inner sep=0.7 mm]
\tikzstyle{black label}=[draw,fill=black,rectangle,inner sep=0.7 mm]

\tikzstyle{dkpoint}=[kpoint,doubled]
\tikzstyle{wide dkpoint}=[wide kpoint,doubled]
\tikzstyle{dkpointdag}=[kpoint adjoint,doubled]
\tikzstyle{wide dkpointdag}=[wide kpointdag,doubled]
\tikzstyle{dkcopoint}=[kpoint adjoint,doubled]
\tikzstyle{dkpointadj}=[kpoint adjoint,doubled]
\tikzstyle{dkpointconj}=[kpoint conjugate,doubled]
\tikzstyle{dkpointtrans}=[kpoint transpose,doubled]

\tikzstyle{kscalar}=[kpoint common, shape=EBox, inner xsep=-1pt, inner ysep=3pt,font=\small]
\tikzstyle{kscalarconj}=[kpoint common, shape=WBox, inner xsep=-1pt, inner ysep=3pt,font=\small]

 \tikzstyle{upground}=[circuit ee IEC,ground,rotate=90,scale=2.5]
 \tikzstyle{downground}=[circuit ee IEC,ground,rotate=-90,scale=2.5]
 \tikzstyle{bigground}=[regular polygon,regular polygon sides=3,draw=gray,scale=0.50,inner sep=-0.5pt,minimum width=10mm,fill=gray]

\tikzstyle{arrs}=[-latex,font=\small,auto]
\tikzstyle{arrow plain}=[arrs]
\tikzstyle{arrow dashed}=[dashed,arrs]
\tikzstyle{arrow bold}=[very thick,arrs]
\tikzstyle{arrow hide}=[draw=white!0,-]
\tikzstyle{arrow reverse}=[latex-]
\tikzstyle{cdnode}=[]

\definecolor{hexcolor0xa9a9a9}{rgb}{0.663,0.663,0.663} 
\tikzstyle{GrayLine}=[dashed,draw=hexcolor0xa9a9a9] 
\tikzstyle{gray}=[dashed,draw=hexcolor0xa9a9a9]

\theoremstyle{definition}
\newtheorem{theorem}{Theorem}[section]
\newtheorem*{theorem*}{Theorem}

\newtheorem{defn}[theorem]{Definition}

\newtheorem{example*}[theorem]{Example*}
\newtheorem{examples*}[theorem]{Examples*}

\newtheorem{remark*}[theorem]{Remark*}

\def\bR{\begin{color}{red}}  
\def\bB{\begin{color}{blue}}
\def\bM{\begin{color}{magenta}}  
\def\bC{\begin{color}{cyan}}
\def\bW{\begin{color}{white}}
\def\bBl{\begin{color}{black}}
\def\bG{\begin{color}{green}}
\def\bY{\begin{color}{yellow}}
\def\e{\end{color}\xspace}
\newcommand{\bit}{\begin{itemize}}
\newcommand{\eit}{\end{itemize}\par\noindent}
\newcommand{\ben}{\begin{enumerate}}
\newcommand{\een}{\end{enumerate}\par\noindent}
\newcommand{\beq}{\begin{equation}}
\newcommand{\eeq}{\end{equation}\par\noindent}
\newcommand{\beqa}{\begin{eqnarray*}}
\newcommand{\eeqa}{\end{eqnarray*}\par\noindent}
\newcommand{\beqn}{\begin{eqnarray}}
\newcommand{\eeqn}{\end{eqnarray}\par\noindent}

\title{Grammar Equations}

\author{Bob Coecke\\
Oxford-based QNLP-team\\
Cambridge Quantum Computing Ltd.\\
  \texttt{\footnotesize bob.coecke@cambridgequantum.com} \\\And
Vincent Wang\\ 
Department of CS\\
University of Oxford\\ 
  \texttt{\footnotesize vincent.wang@stcatz.ox.ac.uk} \\} 
  
\date{}

\aclfinalcopy

\begin{document}    
\maketitle 

\begin{abstract}   
Diagrammatically speaking, grammatical calculi such as pregroups provide wires between words in order to elucidate their interactions, and this enables one to verify grammatical correctness of phrases and sentences.  In this paper we also provide wirings within words.  This will enable us to identify grammatical constructs that we expect to be either equal or closely related.  Hence, our work paves the way for a new theory of grammar, that provides novel `grammatical truths'.  We give a nogo-theorem for the fact that our wirings for words make no sense for preordered monoids, the form which grammatical calculi usually take.  Instead, they require diagrams -- or equivalently, (free) monoidal categories. 
 \end{abstract}
 
\section{Introduction}

Grammatical calculi \cite{Lambek0, Grishin, Lambek1} enable one to verify grammatical correctness of sentences. However, there are certain grammatical constructs that we expect to be closely related, if not the same, but which grammatical calculi fail to identify. We will focus on pregroups \cite{LambekBook}, but the core ideas of this paper extend well beyond pregroup grammars, including CCGs \cite{steedman1987combinatory}, drawing on the recent work in \cite{DRichie} that casts CCGs as augmented pregroups.
 
In this paper we both modify and extend grammatical calculi, by providing so-called `internal wirings' for a substantial portion of English. Diagrammatically speaking, while grammatical calculi provide wires between words in order to elucidate their interactions, we also provide wirings within words.   
For example, a \em pregroup diagram \em for the phrase:
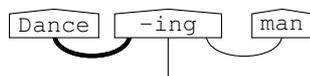
\[
\begin{tikzpicture}
	\begin{pgfonlayer}{nodelayer}
		\node [style=langstate] (0) at (0, 0.25) {\tt -ing};
		\node [style=none] (1) at (0, 0) {};
		\node [style=none] (2) at (0, -1.25) {};
		\node [style=none] (3) at (1, 0) {};
		\node [style=none] (4) at (1, 0) {};
		\node [style=langstate] (5) at (-3, 0.25) {\!\!\!\!{\tt Dance}\!\!\!\!};
		\node [style=none] (6) at (-3, 0) {};
		\node [style=none] (7) at (-1, 0) {};
		\node [style=langstate] (8) at (3, 0.25) {\!\!\!\!{\tt man}\!\!\!\!};
		\node [style=none] (9) at (3, 0) {};
	\end{pgfonlayer}
	\begin{pgfonlayer}{edgelayer}
		\draw [style=boldedge, bend right=90, looseness=1.00] (6.center) to (7.center);
		\draw [style=boldedge, bend right=90, looseness=1.00] (6.center) to (7.center);
		\draw [bend right=90, looseness=1.00] (4.center) to (9.center);
		\draw (1.center) to (2.center);
	\end{pgfonlayer}
\end{tikzpicture}
\]
will become:
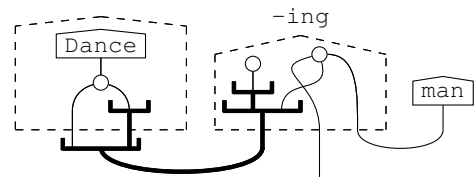
\[
\begin{tikzpicture}
	\begin{pgfonlayer}{nodelayer}
		\node [style=none] (0) at (4, -1) {};
		\node [style=none] (1) at (-0.5, -1) {};
		\node [style=none] (2) at (-0.5, 1) {};
		\node [style=none] (3) at (1.75, 1.5) {};
		\node [style=none] (4) at (4, 1) {};
		\node [style=none] (5) at (1.75, 2) {\small\tt -ing};
		\node [style=none] (6) at (-0.25, -0.25) {};
		\node [style=none] (7) at (-0.25, -0.5) {};
		\node [style=none] (8) at (1.75, -0.5) {};
		\node [style=none] (9) at (-0.25, -0.5) {};
		\node [style=none] (10) at (1.75, -0.25) {};
		\node [style=none] (11) at (1.75, -0.5) {};
		\node [style=none] (12) at (0.75, -0.5) {};
		\node [style=none] (13) at (0.75, -1.5) {};
		\node [style=none] (14) at (0.25, -0.5) {};
		\node [style=none] (15) at (2.25, -1) {};
		\node [style=none] (16) at (0.75, -0.5) {};
		\node [style=none] (17) at (1.25, -0.5) {};
		\node [style=none] (18) at (0.5, 0) {};
		\node [style=none] (19) at (0.5, -0.5) {};
		\node [style=none] (20) at (2.25, -2.25) {};
		\node [style=none] (21) at (0, 0.25) {};
		\node [style=none] (22) at (0, 0) {};
		\node [style=none] (23) at (1, 0) {};
		\node [style=none] (24) at (0, 0) {};
		\node [style=none] (25) at (1, 0.25) {};
		\node [style=none] (26) at (1, 0) {};
		\node [style=none] (27) at (0.5, 0) {};
		\node [style=none] (28) at (0.75, 0) {};
		\node [style=white dot] (29) at (0.5, 0.75) {};
		\node [style=white dot] (30) at (2.25, 1) {};
		\node [style=none] (31) at (3.25, -1) {};
		\node [style=none] (32) at (3.25, -1) {};
		\node [style=none] (33) at (-3.5, 1) {};
		\node [style=none] (34) at (-3.5, 0.25) {};
		\node [style=langstate] (35) at (-3.5, 1.25) {\!\!\!\!{\tt Dance}\!\!\!\!};
		\node [style=white dot] (36) at (-3.5, 0.25) {};
		\node [style=none] (37) at (-4.25, -0.5) {};
		\node [style=none] (38) at (-2.75, -0.5) {};
		\node [style=none] (39) at (-4.25, -1.5) {};
		\node [style=none] (40) at (-2.5, 1) {};
		\node [style=none] (41) at (-1.25, -1) {};
		\node [style=none] (42) at (-5.75, -1) {};
		\node [style=none] (43) at (-5.75, 1.75) {};
		\node [style=none] (44) at (-3.5, 2) {};
		\node [style=none] (45) at (-1.25, 1.75) {};
		\node [style=none] (46) at (-3.25, -0.25) {};
		\node [style=none] (47) at (-3.25, -0.5) {};
		\node [style=none] (48) at (-2.25, -0.5) {};
		\node [style=none] (49) at (-3.25, -0.5) {};
		\node [style=none] (50) at (-2.25, -0.25) {};
		\node [style=none] (51) at (-2.25, -0.5) {};
		\node [style=none] (52) at (-2.75, -0.5) {};
		\node [style=none] (53) at (-2.75, -1.5) {};
		\node [style=none] (54) at (-4.5, -1.25) {};
		\node [style=none] (55) at (-4.5, -1.5) {};
		\node [style=none] (56) at (-2.5, -1.5) {};
		\node [style=none] (57) at (-4.5, -1.5) {};
		\node [style=none] (58) at (-2.5, -1.25) {};
		\node [style=none] (59) at (-2.5, -1.5) {};
		\node [style=none] (60) at (-3.5, -1.5) {};
		\node [style=none] (61) at (0.75, -1.5) {};
		\node [style=langstate] (62) at (5.5, 0) {\!\!\!\!{\tt man}\!\!\!\!};
		\node [style=none] (63) at (5.5, -1) {};
		\node [style=none] (64) at (5.5, -0.25) {};
	\end{pgfonlayer}
	\begin{pgfonlayer}{edgelayer}
		\draw [style=swap, dashed] (2.center) to (1.center);
		\draw [style=swap, dashed] (1.center) to (0.center);
		\draw [style=swap, dashed] (0.center) to (4.center);
		\draw [style=swap, dashed] (3.center) to (2.center);
		\draw [style=swap, dashed] (3.center) to (4.center);
		\draw [style=boldedge] (6.center) to (7.center);
		\draw [style=boldedge] (8.center) to (9.center);
		\draw [style=boldedge] (10.center) to (11.center);
		\draw [style=boldedge] (12.center) to (13.center);
		\draw [style=boldedge] (18.center) to (19.center);
		\draw [style=boldedge] (21.center) to (22.center);
		\draw [style=boldedge] (23.center) to (24.center);
		\draw [style=boldedge] (25.center) to (26.center);
		\draw (27.center) to (29);
		\draw [in=90, out=-165, looseness=1.75] (30) to (15.center);
		\draw [in=90, out=0, looseness=1.00] (30) to (31.center);
		\draw [in=-75, out=90, looseness=1.50] (17.center) to (30);
		\draw [style=swap] (34.center) to (33.center);
		\draw [style=swap] (37.center) to (39.center);
		\draw [style=swap, in=90, out=-15, looseness=1.00] (36) to (38.center);
		\draw [style=swap, in=90, out=-165, looseness=1.00] (36) to (37.center);
		\draw [style=swap, dashed] (43.center) to (42.center);
		\draw [style=swap, dashed] (42.center) to (41.center);
		\draw [style=swap, dashed] (41.center) to (45.center);
		\draw [style=swap, dashed] (44.center) to (43.center);
		\draw [style=swap, dashed] (44.center) to (45.center);
		\draw [style=boldedge] (46.center) to (47.center);
		\draw [style=boldedge] (48.center) to (49.center);
		\draw [style=boldedge] (50.center) to (51.center);
		\draw [style=boldedge] (52.center) to (53.center);
		\draw [style=boldedge] (54.center) to (55.center);
		\draw [style=boldedge] (56.center) to (57.center);
		\draw [style=boldedge] (58.center) to (59.center);
		\draw [style=boldedge, bend right=90, looseness=0.50] (60.center) to (61.center);
		\draw [style=boldedge] (8.center) to (9.center);
		\draw (27.center) to (29);
		\draw [style=boldedge] (12.center) to (13.center);
		\draw [style=swap, dashed] (3.center) to (2.center);
		\draw [style=swap, dashed] (1.center) to (0.center);
		\draw [style=swap, dashed] (0.center) to (4.center);
		\draw [style=boldedge] (23.center) to (24.center);
		\draw [style=boldedge] (48.center) to (49.center);
		\draw [style=swap, dashed] (44.center) to (45.center);
		\draw [style=swap, dashed] (41.center) to (45.center);
		\draw [style=swap, dashed] (43.center) to (42.center);
		\draw [style=swap, dashed] (2.center) to (1.center);
		\draw [style=boldedge] (54.center) to (55.center);
		\draw [style=swap, dashed] (44.center) to (43.center);
		\draw [style=swap, dashed] (3.center) to (4.center);
		\draw [style=boldedge] (50.center) to (51.center);
		\draw [style=boldedge] (58.center) to (59.center);
		\draw [style=swap] (37.center) to (39.center);
		\draw [in=90, out=0, looseness=1.00] (30) to (31.center);
		\draw [style=boldedge] (25.center) to (26.center);
		\draw [style=boldedge] (21.center) to (22.center);
		\draw [style=boldedge] (18.center) to (19.center);
		\draw [style=boldedge] (10.center) to (11.center);
		\draw [style=boldedge] (46.center) to (47.center);
		\draw [style=swap] (34.center) to (33.center);
		\draw [style=boldedge] (52.center) to (53.center);
		\draw [style=boldedge] (56.center) to (57.center);
		\draw [style=swap, in=90, out=-15, looseness=1.00] (36) to (38.center);
		\draw [style=boldedge] (6.center) to (7.center);
		\draw [style=boldedge, bend right=90, looseness=0.50] (60.center) to (61.center);
		\draw [style=swap, dashed] (42.center) to (41.center);
		\draw [style=swap, in=90, out=-165, looseness=1.00] (36) to (37.center);
		\draw (64.center) to (63.center);
		\draw [bend right=90, looseness=0.75] (32.center) to (63.center);
		\draw (15.center) to (20.center);
	\end{pgfonlayer}
\end{tikzpicture}
\]
We show how these additional wirings enable one to identify grammatical constructs that we expect to be closely related. Providing these internal wirings in particular involves decomposing basic types like sentence-types over noun-types, and this decomposition may vary from sentence to sentence. Hence our refinement of grammar-theory also constitutes a departure from some of the practices of traditional grammatical calculi. 

Additional structure for grammatical calculi was previously introduced by providing semantics to certain words, for example, quantifiers within Montague semantics \cite{montague1973proper}. This is not what we do. We strictly stay within the realm of grammar, and grammar only.
Hence, our work paves the way for a new theory of grammar, that provides novel `grammatical truths'.

Usually grammatical calculi take the form of preordered monoids \cite{Gospel}.  However, the internal wirings cannot be defined at the poset level, for which we provide a nogo-theorem. Hence passing to the realm of diagrammatic representations -- which correspond to proper free monoidal categories -- is not just a convenience, but a necessity for this work. They moreover provide a clear insight in the flow of meanings.

Internal wirings were  proposed within the DisCoCat framework \cite{CSC}, for relative pronouns and verbs \cite{FrobMeanI, FrobMeanII, GrefSadr, KartsaklisSadrzadeh2014, CLM, CoeckeText, CoeckeMeich}.  They are made up of `spiders' (a.k.a.~certain Frobenius algebras) \cite{CPV, CKbook}.  We point out a shortcoming of those earlier proposed internal wirings, and fix them by introducing a `wrapping gadget', that forces certain wires to stay together. This  re-introduces composite types such as sentence types.

What we present here is only part of the full story. For the latter we refer to a forthcoming much longer paper \cite{GramCircs}, which besides providing many more internal wirings than given here, also uses them to provide bureaucracy-free grammar as circuits, the equivalence classes for the equations introduced here. These circuits also have direct practical applications within natural language processing -- see e.g.~\cite{QNLP-foundations}.


\section{Statement of the problem} 

For our purposes, a pregroup has a set of `basic types' $n, s, ...$ each of which admit left and right inverses ${}^{-1}n$ and $n^{-1}$. Each grammatical type is assigned a string of these, e.g.~a transitive verb in English gets:
$tv = {}^{-1}n \cdot s \cdot n^{-1}$. The inverses `cancel out' from one direction:
\beq\label{pregroupcanc}
 n \, \cdot  \, {}^{-1}n \ \rightarrow \ 1\qquad\qquad n^{-1} \cdot \, n \ \rightarrow \ 1 
\eeq
A sentence is grammatical if when taking the string of all of its grammatical types, the inverses cancel to leave a special, `final', basic type $s$ (for sentence), like here for $n \cdot tv \cdot n$: 
\beqa
&&n  \, \cdot  \, \left({}^{-1}n \cdot s \cdot n^{-1}\right) \, \cdot\, n \\
&\stackrel{(assoc.)}{\rightarrow} & 
\left(n  \, \cdot   {}^{-1}n\right) \cdot s \cdot \left(n^{-1}  \cdot\, n\right)\\
&\stackrel{\ (\ref{pregroupcanc})\ }{\rightarrow}&
 1 \cdot s\cdot 1 \\
&\stackrel{(unit)}{\rightarrow} &  
s
\eeqa
This calculation can be represented diagrammatically: 
\beq\label{eq:pregroup}
\begin{tikzpicture}[scale=0.80]
	\begin{pgfonlayer}{nodelayer}
		\node [style=none] (0) at (-3.5, 0.75) {};
		\node [style=none] (1) at (-1, 0.75) {};
		\node [style=none] (2) at (0, 0.75) {};
		\node [style=none] (3) at (3.5, 0.75) {};
		\node [style=none] (4) at (1, 0.75) {};
		\node [style=none] (5) at (0, -0.5) {};
		\node [style=langstate] (6) at (-3.5, 1) {\!\!\!$n$\!\!\!};
		\node [style=langstate] (7) at (3.5, 1) {\!\!\!$n$\!\!\!};
		\node [style=langstate] (8) at (0, 1) {\,\,$tv$\,\,};
	\end{pgfonlayer}
	\begin{pgfonlayer}{edgelayer}
		\draw [style=swap, bend right=90, looseness=1.25] (0.center) to (1.center);
		\draw [style=swap, bend right=90, looseness=1.25] (4.center) to (3.center);
		\draw [style=boldedge] (2.center) to (5.center);
	\end{pgfonlayer}
\end{tikzpicture} 
\eeq

Now consider the following examples:
\begin{center}
{\footnotesize\tt Alice likes the flowers that Bob gives Claire}\\
{\footnotesize\tt Bob gives Claire the flowers that Alice likes}   
\end{center}
The pregroup diagrams now look as in Figure (\ref{flowers}). 
\begin{figure*}
\centering
\begin{tikzpicture}
	\begin{pgfonlayer}{nodelayer}
		\node [style=none] (0) at (-12, 0) {};
		\node [style=none] (1) at (-8.75, 0) {};
		\node [style=none] (2) at (12, 0) {};
		\node [style=none] (3) at (8.75, 0) {};
		\node [style=langstate] (4) at (-12, 0.25) {\!\!\!{\tt Alice}\!\!\!};
		\node [style=langstate] (5) at (-4, 0.25) {\!\!\!{\tt flowers}\!\!\!};
		\node [style=langstate] (6) at (-8, 0.25) {\!\!\!{\tt likes}\!\!\!};
		\node [style=langstate] (7) at (0, 0.25) {\!\!\!{\tt that}\!\!\!};
		\node [style=langstate] (8) at (4, 0.25) {\!\!\!{\tt Bob}\!\!\!};
		\node [style=langstate] (9) at (8, 0.25) {\!\!\!{\tt gives}\!\!\!};
		\node [style=langstate] (10) at (12, 0.25) {\!\!\!{\tt Claire}\!\!\!};
		\node [style=none] (11) at (7.25, 0) {};
		\node [style=none] (12) at (4, 0) {};
		\node [style=none] (13) at (-4, 0) {};
		\node [style=none] (14) at (-0.75, 0) {};
		\node [style=none] (15) at (-0.25, 0) {};
		\node [style=none] (16) at (-7.25, 0) {};
		\node [style=none] (17) at (-8, 0) {};
		\node [style=none] (18) at (-8, -2) {};
		\node [style=none] (19) at (8.25, 0) {};
		\node [style=none] (20) at (0.75, 0) {};
		\node [style=none] (21) at (7.75, 0) {};
		\node [style=none] (22) at (0.25, 0) {};
	\end{pgfonlayer}
	\begin{pgfonlayer}{edgelayer}
		\draw [style=swap, bend right=90, looseness=1.00] (0.center) to (1.center);
		\draw [style=swap, bend right=90, looseness=1.00] (3.center) to (2.center);
		\draw [style=swap, bend right=90, looseness=1.00] (12.center) to (11.center);
		\draw [style=swap, bend right=90, looseness=1.00] (13.center) to (14.center);
		\draw [style=swap, bend right=90, looseness=0.75] (16.center) to (15.center);
		\draw [style=boldedge] (17.center) to (18.center);
		\draw [style=boldedge, bend right=90, looseness=0.75] (20.center) to (21.center);
		\draw [style=swap, bend right=90, looseness=0.75] (22.center) to (19.center);
	\end{pgfonlayer}
\end{tikzpicture}
\begin{tikzpicture}
	\begin{pgfonlayer}{nodelayer}
		\node [style=langstate] (0) at (-12, 0.25) {\!\!\!{\tt Bob}\!\!\!};
		\node [style=langstate] (1) at (-8, 0.25) {\!\!\!{\tt gives}\!\!\!};
		\node [style=langstate] (2) at (-4, 0.25) {\!\!\!{\tt Claire}\!\!\!};
		\node [style=langstate] (3) at (0, 0.25) {\!\!\!{\tt flowers}\!\!\!};
		\node [style=langstate] (4) at (4, 0.25) {\!\!\!{\tt that}\!\!\!};
		\node [style=langstate] (5) at (8, 0.25) {\!\!\!{\tt Alice}\!\!\!};
		\node [style=langstate] (6) at (12, 0.25) {\!\!\!{\tt likes}\!\!\!};
		\node [style=none] (7) at (-12, 0) {};
		\node [style=none] (8) at (-8.75, 0) {};
		\node [style=none] (9) at (-7.25, 0) {};
		\node [style=none] (10) at (-4, 0) {};
		\node [style=none] (11) at (-7.75, 0) {};
		\node [style=none] (12) at (3.75, 0) {};
		\node [style=none] (13) at (0, 0) {};
		\node [style=none] (14) at (3.25, 0) {};
		\node [style=none] (15) at (8, 0) {};
		\node [style=none] (16) at (11.25, 0) {};
		\node [style=none] (17) at (4.25, 0) {};
		\node [style=none] (18) at (12.75, 0) {};
		\node [style=none] (19) at (-8.25, -2) {};
		\node [style=none] (20) at (-8.25, 0) {};
		\node [style=none] (21) at (4.75, 0) {};
		\node [style=none] (22) at (12, 0) {};
	\end{pgfonlayer}
	\begin{pgfonlayer}{edgelayer}
		\draw [style=swap, bend right=90, looseness=1.00] (7.center) to (8.center);
		\draw [style=swap, bend right=90, looseness=1.00] (9.center) to (10.center);
		\draw [style=swap, bend right=90, looseness=0.50] (11.center) to (12.center);
		\draw [style=swap, bend right=90, looseness=1.00] (13.center) to (14.center);
		\draw [style=swap, bend right=90, looseness=1.00] (15.center) to (16.center);
		\draw [style=swap, bend right=90, looseness=0.75] (17.center) to (18.center);
		\draw [style=boldedge] (20.center) to (19.center);
		\draw [style=boldedge, bend right=90, looseness=0.75] (21.center) to (22.center);
	\end{pgfonlayer}
\end{tikzpicture}
\caption{\label{flowers}} 
\end{figure*}
Without any further context the factual data conveyed by these two sentences is the same.\footnote{Additional context could indicate a causal connection between the two parts of the sentence, which could result in the two sentences having different meanings -- see \cite{GramCircs} for more details.} How can we formally establish this connection between the two sentences?

\section{Rewriting pregroup diagrams via internal wirings}\label{sec:internal wiring} 

What is needed are `internal wirings' of certain  words, that is, not treating these words as `black boxes', but specifying what is inside, at least to some extent. Equationally speaking, they provide a congruence for pregroup diagrams, and we can establish equality by means of topological deformation.

For constructing these internal wirings we make use of `spiders' \cite{CPV, CKbook} (a.k.a.~Frobenius algebras \cite{CarboniWalters, CatsII}). One can think of these spiders as a generalisation of wires to multi-wires, as rather than having two ends, they can have multiple ends.  Still, all they do, like wires, is connect stuff, and when you connect connected stuff to other connected stuff (a.k.a.~`spider-fusion'):
\[
\begin{tikzpicture}[scale=0.80]
	\begin{pgfonlayer}{nodelayer}
		\node [style=none] (0) at (0, 0) {};
		\node [style=none] (1) at (-2, 0) {};
		\node [style=none] (2) at (-2, 2.5) {};
		\node [style=none] (3) at (0, 2.5) {};
		\node [style=none] (4) at (-1, 2.25) {\ldots};
		\node [style=none] (5) at (-1, 0.25) {\ldots};
		\node [style=white dot] (6) at (-1, 1.25) {};
		\node [style=none] (7) at (-1, 1.25) {};
		\node [style=none] (8) at (0, -2.5) {};
		\node [style=none] (9) at (1, -1.25) {};
		\node [style=none] (10) at (0, 0) {};
		\node [style=none] (11) at (1, -0.25) {\ldots};
		\node [style=none] (12) at (1, -2.25) {\ldots};
		\node [style=white dot] (13) at (1, -1.25) {};
		\node [style=none] (14) at (2, -2.5) {}; 
		\node [style=none] (15) at (2, 0) {}; 
		\node [style=none] (16) at (-2, -2.5) {};
		\node [style=none] (17) at (2, 2.5) {}; 
	\end{pgfonlayer}
	\begin{pgfonlayer}{edgelayer} 
		\draw [style=swap, in=135, out=-90, looseness=1.00] (2.center) to (6);
		\draw [style=swap, in=45, out=-90, looseness=1.00] (3.center) to (6);
		\draw [style=swap, in=90, out=-45, looseness=1.00] (6) to (0.center);
		\draw [style=swap, in=90, out=-135, looseness=1.00] (6) to (1.center);
		\draw [style=swap, in=135, out=-90, looseness=1.00] (10.center) to (13);
		\draw [style=swap, in=45, out=-90, looseness=1.00] (15.center) to (13);
		\draw [style=swap, in=90, out=-45, looseness=1.00] (13) to (14.center);
		\draw [style=swap, in=90, out=-135, looseness=1.00] (13) to (8.center);
		\draw [style=swap] (17.center) to (15.center);
		\draw [style=swap] (1.center) to (16.center);
	\end{pgfonlayer}
\end{tikzpicture}\ \ =\ \ \begin{tikzpicture}[scale=0.80]
	\begin{pgfonlayer}{nodelayer}
		\node [style=none] (0) at (1, -1.25) {};
		\node [style=none] (1) at (-1, -1.25) {};
		\node [style=none] (2) at (-1, 1.25) {};
		\node [style=none] (3) at (1, 1.25) {};
		\node [style=none] (4) at (0, 1) {\ldots};
		\node [style=none] (5) at (0, -1) {\ldots};
		\node [style=white dot] (6) at (0, 0) {}; 
		\node [style=none] (7) at (0, 0) {};
	\end{pgfonlayer}
	\begin{pgfonlayer}{edgelayer}
		\draw [style=swap, in=135, out=-90, looseness=1.00] (2.center) to (6);
		\draw [style=swap, in=45, out=-90, looseness=1.00] (3.center) to (6);
		\draw [style=swap, in=90, out=-45, looseness=1.00] (6) to (0.center);
		\draw [style=swap, in=90, out=-135, looseness=1.00] (6) to (1.center);     
	\end{pgfonlayer}
\end{tikzpicture} 
\]
\[
\begin{tikzpicture}[scale=0.80]
	\begin{pgfonlayer}{nodelayer}
		\node [style=none] (0) at (0, -1) {};
		\node [style=none] (1) at (0, 1) {};
		\node [style=white dot] (2) at (0, 0) {};
		\node [style=none] (3) at (0, 0) {};
	\end{pgfonlayer}
	\begin{pgfonlayer}{edgelayer}
		\draw [style=swap] (0.center) to (1.center); 
	\end{pgfonlayer}
\end{tikzpicture}\ \ =\ \ \begin{tikzpicture}[scale=0.80]
	\begin{pgfonlayer}{nodelayer}
		\node [style=none] (0) at (0, -1) {};
		\node [style=none] (1) at (0, 1) {};
	\end{pgfonlayer}
	\begin{pgfonlayer}{edgelayer}
		\draw [style=swap] (0.center) to (1.center);  
	\end{pgfonlayer}
\end{tikzpicture}
\qquad\quad
\begin{tikzpicture}[scale=0.80]
	\begin{pgfonlayer}{nodelayer}
		\node [style=none] (0) at (1.25, 0.5) {};
		\node [style=none] (1) at (-1.25, 0.5) {};
		\node [style=white dot, yshift=-0.5] (2) at (0, -0.5) {};
		\node [style=none] (3) at (0, -0.5) {};
	\end{pgfonlayer}
	\begin{pgfonlayer}{edgelayer}
		\draw [style=swap, bend right=90, looseness=1.50] (1.center) to (0.center); 
	\end{pgfonlayer}
\end{tikzpicture} \ \ =\ \ \begin{tikzpicture}[scale=0.80]
	\begin{pgfonlayer}{nodelayer}
		\node [style=none] (0) at (1.25, 0.5) {};
		\node [style=none] (1) at (-1.25, 0.5) {};
	\end{pgfonlayer}
	\begin{pgfonlayer}{edgelayer}
		\draw [style=swap, bend right=90, looseness=1.50] (1.center) to (0.center);  
	\end{pgfonlayer}
\end{tikzpicture}
\]

We presented internal wiring in terms of pregroup diagrams.  This is because they do not make sense in terms of symbolic  pregroups presented as preordered monoids: 

\begin{theorem}  
A pregroup with spiders is trivial. Concretely, given a  preordered monoid $(X, \leq, \otimes)$ with unit $1$, if for $x \in X$ there are spiders with $x$ as its legs,  then $x \simeq 1$. 
\end{theorem}
\begin{proof}
Having spiders on $x$ means that for all $j,k \in \mathbb{N}$  there exists: 
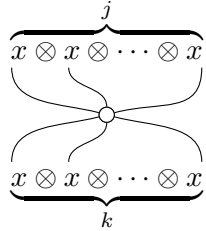
\[
\begin{tikzpicture}
	\begin{pgfonlayer}{nodelayer}
		\node [style=none] (0) at (-2.5, 1.25) {};
		\node [style=none] (1) at (-2.5, -1.25) {};
		\node [style=none] (2) at (-1, 1.25) {};
		\node [style=none] (3) at (-1, -1.25) {};
		\node [style=none] (4) at (2.5, 1.25) {};
		\node [style=none] (5) at (2.5, -1.25) {};
		\node [style=white dot] (6) at (0, 0) {};
		\node [style=white dot] (7) at (0, 0) {};
		\node [style=none] (8) at (0, 2.25) {$\overbrace{x \otimes x \otimes \cdots \otimes x}^{j}$};
		\node [style=none] (9) at (0, -2.25) {$\underbrace{x \otimes x \otimes \cdots \otimes x}_{k}$};
	\end{pgfonlayer}
	\begin{pgfonlayer}{edgelayer}
		\draw [in=-90, out=165, looseness=1.00] (6) to (0.center);
		\draw [in=-90, out=90, looseness=1.00] (6) to (2.center);
		\draw [in=-90, out=15, looseness=1.25] (6) to (4.center);
		\draw [in=90, out=-165, looseness=1.00] (7) to (1.center);
		\draw [in=90, out=-90, looseness=1.00] (7) to (3.center);
		\draw [in=90, out=-15, looseness=1.00] (7) to (5.center);
	\end{pgfonlayer}
\end{tikzpicture} 
\]
that is, we have $\bigotimes^j x \leq \bigotimes^k x$. So in particular, $x \leq 1$ and $1\leq x$, so $x \simeq 1$.
\end{proof}

Hence this paper requires diagrams in a fundamental manner.\footnote{One SEMSPACE referee requested a category-theoretic generalisations of the above stated nogo-theorem.  Such a generalisation has been provided on Twitter following our request \cite{AmarTwitter}.  Our result should also not be confused with the (almost contradictory sounding) following one, which states that pregroups \underline{are}  spiders in the category of preordered relations \cite{DuskoPregroup}.}  

\subsection{Internal wiring for relative pronouns}    
 
For relative pronouns we start with the internal wirings  that were introduced in \cite{FrobMeanI, FrobMeanII}:
\beq\label{eq:relpron}  
\begin{tikzpicture}
	\begin{pgfonlayer}{nodelayer}
		\node [style=label] (0) at (0, 2) {\tt relative pronoun};
		\node [style=none] (1) at (0.5, 0) {};
		\node [style=none] (2) at (-2, 0) {};
		\node [style=none] (3) at (-1.5, 0) {};
		\node [style=none] (4) at (-1.5, 0) {};
		\node [style=none] (5) at (-0.5, 0) {};
		\node [style=none] (6) at (-0.5, 0) {};
		\node [style=none] (7) at (0, 1.5) {};
		\node [style=none] (8) at (-0.5, 0) {};
		\node [style=none] (9) at (0.5, 0) {};
		\node [style=none] (10) at (0.5, 0) {};
		\node [style=none] (11) at (-1.5, 0) {};
		\node [style=none] (12) at (2, 0) {};
		\node [style=none] (13) at (-0.5, 0) {};
		\node [style=none] (14) at (-2, 1) {};
		\node [style=none] (15) at (-1.5, 0) {};
		\node [style=none] (16) at (-0.5, 0) {};
		\node [style=none] (17) at (0.5, 0) {};
		\node [style=none] (18) at (-1.5, 0) {};
		\node [style=none] (19) at (-0.5, 0) {};
		\node [style=none] (20) at (-0.5, 0) {};
		\node [style=none] (21) at (0.5, 0) {};
		\node [style=none] (22) at (1.75, 0) {};
		\node [style=none] (23) at (-0.5, 0) {};
		\node [style=white dot] (24) at (-0.5, 1) {};
		\node [style=none] (25) at (0.5, 0) {};
		\node [style=none] (26) at (-0.5, 0) {};
		\node [style=none] (27) at (2, 1) {};
		\node [style=none] (28) at (0.5, -0.75) {};
		\node [style=none] (29) at (0.5, 0) {};
		\node [style=none] (30) at (0.5, 0) {};
		\node [style=none] (31) at (-0.5, -0.75) {};
		\node [style=none] (32) at (-0.5, 0) {};
		\node [style=none] (33) at (-0.5, 0) {};
		\node [style=none] (34) at (-1.5, -0.75) {};
		\node [style=none] (35) at (-1.5, 0) {};
		\node [style=none] (36) at (-1.5, 0) {};
		\node [style=none] (37) at (3, -1.5) {};
		\node [style=right label] (38) at (3.25, -1.5) {\color{gray}sentence};
		\node [style=none] (39) at (1.5, -1) {};
		\node [style=none] (40) at (-1.5, -1) {};
		\node [style=none] (41) at (-2.75, -1.5) {};
		\node [style=left label] (42) at (-3, -1.5) {\color{gray}nouns};
		\node [style=none] (43) at (-2.75, -1.5) {};
		\node [style=none] (44) at (-0.5, -1) {};
		\node [style=none] (45) at (-2.75, -1.5) {};
		\node [style=none] (46) at (0.5, -1) {};
		\node [style=none] (47) at (1.5, -0.75) {};
		\node [style=none] (48) at (1.5, 0) {};
		\node [style=none] (49) at (1.5, 0) {};
		\node [style=none] (50) at (1.5, 0) {};
		\node [style=none] (51) at (1.5, 0) {};
		\node [style=none] (52) at (1.5, 0) {};
		\node [style=white dot, boldedge] (53) at (1.5, 0.5) {};
		\node [style=none] (54) at (1.5, 0) {};
	\end{pgfonlayer}
	\begin{pgfonlayer}{edgelayer}
		\draw [style=swap, in=180, out=90, looseness=1.00] (11.center) to (24);
		\draw [style=swap] (24) to (19.center);
		\draw [style=swap, in=90, out=0, looseness=1.00] (24) to (9.center);
		\draw [style=swap, dashed] (14.center) to (2.center);
		\draw [style=swap, dashed] (2.center) to (12.center);
		\draw [style=swap, dashed] (12.center) to (27.center);
		\draw [style=swap, dashed] (7.center) to (14.center);
		\draw [style=swap, dashed] (7.center) to (27.center);
		\draw [style=swap] (28.center) to (29.center);
		\draw [style=swap] (31.center) to (32.center);
		\draw [style=swap] (34.center) to (35.center);
		\draw [style=pointer edge, in=-90, out=180, looseness=1.25] (37.center) to (39.center);
		\draw [style=pointer edge, in=-90, out=0, looseness=1.00] (41.center) to (40.center);
		\draw [style=pointer edge, in=-90, out=0, looseness=1.00] (43.center) to (44.center);
		\draw [style=pointer edge, in=-90, out=0, looseness=1.00] (45.center) to (46.center);
		\draw [style=boldedge] (53) to (52.center);
		\draw [style=boldedge] (47.center) to (50.center);
	\end{pgfonlayer}
\end{tikzpicture}  
\eeq
Substituting this internal wiring in the pregroup diagrams we saw above:      
and permuting the boxes a bit, more specifically, swapping {\tt Bob gives Claire} and {\tt Alice likes} in the 2nd diagram, 
the two diagrams start to look a lot more like each other, as can be seen in figure \ref{flowersBIS}. Their only difference is a twist which vanishes if we take spiders to be commutative,\footnote{Non-commutativity can be seen as a witness for the fact that within a broader context the two sentences may defer in meaning due to a potential causal connection between its two parts -- see \cite{GramCircs} for more details.} and either a loose sentence-type wire coming out of the verb {\tt likes} in the first diagram, versus coming out the verb {\tt give} in the second diagram, the other verb having its sentence type deleted.

\begin{figure*}
\centering
\begin{tikzpicture}
	\begin{pgfonlayer}{nodelayer}
		\node [style=none] (0) at (-10, 0) {};
		\node [style=none] (1) at (-6.75, 0) {};
		\node [style=none] (2) at (9.25, 0) {};
		\node [style=none] (3) at (6, 0) {};
		\node [style=langstate] (4) at (-10, 0.25) {\!\!\!{\tt Alice}\!\!\!};
		\node [style=langstate] (5) at (-2, 0.25) {\!\!\!{\tt flowers}\!\!\!};
		\node [style=langstate] (6) at (-6, 0.25) {\!\!\!{\tt likes}\!\!\!};
		\node [style=langstate] (7) at (1.25, 0.25) {\!\!\!{\tt Bob}\!\!\!};
		\node [style=langstate] (8) at (5.25, 0.25) {\!\!\!{\tt gives}\!\!\!};
		\node [style=langstate] (9) at (9.25, 0.25) {\!\!\!{\tt Claire}\!\!\!};
		\node [style=none] (10) at (4.5, 0) {};
		\node [style=none] (11) at (1.25, 0) {};
		\node [style=none] (12) at (-2, 0) {};
		\node [style=none] (13) at (-5.25, 0) {};
		\node [style=none] (14) at (5.5, 0) {};
		\node [style=none] (15) at (-2, 0) {};
		\node [style=none] (16) at (-5.25, 0) {};
		\node [style=none] (17) at (5.5, 0) {};
		\node [style=white dot] (18) at (2.75, -1.75) {};
		\node [style=none] (19) at (-6, 0) {};
		\node [style=none] (20) at (-6, -2) {};
		\node [style=white dot, boldedge] (21) at (5, -0.75) {};
		\node [style=none] (22) at (5, 0) {};
	\end{pgfonlayer}
	\begin{pgfonlayer}{edgelayer}
		\draw [style=swap, bend right=90, looseness=1.00] (0.center) to (1.center);
		\draw [style=swap, bend right=90, looseness=1.00] (3.center) to (2.center);
		\draw [style=swap, bend right=90, looseness=1.00] (11.center) to (10.center);
		\draw [style=swap, in=180, out=-90, looseness=0.50] (16.center) to (18);
		\draw [style=swap, in=-90, out=165, looseness=1.00] (18) to (15.center);
		\draw [style=swap, in=-90, out=0, looseness=1.25] (18) to (17.center);
		\draw [style=swap, boldedge] (19.center) to (20.center);
		\draw [style=swap, boldedge] (22.center) to (21);
	\end{pgfonlayer}
\end{tikzpicture}
\begin{tikzpicture}
	\begin{pgfonlayer}{nodelayer}
		\node [style=none] (0) at (5.5, 0) {};
		\node [style=langstate] (1) at (1.25, 0.25) {\!\!\!{\tt Bob}\!\!\!};
		\node [style=none] (2) at (4.5, 0) {};
		\node [style=none] (3) at (-5.25, 0) {};
		\node [style=none] (4) at (5.5, 0) {};
		\node [style=langstate] (5) at (5.25, 0.25) {\!\!\!{\tt gives}\!\!\!};
		\node [style=none] (6) at (1.25, 0) {};
		\node [style=none] (7) at (-10, 0) {};
		\node [style=langstate] (8) at (-2, 0.25) {\!\!\!{\tt flowers}\!\!\!};
		\node [style=none] (9) at (-2, 0) {};
		\node [style=langstate] (10) at (-6, 0.25) {\!\!\!{\tt likes}\!\!\!};
		\node [style=langstate] (11) at (-10, 0.25) {\!\!\!{\tt Alice}\!\!\!};
		\node [style=none] (12) at (-2, 0) {};
		\node [style=white dot] (13) at (1.25, -1) {};
		\node [style=langstate] (14) at (9.25, 0.25) {\!\!\!{\tt Claire}\!\!\!};
		\node [style=none] (15) at (9.25, 0) {};
		\node [style=none] (16) at (6, 0) {};
		\node [style=none] (17) at (-6.75, 0) {};
		\node [style=none] (18) at (-5.25, 0) {};
		\node [style=none] (19) at (1.25, -1.75) {};
		\node [style=none] (20) at (5, 0) {};
		\node [style=none] (21) at (5, -2.25) {};
		\node [style=white dot, boldedge] (22) at (-6, -0.75) {};
		\node [style=none] (23) at (-6, 0) {};
	\end{pgfonlayer}
	\begin{pgfonlayer}{edgelayer}
		\draw [style=swap, bend right=90, looseness=1.00] (7.center) to (17.center);
		\draw [style=swap, bend right=90, looseness=1.00] (16.center) to (15.center);
		\draw [style=swap, bend right=90, looseness=1.00] (6.center) to (2.center);
		\draw [style=swap, in=-90, out=165, looseness=1.00] (13) to (9.center);
		\draw [style=swap, in=-165, out=-90, looseness=0.75] (18.center) to (19.center);
		\draw [style=swap, in=-15, out=-90, looseness=1.00] (4.center) to (19.center);
		\draw [style=swap, in=165, out=-135, looseness=1.50] (13) to (19.center);
		\draw [style=swap, in=15, out=-45, looseness=1.25] (13) to (19.center);
		\draw [style=boldedge] (20.center) to (21.center);
		\draw [style=boldedge] (23.center) to (22);
	\end{pgfonlayer}
\end{tikzpicture}
\caption{\label{flowersBIS}}
\end{figure*}

\subsection{Internal wiring for verbs}\label{sec:intWverb}    

The deleting of sentence-types of verbs:
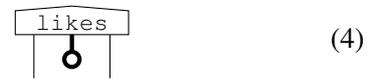
\beq\label{eq:deleted}
\begin{tikzpicture}
	\begin{pgfonlayer}{nodelayer}
		\node [style=langstate] (0) at (0, 0.5) {\!{\tt likes}\!};
		\node [style=none] (1) at (1, 0.25) {};
		\node [style=none] (2) at (1, -1) {};
		\node [style=none] (3) at (1, 0.25) {};
		\node [style=none] (4) at (-1, -1) {};
		\node [style=none] (5) at (-1, 0.25) {};
		\node [style=none] (6) at (-1, 0.25) {};
		\node [style=white dot, boldedge] (7) at (0, -0.5) {};
		\node [style=none] (8) at (0, 0.25) {};
		\node [style=none] (9) at (0, 0.25) {};
		\node [style=none] (10) at (0, 0.25) {};
		\node [style=none] (11) at (0, 0.25) {};
	\end{pgfonlayer}
	\begin{pgfonlayer}{edgelayer}
		\draw [style=swap] (3.center) to (2.center);
		\draw [style=swap] (6.center) to (4.center);
		\draw [style=boldedge] (7) to (8.center);
	\end{pgfonlayer}
\end{tikzpicture}  
\eeq
by the internal wiring of relative pronouns seems to prevent us from bringing the diagrams of Figure (\ref{flowersBIS}) any closer to each other. However, this irreversibility does not happen for a particular kind of internal wiring for the verb \cite{GrefSadr, KartsaklisSadrzadeh2014, CoeckeText, CoeckeMeich}, here generalised to the non-commutative case as demonstrated by the transitive verb in Figure (\ref{eq:verb}). For transitive verbs in spider-form, if the sentence type gets deleted we can bring back the original form by copying the remaining wires:   
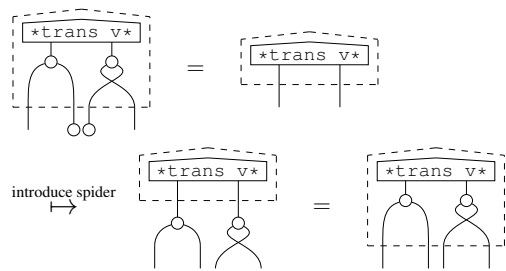
\beqa
&&\scalebox{0.8}{\begin{tikzpicture}
	\begin{pgfonlayer}{nodelayer}
		\node [style=none] (0) at (-5.5, -1.25) {};
		\node [style=none] (1) at (-4, -1.25) {};
		\node [style=none] (2) at (-5.5, -2) {};
		\node [style=none] (3) at (-4, -2) {};
		\node [style=none] (4) at (-3.5, -1.25) {};
		\node [style=none] (5) at (-2, -1.25) {};
		\node [style=none] (6) at (-3.5, -2) {};
		\node [style=none] (7) at (-2, -2) {};
		\node [style=white dot] (8) at (-3.5, -2) {};
		\node [style=white dot] (9) at (-4, -2) {};
		\node [style=none] (10) at (0, 0) {$=$};
		\node [style=none] (11) at (4.75, -1.25) {};
		\node [style=none] (12) at (2.75, 0.25) {};
		\node [style=none] (13) at (4.75, 0.25) {};
		\node [style=none] (14) at (2.75, -1.25) {};
		\node [style=none] (15) at (1.5, -0.5) {};
		\node [style=none] (16) at (6, 1) {};
		\node [style=none] (17) at (1.5, 1) {};
		\node [style=none] (18) at (3.75, 1.25) {};
		\node [style=none] (19) at (6, -0.5) {};
		\node [style=none] (20) at (-2, -1.25) {};
		\node [style=none] (21) at (-6, 1.75) {};
		\node [style=none] (22) at (-2.75, 1) {};
		\node [style=none] (23) at (-2.75, 0.25) {};
		\node [style=white dot] (24) at (-2.75, 0.25) {};
		\node [style=none] (25) at (-6, -1.25) {};
		\node [style=none] (26) at (-4.75, 0.25) {};
		\node [style=none] (27) at (-1.5, -1.25) {};
		\node [style=white dot] (28) at (-4.75, 0.25) {};
		\node [style=none] (29) at (-1.5, 1.75) {};
		\node [style=none] (30) at (-4, -1.25) {};
		\node [style=none] (31) at (-5.5, -1.25) {};
		\node [style=none] (32) at (-4.75, 1) {};
		\node [style=langstate] (33) at (-3.75, 1.25) {\!\!\!\!{\tt *trans v*}\!\!\!\!};
		\node [style=none] (34) at (-3.5, -1.25) {};
		\node [style=none] (35) at (-3.75, 2) {};
		\node [style=langstate] (36) at (3.75, 0.5) {\!\!\!\!{\tt *trans v*}\!\!\!\!};
	\end{pgfonlayer}
	\begin{pgfonlayer}{edgelayer}
		\draw [style=swap] (1.center) to (3.center);
		\draw [style=swap] (0.center) to (2.center);
		\draw [style=swap] (5.center) to (7.center);
		\draw [style=swap] (4.center) to (6.center);
		\draw [style=swap] (14.center) to (12.center);
		\draw [style=swap] (11.center) to (13.center);
		\draw [style=swap, dashed] (17.center) to (15.center);
		\draw [style=swap, dashed] (15.center) to (19.center);
		\draw [style=swap, dashed] (19.center) to (16.center);
		\draw [style=swap, dashed] (18.center) to (17.center);
		\draw [style=swap, dashed] (18.center) to (16.center);
		\draw [style=swap] (26.center) to (32.center);
		\draw [style=swap, in=90, out=-15, looseness=1.00] (28) to (30.center);
		\draw [style=swap, in=90, out=-165, looseness=1.00] (28) to (31.center);
		\draw [style=swap] (23.center) to (22.center);
		\draw [style=swap, in=90, out=-150, looseness=1.25] (24) to (20.center);
		\draw [style=swap, in=90, out=-30, looseness=1.25] (24) to (34.center);
		\draw [style=swap, dashed] (21.center) to (25.center);
		\draw [style=swap, dashed] (25.center) to (27.center);
		\draw [style=swap, dashed] (27.center) to (29.center);
		\draw [style=swap, dashed] (35.center) to (21.center);
		\draw [style=swap, dashed] (35.center) to (29.center);
	\end{pgfonlayer}
\end{tikzpicture}}\\
&&\stackrel{\mbox{\tiny introduce spider}}{\mapsto}\ \ \scalebox{0.8}{\begin{tikzpicture}
	\begin{pgfonlayer}{nodelayer}
		\node [style=none] (0) at (0, 0) {$=$};
		\node [style=none] (1) at (-4.75, -0.5) {};
		\node [style=none] (2) at (-5.5, -2) {};
		\node [style=none] (3) at (-4.75, -0.5) {};
		\node [style=white dot] (4) at (-4.75, -0.5) {};
		\node [style=none] (5) at (-4, -1.25) {};
		\node [style=none] (6) at (-4, -2) {};
		\node [style=none] (7) at (-5.5, -1.25) {};
		\node [style=none] (8) at (-6, 0.25) {};
		\node [style=none] (9) at (-6, 1.75) {};
		\node [style=none] (10) at (-4.75, -0.5) {};
		\node [style=none] (11) at (-1.5, 1.75) {};
		\node [style=none] (12) at (-1.5, 0.25) {};
		\node [style=none] (13) at (-2.75, -0.5) {};
		\node [style=none] (14) at (-4.75, 1) {};
		\node [style=none] (15) at (-2.75, 1) {};
		\node [style=none] (16) at (-3.75, 2) {};
		\node [style=white dot] (17) at (-2.75, -0.5) {};
		\node [style=none] (18) at (-2.75, -0.5) {};
		\node [style=none] (19) at (-2, -2) {};
		\node [style=none] (20) at (-3.5, -2) {};
		\node [style=langstate] (21) at (3.75, 1.25) {\!\!\!\!{\tt *trans v*}\!\!\!\!};
		\node [style=none] (22) at (3.75, 2) {};
		\node [style=none] (23) at (2, -2) {};
		\node [style=none] (24) at (2.75, 1) {};
		\node [style=none] (25) at (6, -1.25) {};
		\node [style=none] (26) at (4, -2) {};
		\node [style=none] (27) at (3.5, -2) {};
		\node [style=none] (28) at (3.5, -1.25) {};
		\node [style=none] (29) at (5.5, -2) {};
		\node [style=none] (30) at (4.75, 1) {};
		\node [style=white dot] (31) at (4.75, 0.25) {};
		\node [style=none] (32) at (6, 1.75) {};
		\node [style=none] (33) at (1.5, 1.75) {};
		\node [style=none] (34) at (4.75, 0.25) {};
		\node [style=none] (35) at (2.75, 0.25) {};
		\node [style=none] (36) at (5.5, -1.25) {};
		\node [style=white dot] (37) at (2.75, 0.25) {};
		\node [style=none] (38) at (1.5, -1.25) {};
		\node [style=none] (39) at (2, -1.25) {};
		\node [style=none] (40) at (4, -1.25) {};
		\node [style=langstate] (41) at (-3.75, 1.25) {\!\!\!\!{\tt *trans v*}\!\!\!\!};
	\end{pgfonlayer}
	\begin{pgfonlayer}{edgelayer}
		\draw [style=swap] (5.center) to (6.center);
		\draw [style=swap] (7.center) to (2.center);
		\draw [style=swap, in=90, out=-15, looseness=1.00] (4) to (5.center);
		\draw [style=swap, in=90, out=-165, looseness=1.00] (4) to (7.center);
		\draw [style=swap] (10.center) to (14.center);
		\draw [style=swap] (13.center) to (15.center);
		\draw [style=swap, dashed] (9.center) to (8.center);
		\draw [style=swap, dashed] (8.center) to (12.center);
		\draw [style=swap, dashed] (12.center) to (11.center);
		\draw [style=swap, dashed] (16.center) to (9.center);
		\draw [style=swap, dashed] (16.center) to (11.center);
		\draw [style=swap, in=90, out=-150, looseness=1.25] (17) to (19.center);
		\draw [style=swap, in=90, out=-30, looseness=1.25] (17) to (20.center);
		\draw [style=swap] (35.center) to (24.center);
		\draw [style=swap] (28.center) to (27.center);
		\draw [style=swap] (39.center) to (23.center);
		\draw [style=swap, in=90, out=-15, looseness=1.00] (37) to (28.center);
		\draw [style=swap, in=90, out=-165, looseness=1.00] (37) to (39.center);
		\draw [style=swap] (34.center) to (30.center);
		\draw [style=swap] (36.center) to (29.center);
		\draw [style=swap] (40.center) to (26.center);
		\draw [style=swap, in=90, out=-150, looseness=1.25] (31) to (36.center);
		\draw [style=swap, in=90, out=-30, looseness=1.25] (31) to (40.center);
		\draw [style=swap, dashed] (33.center) to (38.center);
		\draw [style=swap, dashed] (38.center) to (25.center);
		\draw [style=swap, dashed] (25.center) to (32.center);
		\draw [style=swap, dashed] (22.center) to (33.center);
		\draw [style=swap, dashed] (22.center) to (32.center);
	\end{pgfonlayer}
\end{tikzpicture}}
\eeqa
So nothing was ever lost. To conclude, for the internal wiring of verbs proposed above, the copying and deleting spiders now guarantee that in (\ref{eq:deleted}) nothing gets lost.

\subsection{Rewriting pregroup diagrams into each other} 

Introducing the internal wiring (\ref{eq:verb}) and deleting all outputs, our example sentences now appear as in the first two diagrams of Figure (\ref{flowersQUAD}). Except for the twist the two pregroup diagrams have become the same.  As we have no outputs anymore, let's just stick in a copy-spider for all nouns, and then  after fusing all deletes away, our sentence is transformed into the third diagram of figure \ref{flowersQUAD}.

\begin{figure*}
\centering
\begin{tikzpicture}
	\begin{pgfonlayer}{nodelayer}
		\node [style=none] (0) at (-6.25, 0) {};
		\node [style=none] (1) at (7.5, 0) {};
		\node [style=langstate] (2) at (1.25, 0.25) {\!\!\!{\tt Bob}\!\!\!};
		\node [style=none] (3) at (4.5, 0) {};
		\node [style=none] (4) at (-5.25, 0) {};
		\node [style=none] (5) at (6, 0) {};
		\node [style=none] (6) at (1.25, 0) {};
		\node [style=none] (7) at (-11, 0) {};
		\node [style=langstate] (8) at (-2, 0.25) {\!\!\!{\tt flowers}\!\!\!};
		\node [style=none] (9) at (-2, 0) {};
		\node [style=white dot] (10) at (-6.25, -0.25) {};
		\node [style=langstate] (11) at (-11, 0.25) {\!\!\!{\tt Alice}\!\!\!};
		\node [style=none] (12) at (-2, 0) {};
		\node [style=white dot] (13) at (1, -1.25) {};
		\node [style=langstate] (14) at (11.75, 0.25) {\!\!\!{\tt Claire}\!\!\!};
		\node [style=none] (15) at (11.75, 0) {};
		\node [style=none] (16) at (8.5, 0) {};
		\node [style=none] (17) at (-7.75, 0) {};
		\node [style=none] (18) at (-5.25, 0) {};
		\node [style=white dot] (19) at (-7.25, 1.25) {};
		\node [style=none] (20) at (-6.25, 0) {};
		\node [style=none] (21) at (-5.75, 1.25) {};
		\node [style=none] (22) at (-6.75, 0) {};
		\node [style=none] (23) at (-7.25, 1.25) {};
		\node [style=none] (24) at (-7.75, 0) {};
		\node [style=none] (25) at (-7.25, 2) {};
		\node [style=none] (26) at (-5.75, 2) {};
		\node [style=langstate] (27) at (-6.5, 2.25) {\!\!\!\!{\tt *likes*}\!\!\!\!};
		\node [style=none] (28) at (-5.25, 0) {};
		\node [style=white dot] (29) at (-5.75, 1.25) {};
		\node [style=none] (30) at (-6.75, 0) {};
		\node [style=white dot] (31) at (-6.75, -0.25) {};
		\node [style=none] (32) at (8, 1.25) {};
		\node [style=none] (33) at (6.5, 1.25) {};
		\node [style=none] (34) at (4.5, 0) {};
		\node [style=none] (35) at (6, 0) {};
		\node [style=none] (36) at (8.5, 0) {};
		\node [style=white dot] (37) at (5, 1.25) {};
		\node [style=none] (38) at (5.5, 0) {};
		\node [style=none] (39) at (6.5, 2) {};
		\node [style=none] (40) at (6.5, 0) {};
		\node [style=white dot] (41) at (6.5, 1.25) {};
		\node [style=none] (42) at (7.5, 0) {};
		\node [style=none] (43) at (5, 1.25) {};
		\node [style=none] (44) at (5, 2) {};
		\node [style=langstate] (45) at (6.5, 2.25) {\!\!{\tt *gives*}\!\!};
		\node [style=none] (46) at (8, 2) {};
		\node [style=white dot] (47) at (8, 1.25) {};
		\node [style=none] (48) at (6.5, 0) {};
		\node [style=none] (49) at (5.5, 0) {};
		\node [style=none] (50) at (5.5, 0) {};
		\node [style=white dot] (51) at (5.5, -0.25) {};
		\node [style=white dot] (52) at (6, -0.25) {};
		\node [style=none] (53) at (6, 0) {};
		\node [style=none] (54) at (6, 0) {};
		\node [style=none] (55) at (6.5, 0) {};
		\node [style=white dot] (56) at (6.5, -0.25) {};
		\node [style=none] (57) at (5.5, 0) {};
		\node [style=none] (58) at (6.5, 0) {};
	\end{pgfonlayer}
	\begin{pgfonlayer}{edgelayer}
		\draw [style=swap, bend right=90, looseness=1.00] (7.center) to (17.center);
		\draw [style=swap, bend right=90, looseness=1.00] (16.center) to (15.center);
		\draw [style=swap, bend right=90, looseness=1.00] (6.center) to (3.center);
		\draw [style=swap] (0.center) to (10);
		\draw [style=swap, in=180, out=-90, looseness=0.50] (4.center) to (13);
		\draw [style=swap, in=-90, out=165, looseness=1.00] (13) to (9.center);
		\draw [style=swap, in=-90, out=0, looseness=0.50] (13) to (1.center);
		\draw [style=swap] (23.center) to (25.center);
		\draw [style=swap, in=90, out=-15, looseness=1.00] (19) to (22.center);
		\draw [style=swap, in=90, out=-165, looseness=1.00] (19) to (24.center);
		\draw [style=swap] (21.center) to (26.center);
		\draw [style=swap, in=90, out=-150, looseness=1.25] (29) to (28.center);
		\draw [style=swap, in=90, out=-30, looseness=1.25] (29) to (20.center);
		\draw [style=swap] (30.center) to (31);
		\draw [style=swap] (43.center) to (44.center);
		\draw [style=swap, in=90, out=-15, looseness=1.00] (37) to (38.center);
		\draw [style=swap, in=90, out=-165, looseness=1.00] (37) to (34.center);
		\draw [style=swap] (33.center) to (39.center);
		\draw [style=swap, in=90, out=-150, looseness=1.25] (41) to (42.center);
		\draw [style=swap, in=90, out=-30, looseness=1.25] (41) to (35.center);
		\draw [style=swap] (32.center) to (46.center);
		\draw [style=swap, in=90, out=-150, looseness=1.25] (47) to (36.center);
		\draw [style=swap, in=90, out=-45, looseness=1.25] (47) to (40.center);
		\draw [style=swap] (50.center) to (51);
		\draw [style=swap] (53.center) to (52);
		\draw [style=swap] (58.center) to (56);
	\end{pgfonlayer}
\end{tikzpicture}
\begin{tikzpicture}
	\begin{pgfonlayer}{nodelayer}
		\node [style=none] (0) at (-6.25, 0) {};
		\node [style=none] (1) at (7.5, 0) {};
		\node [style=langstate] (2) at (1.25, 0.25) {\!\!\!{\tt Bob}\!\!\!};
		\node [style=none] (3) at (4.5, 0) {};
		\node [style=none] (4) at (-5.25, 0) {};
		\node [style=none] (5) at (6, 0) {};
		\node [style=none] (6) at (1.25, 0) {};
		\node [style=none] (7) at (-11, 0) {};
		\node [style=langstate] (8) at (-2, 0.25) {\!\!\!{\tt flowers}\!\!\!};
		\node [style=none] (9) at (-2, 0) {};
		\node [style=white dot] (10) at (-6.25, -0.25) {};
		\node [style=langstate] (11) at (-11, 0.25) {\!\!\!{\tt Alice}\!\!\!};
		\node [style=none] (12) at (-2, 0) {};
		\node [style=langstate] (13) at (11.75, 0.25) {\!\!\!{\tt Claire}\!\!\!};
		\node [style=none] (14) at (11.75, 0) {};
		\node [style=none] (15) at (8.5, 0) {};
		\node [style=none] (16) at (-7.75, 0) {};
		\node [style=none] (17) at (-5.25, 0) {};
		\node [style=white dot] (18) at (-7.25, 1.25) {};
		\node [style=none] (19) at (-6.25, 0) {};
		\node [style=none] (20) at (-5.75, 1.25) {};
		\node [style=none] (21) at (-6.75, 0) {};
		\node [style=none] (22) at (-7.25, 1.25) {};
		\node [style=none] (23) at (-7.75, 0) {};
		\node [style=none] (24) at (-7.25, 2) {};
		\node [style=none] (25) at (-5.75, 2) {};
		\node [style=langstate] (26) at (-6.5, 2.25) {\!\!\!\!{\tt *likes*}\!\!\!\!};
		\node [style=none] (27) at (-5.25, 0) {};
		\node [style=white dot] (28) at (-5.75, 1.25) {};
		\node [style=none] (29) at (-6.75, 0) {};
		\node [style=white dot] (30) at (-6.75, -0.25) {};
		\node [style=none] (31) at (8, 1.25) {};
		\node [style=none] (32) at (6.5, 1.25) {};
		\node [style=none] (33) at (4.5, 0) {};
		\node [style=none] (34) at (6, 0) {};
		\node [style=none] (35) at (8.5, 0) {};
		\node [style=white dot] (36) at (5, 1.25) {};
		\node [style=none] (37) at (5.5, 0) {};
		\node [style=none] (38) at (6.5, 2) {};
		\node [style=none] (39) at (6.5, 0) {};
		\node [style=white dot] (40) at (6.5, 1.25) {};
		\node [style=none] (41) at (7.5, 0) {};
		\node [style=none] (42) at (5, 1.25) {};
		\node [style=none] (43) at (5, 2) {};
		\node [style=langstate] (44) at (6.5, 2.25) {\!\!{\tt *gives*}\!\!};
		\node [style=none] (45) at (8, 2) {};
		\node [style=white dot] (46) at (8, 1.25) {};
		\node [style=none] (47) at (6.5, 0) {};
		\node [style=none] (48) at (5.5, 0) {};
		\node [style=white dot] (49) at (1, -1) {};
		\node [style=none] (50) at (-2, 0) {};
		\node [style=none] (51) at (7.5, 0) {};
		\node [style=none] (52) at (-5.25, 0) {};
		\node [style=none] (53) at (1, -1.75) {};
		\node [style=none] (54) at (-5.25, 0) {};
		\node [style=none] (55) at (-2, 0) {};
		\node [style=none] (56) at (6, 0) {};
		\node [style=white dot] (57) at (6, -0.25) {};
		\node [style=none] (58) at (6, 0) {};
		\node [style=white dot] (59) at (5.5, -0.25) {};
		\node [style=none] (60) at (5.5, 0) {};
		\node [style=none] (61) at (5.5, 0) {};
		\node [style=white dot] (62) at (6.5, -0.25) {};
		\node [style=none] (63) at (6.5, 0) {};
		\node [style=none] (64) at (6.5, 0) {};
	\end{pgfonlayer}
	\begin{pgfonlayer}{edgelayer}
		\draw [style=swap, bend right=90, looseness=1.00] (7.center) to (16.center);
		\draw [style=swap, bend right=90, looseness=1.00] (15.center) to (14.center);
		\draw [style=swap, bend right=90, looseness=1.00] (6.center) to (3.center);
		\draw [style=swap] (0.center) to (10);
		\draw [style=swap] (22.center) to (24.center);
		\draw [style=swap, in=90, out=-15, looseness=1.00] (18) to (21.center);
		\draw [style=swap, in=90, out=-165, looseness=1.00] (18) to (23.center);
		\draw [style=swap] (20.center) to (25.center);
		\draw [style=swap, in=90, out=-150, looseness=1.25] (28) to (27.center);
		\draw [style=swap, in=90, out=-30, looseness=1.25] (28) to (19.center);
		\draw [style=swap] (29.center) to (30);
		\draw [style=swap] (42.center) to (43.center);
		\draw [style=swap, in=90, out=-15, looseness=1.00] (36) to (37.center);
		\draw [style=swap, in=90, out=-165, looseness=1.00] (36) to (33.center);
		\draw [style=swap] (32.center) to (38.center);
		\draw [style=swap, in=90, out=-150, looseness=1.25] (40) to (41.center);
		\draw [style=swap, in=90, out=-30, looseness=1.25] (40) to (34.center);
		\draw [style=swap] (31.center) to (45.center);
		\draw [style=swap, in=90, out=-150, looseness=1.25] (46) to (35.center);
		\draw [style=swap, in=90, out=-45, looseness=1.25] (46) to (39.center);
		\draw [style=swap, in=-90, out=165, looseness=1.00] (49) to (50.center);
		\draw [style=swap, in=-165, out=-90, looseness=0.75] (52.center) to (53.center);
		\draw [style=swap, in=-15, out=-90, looseness=0.75] (51.center) to (53.center);
		\draw [style=swap, in=165, out=-135, looseness=1.50] (49) to (53.center);
		\draw [style=swap, in=15, out=-45, looseness=1.25] (49) to (53.center);
		\draw [style=swap] (56.center) to (57);
		\draw [style=swap] (60.center) to (59);
		\draw [style=swap] (63.center) to (62);
	\end{pgfonlayer}
\end{tikzpicture}
\begin{tikzpicture}
	\begin{pgfonlayer}{nodelayer}
		\node [style=none] (0) at (0, 0.75) {};
		\node [style=langstate] (1) at (0, 1) {\!\!\!{\tt Bob}\!\!\!};
		\node [style=none] (2) at (-6.5, 0.75) {};
		\node [style=langstate] (3) at (-7.25, 1) {\!\!\!\!{\tt *likes*}\!\!\!\!};
		\node [style=white dot] (4) at (-9.75, -0.25) {};
		\node [style=none] (5) at (-11.25, 0.75) {};
		\node [style=none] (6) at (3.25, 0.75) {};
		\node [style=none] (7) at (4.75, 0.75) {};
		\node [style=none] (8) at (-6.5, 0.75) {};
		\node [style=none] (9) at (4.75, 0.75) {};
		\node [style=langstate] (10) at (-3.25, 1) {\!\!\!{\tt flowers}\!\!\!};
		\node [style=none] (11) at (-8, 0.75) {};
		\node [style=white dot] (12) at (1.5, -0.25) {};
		\node [style=none] (13) at (-3.5, -1.75) {};
		\node [style=langstate] (14) at (4.75, 1) {{\tt *gives*}};
		\node [style=none] (15) at (-9.75, -1.75) {};
		\node [style=white dot] (16) at (-0.25, -1.25) {};
		\node [style=none] (17) at (-3.25, 0.75) {};
		\node [style=langstate] (18) at (9.5, 1) {\!\!\!{\tt Claire}\!\!\!};
		\node [style=none] (19) at (-6.5, 0.75) {};
		\node [style=white dot] (20) at (-3, -0.25) {};
		\node [style=none] (21) at (3.25, 0.75) {};
		\node [style=none] (22) at (-3.25, 0.75) {};
		\node [style=none] (23) at (-8, 0.75) {};
		\node [style=langstate] (24) at (-11.25, 1) {\!\!\!{\tt Alice}\!\!\!};
		\node [style=none] (25) at (-8, 0.75) {};
		\node [style=none] (26) at (-6.5, 0.75) {};
		\node [style=none] (27) at (1.5, -1.75) {};
		\node [style=none] (28) at (4.75, 0.75) {};
		\node [style=none] (29) at (3.25, 0.75) {};
		\node [style=none] (30) at (9.5, 0.75) {};
		\node [style=white dot] (31) at (8, -0.75) {};
		\node [style=none] (32) at (6.25, 0.75) {};
		\node [style=none] (33) at (6.25, 0.75) {};
		\node [style=none] (34) at (8, -1.75) {};
		\node [style=none] (35) at (8, -0.75) {};
		\node [style=none] (36) at (6.25, 0.75) {};
		\node [style=none] (37) at (6.25, 0.75) {};
		\node [style=none] (38) at (9.5, 0.75) {};
	\end{pgfonlayer}
	\begin{pgfonlayer}{edgelayer}
		\draw [style=swap, bend right=90, looseness=1.00] (5.center) to (11.center);
		\draw [style=swap, bend right=90, looseness=1.00] (0.center) to (21.center);
		\draw [style=swap, in=180, out=-90, looseness=0.75] (2.center) to (16);
		\draw [style=swap, in=-90, out=165, looseness=1.00] (16) to (17.center);
		\draw [style=swap, in=-90, out=0, looseness=0.75] (16) to (28.center);
		\draw [style=swap] (4) to (15.center);
		\draw [style=swap] (12) to (27.center);
		\draw [style=swap, in=90, out=-135, looseness=0.75] (20) to (13.center);
		\draw [style=swap, in=-90, out=150, looseness=1.25] (31) to (38.center);
		\draw [style=swap, in=-90, out=30, looseness=1.25] (31) to (33.center);
		\draw [style=swap] (35.center) to (34.center);
	\end{pgfonlayer}
\end{tikzpicture}
\caption{\label{flowersQUAD}}  
\end{figure*}

The recipe we followed here is an instance of a general result that allows us to relate sentences for a substantial portion of English, by providing internal wirings for that fragment. In Section \ref{sec:catalog} we will provide internal wirings some grammatical word classes -- in \cite{GramCircs} we provide a much larger catalog -- that will generate correspondences between grammatical constructs, just like the one established above. In Section \ref{Proof-of-concept} we provide some further examples of this. In \cite{GramCircs} we also provide a normal form induced by grammar equations.

\section{The wrapping gadget}

Above in (\ref{eq:verb}) we saw that sentence wires were decomposed into noun wires. However, for pregroup proofs it is important to know that those wires do belong together, so we need to introduce a tool that enables us to express that they belong together.

\begin{defn} 
\label{def:wrapgadget}
The \underline{wrapping gadget}  forces a number of  wires to be treated as one, i.e.~it \underline{wraps} them, and is denoted as follows:
\[
\begin{tikzpicture}
	\begin{pgfonlayer}{nodelayer}
		\node [style=none] (0) at (-2.5, 0.5) {};
		\node [style=none] (1) at (-2.5, 0) {};
		\node [style=none] (2) at (2.5, 0) {};
		\node [style=none] (3) at (-2.5, 0) {};
		\node [style=none] (4) at (2.5, 0.5) {};
		\node [style=none] (5) at (2.5, 0) {};
		\node [style=none] (6) at (-2, 2) {};
		\node [style=none] (7) at (-2, 0) {};
		\node [style=none] (8) at (-1, 0.5) {$\cdots$};
		\node [style=none] (9) at (1, 0.5) {$\cdots$};
		\node [style=none] (10) at (0, 2) {};
		\node [style=none] (11) at (0, 0) {};
		\node [style=none] (12) at (2, 2) {};
		\node [style=none] (13) at (2, 0) {};
		\node [style=none] (14) at (-1.25, 1.5) {$Y_1$};
		\node [style=none] (15) at (0.75, 1.5) {$Y_i$};
		\node [style=none] (16) at (2.75, 1.5) {$Y_N$};
		\node [style=none] (17) at (0, -2) {};
		\node [style=none] (18) at (0, 0) {};
		\node [style=none] (19) at (1.5, -1.25) {$\big[ \bigotimes\limits_{i=1}^{N} Y_i \big]$};
	\end{pgfonlayer}
	\begin{pgfonlayer}{edgelayer}
		\draw [style=boldedge] (0.center) to (1.center);
		\draw [style=boldedge] (2.center) to (3.center);
		\draw [style=boldedge] (4.center) to (5.center);
		\draw (12.center) to (13.center);
		\draw (10.center) to (11.center);
		\draw (6.center) to (7.center);
		\draw [style=boldedge] (17.center) to (18.center);
	\end{pgfonlayer}
\end{tikzpicture}     
\]
By \underline{unfolding} we mean dropping the restrictions imposed by the wrapping gadget.  Cups and spiders carry over to wrapped wires in the expected way, following the conventions of \cite{CKbook}. 
\end{defn}  

In fact, in the case of relative pronouns simply wrapping the noun-wires making up the sentence type isn't enough, as the counterexample in Figure (\ref{ctrex}) shows.

\begin{figure*}
\centering
\begin{tikzpicture}
	\begin{pgfonlayer}{nodelayer}
		\node [style=none] (0) at (-13.75, 0) {};
		\node [style=none] (1) at (-10.5, 0) {};
		\node [style=none] (2) at (10.5, 0) {};
		\node [style=none] (3) at (7.5, 0) {};
		\node [style=langstate] (4) at (-13.75, 0.25) {\!\!\!{\tt Alice}\!\!\!};
		\node [style=langstate] (5) at (-5.75, 0.25) {\!\!\!{\tt flowers}\!\!\!};
		\node [style=langstate] (6) at (-9.75, 0.25) {\!\!\!{\tt gives}\!\!\!};
		\node [style=none] (7) at (-0.75, 2.75) {\!\!\!{\tt\footnotesize that}\!\!\!};
		\node [style=langstate] (8) at (3.5, 0.25) {\!\!\!{\tt Bob}\!\!\!};
		\node [style=langstate] (9) at (6.75, 0.25) {\!\!\!{\tt plays}\!\!\!};
		\node [style=langstate] (10) at (10.5, 0.25) {\!\!\!{\tt chess}\!\!\!};
		\node [style=none] (11) at (6, 0) {};
		\node [style=none] (12) at (3.5, 0) {};
		\node [style=none] (13) at (-5.75, 0) {};
		\node [style=none] (14) at (-2.25, 0) {};
		\node [style=none] (15) at (-1.5, 0) {};
		\node [style=none] (16) at (-9, 0) {};
		\node [style=none] (17) at (-0.75, 0) {};
		\node [style=none] (18) at (-9.5, 0) {};
		\node [style=none] (19) at (0.75, 0.75) {};
		\node [style=none] (20) at (-1.5, 0.75) {};
		\node [style=none] (21) at (-1.5, 0.75) {};
		\node [style=white dot] (22) at (0, 1.25) {};
		\node [style=none] (23) at (-0.75, 0.75) {};
		\node [style=none] (24) at (0.5, 0.5) {};
		\node [style=none] (25) at (-2.75, 0) {};
		\node [style=none] (26) at (0.5, 0.75) {};
		\node [style=none] (27) at (0.25, 0.5) {};
		\node [style=none] (28) at (0, 0.75) {};
		\node [style=none] (29) at (-2.25, 0.75) {};
		\node [style=none] (30) at (-0.75, 0.75) {};
		\node [style=none] (31) at (0.5, 0.75) {};
		\node [style=none] (32) at (0.5, 0.75) {};
		\node [style=none] (33) at (0.25, 0) {};
		\node [style=none] (34) at (-1.5, 0.75) {};
		\node [style=none] (35) at (0, 0.75) {};
		\node [style=none] (36) at (0.5, 0.75) {};
		\node [style=none] (37) at (-0.25, 0.5) {};
		\node [style=none] (38) at (-1.5, 0.75) {};
		\node [style=none] (39) at (-0.75, 0.75) {};
		\node [style=none] (40) at (0.75, 0.5) {};
		\node [style=none] (41) at (-0.75, 0.75) {};
		\node [style=none] (42) at (0, 0.5) {};
		\node [style=none] (43) at (-1.5, 0) {};
		\node [style=none] (44) at (-1.5, 0.75) {};
		\node [style=none] (45) at (-1.5, 0.75) {};
		\node [style=none] (46) at (-0.75, 0.75) {};
		\node [style=none] (47) at (1.25, 1.75) {};
		\node [style=none] (48) at (-0.75, 0.75) {};
		\node [style=none] (49) at (0, 0.75) {};
		\node [style=none] (50) at (-0.75, 0.75) {};
		\node [style=none] (51) at (-0.75, 2.25) {};
		\node [style=none] (52) at (-0.75, 0) {};
		\node [style=none] (53) at (-1.5, 0.75) {};
		\node [style=none] (54) at (-2.25, 0.75) {};
		\node [style=none] (55) at (0, 0.75) {};
		\node [style=none] (56) at (-2.75, 1.75) {};
		\node [style=none] (57) at (-1.5, 0.75) {};
		\node [style=none] (58) at (-0.75, 0.75) {};
		\node [style=none] (59) at (-0.25, 0.5) {};
		\node [style=none] (60) at (0.75, 0.75) {};
		\node [style=none] (61) at (0, 0.75) {};
		\node [style=none] (62) at (0.75, 0.5) {};
		\node [style=none] (63) at (-2.25, 0.75) {};
		\node [style=none] (64) at (-1.5, 0.75) {};
		\node [style=none] (65) at (-2.25, 0.75) {};
		\node [style=white dot] (66) at (0.5, 1.25) {};
		\node [style=none] (67) at (-1.5, 0.75) {};
		\node [style=none] (68) at (-2.25, 0.75) {};
		\node [style=none] (69) at (0.25, 0) {};
		\node [style=none] (70) at (1.25, 0) {};
		\node [style=none] (71) at (0.5, 0.75) {};
		\node [style=none] (72) at (-2.25, 0) {};
		\node [style=none] (73) at (-1.5, 0.75) {};
		\node [style=none] (74) at (-2.25, 0.75) {};
		\node [style=none] (75) at (0, 0.75) {};
		\node [style=white dot] (76) at (-1.5, 1.5) {};
		\node [style=none] (77) at (0.5, 0.75) {};
		\node [style=none] (78) at (-2.25, 0.75) {};
		\node [style=none] (79) at (-0.25, 0.75) {};
		\node [style=none] (80) at (6.75, 0) {};
		\node [style=none] (81) at (0.25, 0) {};
		\node [style=none] (82) at (-10, -3) {};
		\node [style=none] (83) at (-10, 0) {};
	\end{pgfonlayer}
	\begin{pgfonlayer}{edgelayer}
		\draw [style=swap, bend right=90, looseness=1.00] (0.center) to (1.center);
		\draw [style=swap, bend right=90, looseness=1.00] (3.center) to (2.center);
		\draw [style=swap, bend right=90, looseness=1.00] (12.center) to (11.center);
		\draw [style=swap, bend right=90, looseness=1.00] (13.center) to (14.center);
		\draw [style=swap, bend right=90, looseness=0.75] (16.center) to (15.center);
		\draw [style=swap, bend left=90, looseness=1.00] (17.center) to (18.center);
		\draw [style=swap, in=180, out=90, looseness=1.00] (65.center) to (76);
		\draw [style=swap] (76) to (44.center);
		\draw [style=swap, in=90, out=0, looseness=1.00] (76) to (41.center);
		\draw [style=swap, dashed] (56.center) to (25.center);
		\draw [style=swap, dashed] (25.center) to (70.center);
		\draw [style=swap, dashed] (70.center) to (47.center);
		\draw [style=swap, dashed] (51.center) to (56.center);
		\draw [style=swap, dashed] (51.center) to (47.center);
		\draw [style=swap] (66) to (71.center);
		\draw [style=swap] (24.center) to (77.center);
		\draw [style=swap] (52.center) to (30.center);
		\draw [style=swap] (43.center) to (38.center);
		\draw [style=swap] (72.center) to (74.center);
		\draw [style=swap] (22) to (49.center);
		\draw [style=swap] (42.center) to (61.center);
		\draw [style=boldedge] (27.center) to (69.center);
		\draw [style=boldedge] (79.center) to (59.center);
		\draw [style=boldedge] (62.center) to (37.center);
		\draw [style=boldedge] (19.center) to (40.center);
		\draw [style=boldedge, bend right=90, looseness=0.75] (81.center) to (80.center);
		\draw [style=boldedge] (83.center) to (82.center);
	\end{pgfonlayer}
\end{tikzpicture}
\caption{\label{ctrex}The deleting of the sentence type of {\tt plays} belongs together with the noun-wire now connecting the relative pronoun with {\tt gives}, like in Figure (\ref{flowers}). This is enforced by the internal wiring of the object relative pronoun in Figure (\ref{eq:relpron})}
\end{figure*}
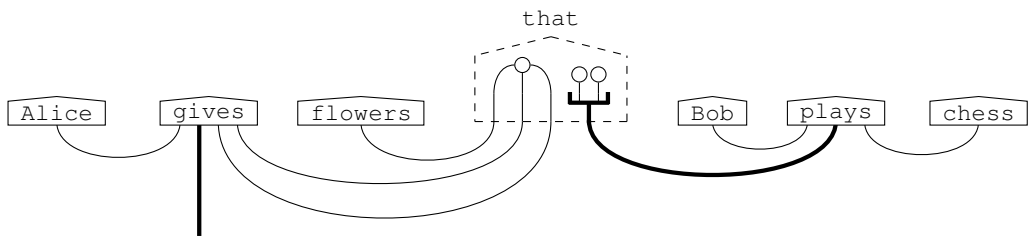

\section{Some more internal wirings}
\label{sec:catalog}

We now provide internal wirings for some grammatical word classes that feature in the examples of the next section. We distinguish between `content words', like the verbs in (\ref{eq:verb}), and `functional words', like the relative pronouns in (\ref{eq:relpron}).

\subsection{Content words}

We provide internal wirings for intransitive and transitive verbs in Figure (\ref{eq:verb}), and predicative and attributive adverbs for transitive verbs in Figure (\ref{eq:adverb}).

\begin{figure*}
\centering
\begin{tikzpicture}
	\begin{pgfonlayer}{nodelayer}
		\node [style=none] (0) at (5.75, 1.25) {};
		\node [style=none] (1) at (5, 0.5) {};
		\node [style=langstate] (2) at (6, 1.5) {\!\!\!\!{\tt *TV*}\!\!\!\!};
		\node [style=white dot] (3) at (5, 0.5) {};
		\node [style=none] (4) at (4.25, -1.25) {};
		\node [style=none] (5) at (5.75, -0.75) {};
		\node [style=none] (6) at (4.25, -2) {};
		\node [style=none] (7) at (6.25, -0.75) {};
		\node [style=none] (8) at (6.25, 1.25) {};
		\node [style=white dot] (9) at (7, 0.5) {};
		\node [style=none] (10) at (7.75, -0.75) {};
		\node [style=none] (11) at (7, 0.5) {};
		\node [style=none] (12) at (7.75, -2) {};
		\node [style=none] (13) at (8.25, -1.25) {};
		\node [style=none] (14) at (3.75, -1.25) {};
		\node [style=none] (15) at (3.75, 2) {};
		\node [style=none] (16) at (6, 2.25) {};
		\node [style=none] (17) at (8.25, 2) {};
		\node [style=none] (18) at (5.25, -0.5) {};
		\node [style=none] (19) at (5.25, -0.75) {};
		\node [style=none] (20) at (6.75, -0.75) {};
		\node [style=none] (21) at (5.25, -0.75) {};
		\node [style=none] (22) at (6.75, -0.5) {};
		\node [style=none] (23) at (6.75, -0.75) {};
		\node [style=none] (24) at (6, -0.75) {};
		\node [style=none] (25) at (6, -2) {};
		\node [style=none] (26) at (6, -2.75) {\footnotesize${}^{-1}n \cdot [n \cdot n] \cdot n^{-1}$};
		\node [style=none] (27) at (-9.5, 1.25) {};
		\node [style=none] (28) at (-9.5, 0.5) {};
		\node [style=langstate] (29) at (-9.5, 1.5) {\!\!\!\!{\tt *IV*}\!\!\!\!};
		\node [style=white dot] (30) at (-9.5, 0.5) {};
		\node [style=none] (31) at (-10.25, -1.25) {};
		\node [style=none] (32) at (-8.75, -0.75) {};
		\node [style=none] (33) at (-10.25, -2) {};
		\node [style=none] (34) at (-8.5, 1.25) {};
		\node [style=none] (35) at (-7.25, -1.25) {};
		\node [style=none] (36) at (-11.75, -1.25) {};
		\node [style=none] (37) at (-11.75, 2) {};
		\node [style=none] (38) at (-9.5, 2.25) {};
		\node [style=none] (39) at (-7.25, 2) {};
		\node [style=none] (40) at (-9.25, -0.5) {};
		\node [style=none] (41) at (-9.25, -0.75) {};
		\node [style=none] (42) at (-8.25, -0.75) {};
		\node [style=none] (43) at (-9.25, -0.75) {};
		\node [style=none] (44) at (-8.25, -0.5) {};
		\node [style=none] (45) at (-8.25, -0.75) {};
		\node [style=none] (46) at (-8.75, -0.75) {};
		\node [style=none] (47) at (-8.75, -2) {};
		\node [style=none] (48) at (-9.5, -2.75) {\footnotesize${}^{-1}n \cdot [n]$};
		\node [style=none] (49) at (-9.5, 3) {\small\tt Intransitive Verb};
		\node [style=none] (50) at (6, 3) {\small\tt Transitive Verb};
	\end{pgfonlayer}
	\begin{pgfonlayer}{edgelayer}
		\draw [style=swap, in=-75, out=90, looseness=0.75] (1.center) to (0.center);
		\draw [style=swap] (4.center) to (6.center);
		\draw [style=swap, in=90, out=-15, looseness=1.00] (3) to (5.center);
		\draw [style=swap, in=90, out=-165, looseness=1.00] (3) to (4.center);
		\draw [style=swap, in=-90, out=90, looseness=0.75] (11.center) to (8.center);
		\draw [style=swap, in=90, out=-150, looseness=1.25] (9) to (10.center);
		\draw [style=swap, in=60, out=-45, looseness=1.75] (9) to (7.center);
		\draw [style=swap, dashed] (15.center) to (14.center);
		\draw [style=swap, dashed] (14.center) to (13.center);
		\draw [style=swap, dashed] (13.center) to (17.center);
		\draw [style=swap, dashed] (16.center) to (15.center);
		\draw [style=swap, dashed] (16.center) to (17.center);
		\draw [style=boldedge] (18.center) to (19.center);
		\draw [style=boldedge] (20.center) to (21.center);
		\draw [style=boldedge] (22.center) to (23.center);
		\draw [style=boldedge] (24.center) to (25.center);
		\draw (10.center) to (12.center);
		\draw [style=swap] (28.center) to (27.center);
		\draw [style=swap] (31.center) to (33.center);
		\draw [style=swap, in=90, out=-15, looseness=1.00] (30) to (32.center);
		\draw [style=swap, in=90, out=-165, looseness=1.00] (30) to (31.center);
		\draw [style=swap, dashed] (37.center) to (36.center);
		\draw [style=swap, dashed] (36.center) to (35.center);
		\draw [style=swap, dashed] (35.center) to (39.center);
		\draw [style=swap, dashed] (38.center) to (37.center);
		\draw [style=swap, dashed] (38.center) to (39.center);
		\draw [style=boldedge] (40.center) to (41.center);
		\draw [style=boldedge] (42.center) to (43.center);
		\draw [style=boldedge] (44.center) to (45.center);
		\draw [style=boldedge] (46.center) to (47.center);
	\end{pgfonlayer}
\end{tikzpicture}
\caption{\label{eq:verb}}
\end{figure*}

\begin{figure*}
\centering
\input{vtype-adverb-inspired-2.tikz}
\caption{\label{eq:adverb}}
\end{figure*}

\subsection{Functional words}

We provide internal wirings for subject and object relative pronouns for intransitive verbs, and a passive-voice construction `word' for transitive verbs in Figure (\ref{eq:relpron}).

\begin{figure*}
\centering
\begin{tikzpicture}
	\begin{pgfonlayer}{nodelayer}
		\node [style=none] (0) at (17, -2) {};
		\node [style=none] (1) at (9, -2) {};
		\node [style=none] (2) at (9, 1.5) {};
		\node [style=none] (3) at (13, 2) {};
		\node [style=none] (4) at (17, 1.5) {};
		\node [style=none] (5) at (13, 2.5) {\small\tt Passive voice${}_{\tt TV}$};
		\node [style=none] (6) at (9.75, -1.25) {};
		\node [style=none] (7) at (9.75, -1.5) {};
		\node [style=none] (8) at (12.25, -1.5) {};
		\node [style=none] (9) at (9.75, -1.5) {};
		\node [style=none] (10) at (12.25, -1.25) {};
		\node [style=none] (11) at (12.25, -1.5) {};
		\node [style=none] (12) at (11, -1.5) {};
		\node [style=none] (13) at (11, -2.75) {};
		\node [style=none] (14) at (11, -1.5) {};
		\node [style=none] (15) at (11, -1.5) {};
		\node [style=none] (16) at (13, -3.5) {\footnotesize${}^{-1}[{}^{-1}n \cdot [n \cdot n] \cdot n^{-1}] \cdot {}^{-1}n \cdot [n \cdot n] \cdot n^{-1}$};
		\node [style=none] (17) at (10, -1.5) {};
		\node [style=none] (18) at (12, -1.5) {};
		\node [style=none] (19) at (10, 0.5) {};
		\node [style=none] (20) at (12, 0.25) {};
		\node [style=none] (21) at (11, -1.5) {};
		\node [style=none] (22) at (16, -2) {};
		\node [style=none] (23) at (16, -2.75) {};
		\node [style=none] (24) at (13, -2) {};
		\node [style=none] (25) at (13, -2.75) {};
		\node [style=none] (26) at (13, 0.25) {};
		\node [style=none] (27) at (16, 0.5) {};
		\node [style=none] (28) at (14.5, -2) {};
		\node [style=none] (29) at (11, -1.5) {};
		\node [style=none] (30) at (14.5, -1.5) {};
		\node [style=none] (31) at (14.5, -2.75) {};
		\node [style=none] (32) at (11, -1) {};
		\node [style=none] (33) at (10.5, -0.75) {};
		\node [style=none] (34) at (10.5, -1) {};
		\node [style=none] (35) at (11.5, -1) {};
		\node [style=none] (36) at (10.5, -1) {};
		\node [style=none] (37) at (11.5, -0.75) {};
		\node [style=none] (38) at (11.5, -1) {};
		\node [style=none] (39) at (14.5, -1) {};
		\node [style=none] (40) at (14, -1.25) {};
		\node [style=none] (41) at (14, -1.5) {};
		\node [style=none] (42) at (15, -1.5) {};
		\node [style=none] (43) at (14, -1.5) {};
		\node [style=none] (44) at (15, -1.25) {};
		\node [style=none] (45) at (15, -1.5) {};
		\node [style=none] (46) at (11, -1.25) {};
		\node [style=none] (47) at (11, -1) {};
		\node [style=none] (48) at (11, -1.5) {};
		\node [style=none] (49) at (14.25, -1) {};
		\node [style=none] (50) at (14.75, -1) {};
		\node [style=none] (51) at (11.25, -1) {};
		\node [style=none] (52) at (10.75, -1) {};
		\node [style=none] (53) at (14.25, -1.5) {};
		\node [style=none] (54) at (14.75, -1.5) {};
		\node [style=none] (55) at (1.5, -2.25) {};
		\node [style=none] (56) at (4, -0.25) {};
		\node [style=none] (57) at (-8.75, 1) {};
		\node [style=white dot] (58) at (3.75, 0.25) {};
		\node [style=none] (59) at (-7, -0.5) {};
		\node [style=none] (60) at (-6, -0.75) {};
		\node [style=white dot] (61) at (1.75, 0.25) {};
		\node [style=none] (62) at (2.5, -1) {};
		\node [style=none] (63) at (-6.75, 2) {\small\tt Sub.~Rel.~Pron.${}_{\tt IV}$};
		\node [style=none] (64) at (-6.25, -0.75) {};
		\node [style=none] (65) at (-8.75, -1.75) {};
		\node [style=none] (66) at (3.25, -0.5) {};
		\node [style=none] (67) at (4, -0.5) {};
		\node [style=none] (68) at (-5, -1.25) {};
		\node [style=none] (69) at (4, -1) {};
		\node [style=none] (70) at (-8.25, -1.75) {};
		\node [style=white dot] (71) at (-7, 0.75) {};
		\node [style=none] (72) at (2.5, -1) {};
		\node [style=none] (73) at (4, -0.5) {};
		\node [style=none] (74) at (4.25, 0.75) {};
		\node [style=none] (75) at (-6.5, -0.75) {};
		\node [style=none] (76) at (0.25, 0.75) {};
		\node [style=none] (77) at (4.25, -1.5) {};
		\node [style=none] (78) at (-7, -0.75) {};
		\node [style=none] (79) at (4, -1) {};
		\node [style=none] (80) at (3.75, -0.5) {};
		\node [style=none] (81) at (-7.5, -2.5) {};
		\node [style=none] (82) at (-6.75, -3.25) {\footnotesize${}^{-1}n \cdot n \cdot [{}^{-1}n \cdot [n]]^{-1}$};
		\node [style=none] (83) at (2.25, -3) {\footnotesize${}^{-1}n \cdot n \cdot [[n \cdot n] \cdot n^{-1}]^{-1}$};
		\node [style=none] (84) at (2.25, 1.75) {\small\tt Ob.~Rel.~Pron.${}_{\tt TV}$};
		\node [style=none] (85) at (-6.5, -1.25) {};
		\node [style=none] (86) at (3.5, -1) {};
		\node [style=none] (87) at (3.5, -0.5) {};
		\node [style=none] (88) at (-6, -0.5) {};
		\node [style=none] (89) at (3, -1) {};
		\node [style=none] (90) at (3, -2.25) {};
		\node [style=none] (91) at (1.5, -1.5) {};
		\node [style=none] (92) at (-7, -1.25) {};
		\node [style=none] (93) at (-6, -0.75) {};
		\node [style=none] (94) at (3, -0.5) {};
		\node [style=none] (95) at (0.75, -2.25) {};
		\node [style=none] (96) at (2, -1) {};
		\node [style=none] (97) at (-5, -1.25) {};
		\node [style=none] (98) at (3.75, -0.5) {};
		\node [style=none] (99) at (3, -0.5) {};
		\node [style=none] (100) at (0.75, -1.5) {};
		\node [style=none] (101) at (-6.5, -0.75) {};
		\node [style=none] (102) at (2.25, 1.25) {};
		\node [style=none] (103) at (-5, -1) {};
		\node [style=none] (104) at (-8.25, -2.5) {};
		\node [style=none] (105) at (-7.5, -1.75) {};
		\node [style=none] (106) at (3, -0.25) {};
		\node [style=none] (107) at (3.5, -1) {};
		\node [style=none] (108) at (3, -1) {};
		\node [style=white dot] (109) at (3.25, 0.25) {};
		\node [style=white dot] (110) at (-6.5, 0) {};
		\node [style=none] (111) at (2, -0.75) {};
		\node [style=none] (112) at (-6.5, -1.25) {};
		\node [style=none] (113) at (-6.75, 1.5) {};
		\node [style=none] (114) at (4, -0.75) {};
		\node [style=none] (115) at (-6, -2.5) {};
		\node [style=none] (116) at (-7, -0.75) {};
		\node [style=none] (117) at (-7, -1) {};
		\node [style=none] (118) at (2, -1) {};
		\node [style=none] (119) at (-4.75, -1.75) {};
		\node [style=none] (120) at (0.25, -1.5) {};
		\node [style=none] (121) at (-5.5, -1.25) {};
		\node [style=none] (122) at (-6, -1.25) {};
		\node [style=none] (123) at (-7, -1.25) {};
		\node [style=none] (124) at (-4.75, 1) {};
	\end{pgfonlayer}
	\begin{pgfonlayer}{edgelayer}
		\draw [style=swap, dashed] (2.center) to (1.center);
		\draw [style=swap, dashed] (1.center) to (0.center);
		\draw [style=swap, dashed] (0.center) to (4.center);
		\draw [style=swap, dashed] (3.center) to (2.center);
		\draw [style=swap, dashed] (3.center) to (4.center);
		\draw [style=boldedge] (6.center) to (7.center);
		\draw [style=boldedge] (8.center) to (9.center);
		\draw [style=boldedge] (10.center) to (11.center);
		\draw [style=boldedge] (12.center) to (13.center);
		\draw (17.center) to (19.center);
		\draw (18.center) to (20.center);
		\draw (22.center) to (23.center);
		\draw (24.center) to (25.center);
		\draw (26.center) to (24.center);
		\draw (27.center) to (22.center);
		\draw [style=boldedge] (30.center) to (31.center);
		\draw [style=boldedge] (33.center) to (34.center);
		\draw [style=boldedge] (35.center) to (36.center);
		\draw [style=boldedge] (37.center) to (38.center);
		\draw [style=boldedge] (40.center) to (41.center);
		\draw [style=boldedge] (42.center) to (43.center);
		\draw [style=boldedge] (44.center) to (45.center);
		\draw [style=boldedge] (47.center) to (48.center);
		\draw [bend right=90, looseness=1.00] (49.center) to (51.center);
		\draw [bend right=90, looseness=1.00] (50.center) to (52.center);
		\draw (50.center) to (54.center);
		\draw (49.center) to (53.center);
		\draw [in=90, out=90, looseness=1.00] (19.center) to (26.center);
		\draw [in=90, out=90, looseness=0.75] (20.center) to (27.center);
		\draw [style=swap, in=90, out=0, looseness=1.25] (71) to (121.center);
		\draw [style=swap, in=90, out=-120, looseness=1.00] (71) to (105.center);
		\draw [style=swap, dashed] (57.center) to (65.center);
		\draw [style=swap, dashed] (65.center) to (119.center);
		\draw [style=swap, dashed] (119.center) to (124.center);
		\draw [style=swap, dashed] (113.center) to (57.center);
		\draw [style=swap, dashed] (113.center) to (124.center);
		\draw [style=boldedge] (117.center) to (92.center);
		\draw [style=boldedge] (68.center) to (123.center);
		\draw [style=boldedge] (103.center) to (97.center);
		\draw [style=boldedge] (122.center) to (115.center);
		\draw [style=boldedge] (75.center) to (85.center);
		\draw (81.center) to (105.center);
		\draw [style=boldedge] (59.center) to (78.center);
		\draw [style=boldedge] (60.center) to (116.center);
		\draw [style=boldedge] (88.center) to (93.center);
		\draw (101.center) to (110);
		\draw [in=180, out=90, looseness=1.00] (70.center) to (71);
		\draw (70.center) to (104.center);
		\draw [style=swap, in=90, out=0, looseness=1.00] (61) to (72.center);
		\draw [style=swap, in=90, out=-105, looseness=1.00] (61) to (91.center);
		\draw [style=swap, dashed] (76.center) to (120.center);
		\draw [style=swap, dashed] (120.center) to (77.center);
		\draw [style=swap, dashed] (77.center) to (74.center);
		\draw [style=swap, dashed] (102.center) to (76.center);
		\draw [style=swap, dashed] (102.center) to (74.center);
		\draw [style=boldedge] (111.center) to (118.center);
		\draw [style=boldedge] (69.center) to (96.center);
		\draw [style=boldedge] (114.center) to (79.center);
		\draw [style=boldedge] (108.center) to (90.center);
		\draw [style=boldedge] (87.center) to (86.center);
		\draw (55.center) to (91.center);
		\draw [style=boldedge] (106.center) to (94.center);
		\draw [style=boldedge] (67.center) to (99.center);
		\draw [style=boldedge] (56.center) to (73.center);
		\draw (66.center) to (109);
		\draw [in=180, out=90, looseness=1.00] (100.center) to (61);
		\draw (80.center) to (58);
		\draw (100.center) to (95.center);
	\end{pgfonlayer}
\end{tikzpicture}
\caption{\label{eq:relpron}}
\end{figure*}

\section{Proof-of-concept}\label{Proof-of-concept}


We provide a number of examples of how the internal wirings proposed above enable us to relate different grammatical constructs just as in the case of what the relative pronoun and verb internal wirings did for the sentences in Figure (\ref{flowers}). We omit the pregroup typings, instead depicting the pregroup diagrams directly. Wrapping gadgets correspond to bracketing pregroup types together.

\bibliography{mainNOW}
\bibliographystyle{acl_natbib}

\begin{figure*}
\centering
\begin{tikzpicture}
	\begin{pgfonlayer}{nodelayer}
		\node [style=none] (0) at (6, -1) {};
		\node [style=none] (1) at (-2, -1) {};
		\node [style=none] (2) at (-2, 2.5) {};
		\node [style=none] (3) at (2, 3) {};
		\node [style=none] (4) at (6, 2.5) {};
		\node [style=none] (5) at (2, 3.5) {\small$\texttt{Passive Voice}_{\texttt{TV}}$};
		\node [style=none] (6) at (-1.25, -0.25) {};
		\node [style=none] (7) at (-1.25, -0.5) {};
		\node [style=none] (8) at (1.25, -0.5) {};
		\node [style=none] (9) at (-1.25, -0.5) {};
		\node [style=none] (10) at (1.25, -0.25) {};
		\node [style=none] (11) at (1.25, -0.5) {};
		\node [style=none] (12) at (0, -0.5) {};
		\node [style=none] (13) at (0, -2) {};
		\node [style=none] (14) at (0, -0.5) {};
		\node [style=none] (15) at (0, -0.5) {};
		\node [style=none] (16) at (-1, -0.5) {};
		\node [style=none] (17) at (1, -0.5) {};
		\node [style=none] (18) at (-1, 1.25) {};
		\node [style=none] (19) at (1, 1) {};
		\node [style=none] (20) at (0, -0.5) {};
		\node [style=none] (21) at (5, -1) {};
		\node [style=none] (22) at (5, -1.25) {};
		\node [style=none] (23) at (2, -1) {};
		\node [style=none] (24) at (2, -1.75) {};
		\node [style=none] (29) at (2, 1) {};
		\node [style=none] (30) at (5, 1.25) {};
		\node [style=none] (31) at (3.5, -1) {};
		\node [style=none] (32) at (0, -0.5) {};
		\node [style=none] (33) at (3.5, -0.5) {};
		\node [style=none] (34) at (3.5, -3.75) {};
		\node [style=none] (35) at (0, 0) {};
		\node [style=none] (36) at (-0.5, 0.25) {};
		\node [style=none] (37) at (-0.5, 0) {};
		\node [style=none] (38) at (0.5, 0) {};
		\node [style=none] (39) at (-0.5, 0) {};
		\node [style=none] (40) at (0.5, 0.25) {};
		\node [style=none] (41) at (0.5, 0) {};
		\node [style=none] (42) at (3.5, 0) {};
		\node [style=none] (43) at (3, -0.25) {};
		\node [style=none] (44) at (3, -0.5) {};
		\node [style=none] (45) at (4, -0.5) {};
		\node [style=none] (46) at (3, -0.5) {};
		\node [style=none] (47) at (4, -0.25) {};
		\node [style=none] (48) at (4, -0.5) {};
		\node [style=none] (49) at (0, -0.25) {};
		\node [style=none] (50) at (0, 0) {};
		\node [style=none] (51) at (0, -0.5) {};
		\node [style=none] (52) at (3.25, 0) {};
		\node [style=none] (53) at (3.75, 0) {};
		\node [style=none] (54) at (0.25, 0) {};
		\node [style=none] (55) at (-0.25, 0) {};
		\node [style=none] (56) at (3.25, -0.5) {};
		\node [style=none] (57) at (3.75, -0.5) {};
		\node [style=langstate] (58) at (-9, 1.5) {\!\!\!\!{\tt Obj.}\!\!\!\!};
		\node [style=none] (59) at (-9, -1.75) {};
		\node [style=none] (60) at (-9, 1.25) {};
		\node [style=langstate] (61) at (8, 1.5) {\!\!\!\!{\tt Subj.}\!\!\!\!};
		\node [style=none] (62) at (8, -1.25) {};
		\node [style=none] (63) at (8, 1.25) {};
		\node [style=none] (64) at (-5.25, 1.25) {};
		\node [style=none] (65) at (-6, 0.5) {};
		\node [style=langstate] (66) at (-5, 1.5) {\!\!\!\!{\tt *TV*}\!\!\!\!};
		\node [style=white dot] (67) at (-6, 0.5) {};
		\node [style=none] (68) at (-6.75, -1) {};
		\node [style=none] (69) at (-5.25, -0.5) {};
		\node [style=none] (70) at (-6.75, -2) {};
		\node [style=none] (71) at (-4.75, -0.5) {};
		\node [style=none] (72) at (-4.75, 1.25) {};
		\node [style=white dot] (73) at (-4, 0.5) {};
		\node [style=none] (74) at (-3.25, -1) {};
		\node [style=none] (75) at (-4, 0.5) {};
		\node [style=none] (76) at (-3.25, -2) {};
		\node [style=none] (77) at (-2.75, -1) {};
		\node [style=none] (78) at (-7.25, -1) {};
		\node [style=none] (79) at (-7.25, 2) {};
		\node [style=none] (80) at (-5, 2.25) {};
		\node [style=none] (81) at (-2.75, 2) {};
		\node [style=none] (82) at (-5.75, -0.25) {};
		\node [style=none] (83) at (-5.75, -0.5) {};
		\node [style=none] (84) at (-4.25, -0.5) {};
		\node [style=none] (85) at (-5.75, -0.5) {};
		\node [style=none] (86) at (-4.25, -0.25) {};
		\node [style=none] (87) at (-4.25, -0.5) {};
		\node [style=none] (88) at (-5, -0.5) {};
		\node [style=none] (89) at (-5, -2) {};
		\node [style=none] (90) at (-7.25, -1.75) {};
		\node [style=none] (91) at (-7.25, -2) {};
		\node [style=none] (92) at (-2.75, -2) {};
		\node [style=none] (93) at (-7.25, -2) {};
		\node [style=none] (94) at (-2.75, -1.75) {};
		\node [style=none] (95) at (-2.75, -2) {};
		\node [style=none] (96) at (-5, -2) {};
		\node [style=none] (97) at (0, -2) {};
	\end{pgfonlayer}
	\begin{pgfonlayer}{edgelayer}
		\draw [style=swap, dashed] (2.center) to (1.center);
		\draw [style=swap, dashed] (1.center) to (0.center);
		\draw [style=swap, dashed] (0.center) to (4.center);
		\draw [style=swap, dashed] (3.center) to (2.center);
		\draw [style=swap, dashed] (3.center) to (4.center);
		\draw [style=boldedge] (6.center) to (7.center);
		\draw [style=boldedge] (8.center) to (9.center);
		\draw [style=boldedge] (10.center) to (11.center);
		\draw [style=boldedge] (12.center) to (13.center);
		\draw (16.center) to (18.center);
		\draw (17.center) to (19.center);
		\draw (21.center) to (22.center);
		\draw (23.center) to (24.center);
		\draw (29.center) to (23.center);
		\draw (30.center) to (21.center);
		\draw [style=boldedge] (33.center) to (34.center);
		\draw [style=boldedge] (36.center) to (37.center);
		\draw [style=boldedge] (38.center) to (39.center);
		\draw [style=boldedge] (40.center) to (41.center);
		\draw [style=boldedge] (43.center) to (44.center);
		\draw [style=boldedge] (45.center) to (46.center);
		\draw [style=boldedge] (47.center) to (48.center);
		\draw [style=boldedge] (50.center) to (51.center);
		\draw [bend right=90] (52.center) to (54.center);
		\draw [bend right=90] (53.center) to (55.center);
		\draw (53.center) to (57.center);
		\draw (52.center) to (56.center);
		\draw (60.center) to (59.center);
		\draw (63.center) to (62.center);
		\draw [style=swap, in=-75, out=90, looseness=0.75] (65.center) to (64.center);
		\draw [style=swap] (68.center) to (70.center);
		\draw [style=swap, in=90, out=-15] (67) to (69.center);
		\draw [style=swap, in=90, out=-165] (67) to (68.center);
		\draw [style=swap, in=-90, out=90, looseness=0.75] (75.center) to (72.center);
		\draw [style=swap, in=90, out=-150, looseness=1.25] (73) to (74.center);
		\draw [style=swap, in=60, out=-45, looseness=1.75] (73) to (71.center);
		\draw [style=swap, dashed] (79.center) to (78.center);
		\draw [style=swap, dashed] (78.center) to (77.center);
		\draw [style=swap, dashed] (77.center) to (81.center);
		\draw [style=swap, dashed] (80.center) to (79.center);
		\draw [style=swap, dashed] (80.center) to (81.center);
		\draw [style=boldedge] (82.center) to (83.center);
		\draw [style=boldedge] (84.center) to (85.center);
		\draw [style=boldedge] (86.center) to (87.center);
		\draw [style=boldedge] (88.center) to (89.center);
		\draw (74.center) to (76.center);
		\draw [style=boldedge] (90.center) to (91.center);
		\draw [style=boldedge] (92.center) to (93.center);
		\draw [style=boldedge] (94.center) to (95.center);
		\draw [style=boldedge, bend right=90, looseness=0.50] (96.center) to (97.center);
		\draw [bend right=90, looseness=0.50] (59.center) to (24.center);
		\draw [bend left=90, looseness=1.50] (62.center) to (22.center);
		\draw [in=90, out=90, looseness=1.25] (18.center) to (29.center);
		\draw [in=90, out=90] (19.center) to (30.center);
	\end{pgfonlayer}
\end{tikzpicture}
\begin{tikzpicture}
	\begin{pgfonlayer}{nodelayer}
		\node [style=none] (0) at (-0.25, -1.25) {};
		\node [style=none] (1) at (1.75, -1.25) {};
		\node [style=none] (2) at (-0.25, 0.5) {};
		\node [style=none] (3) at (1.75, 0) {};
		\node [style=none] (4) at (5.75, -0.25) {};
		\node [style=none] (5) at (5.75, -1) {};
		\node [style=none] (6) at (2.75, -0.75) {};
		\node [style=none] (7) at (2.75, -1) {};
		\node [style=none] (12) at (2.75, 0) {};
		\node [style=none] (13) at (5.75, 0.5) {};
		\node [style=none] (14) at (4.25, -1.5) {};
		\node [style=none] (15) at (4.25, -1.5) {};
		\node [style=none] (16) at (4.25, -3.75) {};
		\node [style=none] (17) at (4.25, -1) {};
		\node [style=none] (18) at (3.75, -1.25) {};
		\node [style=none] (19) at (3.75, -1.5) {};
		\node [style=none] (20) at (4.75, -1.5) {};
		\node [style=none] (21) at (3.75, -1.5) {};
		\node [style=none] (22) at (4.75, -1.25) {};
		\node [style=none] (23) at (4.75, -1.5) {};
		\node [style=none] (24) at (4, -1) {};
		\node [style=none] (25) at (4.5, -1) {};
		\node [style=none] (26) at (1, -1) {};
		\node [style=none] (27) at (0.5, -1) {};
		\node [style=none] (28) at (4, -1.5) {};
		\node [style=none] (29) at (4.5, -1.5) {};
		\node [style=langstate] (30) at (-8.25, 1.25) {\!\!\!\!{\tt Obj.}\!\!\!\!};
		\node [style=none] (31) at (-8.25, -1) {};
		\node [style=none] (32) at (-8.25, 1) {};
		\node [style=langstate] (33) at (8.75, 1.25) {\!\!\!\!{\tt Subj.}\!\!\!\!};
		\node [style=none] (34) at (8.75, -1) {};
		\node [style=none] (35) at (8.75, 1) {};
		\node [style=none] (36) at (-4.5, 1) {};
		\node [style=none] (37) at (-5.25, 0.25) {};
		\node [style=langstate] (38) at (-4.25, 1.25) {\!\!\!\!{\tt *TV*}\!\!\!\!};
		\node [style=white dot] (39) at (-5.25, 0.25) {};
		\node [style=none] (40) at (-6, -1) {};
		\node [style=none] (41) at (-4.5, -1) {};
		\node [style=none] (42) at (-6, -1.25) {};
		\node [style=none] (43) at (-4, -1) {};
		\node [style=none] (44) at (-4, 1) {};
		\node [style=white dot] (45) at (-3.25, 0.25) {};
		\node [style=none] (46) at (-2.5, -1) {};
		\node [style=none] (47) at (-3.25, 0.25) {};
		\node [style=none] (48) at (-2.5, -1.25) {};
		\node [style=none] (49) at (-4, -1.25) {};
		\node [style=none] (50) at (-4.5, -1.25) {};
		\node [style=none] (51) at (0.5, -1.25) {};
		\node [style=none] (52) at (1, -1.25) {};
		\node [style=none] (53) at (-0.25, -1.25) {};
		\node [style=none] (87) at (0, 3) {\text{(unfolding)}};
		\node [style=none] (90) at (-10.25, 0) {$\mapsto$};
	\end{pgfonlayer}
	\begin{pgfonlayer}{edgelayer}
		\draw (0.center) to (2.center);
		\draw (1.center) to (3.center);
		\draw (4.center) to (5.center);
		\draw (6.center) to (7.center);
		\draw (12.center) to (6.center);
		\draw (13.center) to (4.center);
		\draw [style=boldedge] (15.center) to (16.center);
		\draw [style=boldedge] (18.center) to (19.center);
		\draw [style=boldedge] (20.center) to (21.center);
		\draw [style=boldedge] (22.center) to (23.center);
		\draw [bend right=90] (24.center) to (26.center);
		\draw [bend right=90] (25.center) to (27.center);
		\draw (25.center) to (29.center);
		\draw (24.center) to (28.center);
		\draw (32.center) to (31.center);
		\draw (35.center) to (34.center);
		\draw [style=swap, in=-75, out=90, looseness=0.75] (37.center) to (36.center);
		\draw [style=swap] (40.center) to (42.center);
		\draw [style=swap, in=90, out=-15] (39) to (41.center);
		\draw [style=swap, in=90, out=-165] (39) to (40.center);
		\draw [style=swap, in=-90, out=90, looseness=0.75] (47.center) to (44.center);
		\draw [style=swap, in=90, out=-150, looseness=1.25] (45) to (46.center);
		\draw [style=swap, in=90, out=-45, looseness=1.50] (45) to (43.center);
		\draw (46.center) to (48.center);
		\draw [bend right=90, looseness=0.75] (31.center) to (7.center);
		\draw [bend left=90, looseness=1.50] (34.center) to (5.center);
		\draw (27.center) to (51.center);
		\draw (26.center) to (52.center);
		\draw [bend left=90, looseness=0.75] (51.center) to (49.center);
		\draw [bend left=90, looseness=0.75] (52.center) to (50.center);
		\draw [bend left=90] (53.center) to (48.center);
		\draw [bend right=90, looseness=0.75] (42.center) to (1.center);
		\draw (41.center) to (50.center);
		\draw (43.center) to (49.center);
		\draw [in=90, out=90] (12.center) to (2.center);
		\draw [in=90, out=90, looseness=0.75] (3.center) to (13.center);
	\end{pgfonlayer}
\end{tikzpicture}
\begin{tikzpicture}
	\begin{pgfonlayer}{nodelayer}
		\node [style=none] (0) at (-0.25, 1) {};
		\node [style=none] (1) at (-1, 0.25) {};
		\node [style=langstate] (2) at (0, 1.25) {\!\!\!\!{\tt *TV*}\!\!\!\!};
		\node [style=white dot] (3) at (-1, 0.25) {};
		\node [style=none] (4) at (-1.75, -1.25) {};
		\node [style=none] (5) at (-0.25, -0.75) {};
		\node [style=none] (6) at (0.25, -0.75) {};
		\node [style=none] (7) at (0.25, 1) {};
		\node [style=white dot] (8) at (1, 0.25) {};
		\node [style=none] (9) at (1.75, -1.25) {};
		\node [style=none] (10) at (1, 0.25) {};
		\node [style=none] (11) at (2.25, -1.25) {};
		\node [style=none] (12) at (-2.25, -1.25) {};
		\node [style=none] (13) at (-2.25, 1.75) {};
		\node [style=none] (14) at (0, 2) {};
		\node [style=none] (15) at (2.25, 1.75) {};
		\node [style=none] (16) at (-0.75, -0.5) {};
		\node [style=none] (17) at (-0.75, -0.75) {};
		\node [style=none] (18) at (0.75, -0.75) {};
		\node [style=none] (19) at (-0.75, -0.75) {};
		\node [style=none] (20) at (0.75, -0.5) {};
		\node [style=none] (21) at (0.75, -0.75) {};
		\node [style=none] (22) at (0, -0.75) {};
		\node [style=none] (23) at (0, -0.75) {};
		\node [style=none] (24) at (0, -2.25) {};
		\node [style=langstate] (25) at (-4, 1.25) {\!\!\!\!{\tt Subj.}\!\!\!\!};
		\node [style=none] (26) at (-4, -1.5) {};
		\node [style=none] (27) at (-4, 1) {};
		\node [style=langstate] (28) at (4, 1.25) {\!\!\!\!{\tt Obj.}\!\!\!\!};
		\node [style=none] (29) at (4, -1.5) {};
		\node [style=none] (30) at (4, 1) {};
		\node [style=none] (31) at (1.75, -1.5) {};
		\node [style=none] (32) at (-1.75, -1.5) {};
		\node [style=none] (33) at (0, 3) {\text{(rearranging wires, wrapping)}};
		\node [style=none] (34) at (-6.25, 0) {$\mapsto$};
	\end{pgfonlayer}
	\begin{pgfonlayer}{edgelayer}
		\draw [style=swap, in=-75, out=90, looseness=0.75] (1.center) to (0.center);
		\draw [style=swap, in=90, out=-15] (3) to (5.center);
		\draw [style=swap, in=90, out=-165] (3) to (4.center);
		\draw [style=swap, in=-90, out=90, looseness=0.75] (10.center) to (7.center);
		\draw [style=swap, in=90, out=-150, looseness=1.25] (8) to (9.center);
		\draw [style=swap, in=60, out=-45, looseness=1.75] (8) to (6.center);
		\draw [style=swap, dashed] (13.center) to (12.center);
		\draw [style=swap, dashed] (12.center) to (11.center);
		\draw [style=swap, dashed] (11.center) to (15.center);
		\draw [style=swap, dashed] (14.center) to (13.center);
		\draw [style=swap, dashed] (14.center) to (15.center);
		\draw [style=boldedge] (16.center) to (17.center);
		\draw [style=boldedge] (18.center) to (19.center);
		\draw [style=boldedge] (20.center) to (21.center);
		\draw [style=boldedge] (23.center) to (24.center);
		\draw (27.center) to (26.center);
		\draw (30.center) to (29.center);
		\draw (9.center) to (31.center);
		\draw (32.center) to (4.center);
		\draw [bend right=90, looseness=1.25] (26.center) to (32.center);
		\draw [bend left=90, looseness=1.25] (29.center) to (31.center);
	\end{pgfonlayer}
\end{tikzpicture}
\caption{
\textbf{We relate:} $\texttt{Alice } \underbrace{ \texttt{is bored by} }_{\text{passive voice}}  \texttt{ the class}$ {\bf to:} {\tt The class bores Alice}
}
\end{figure*}

\begin{figure*}
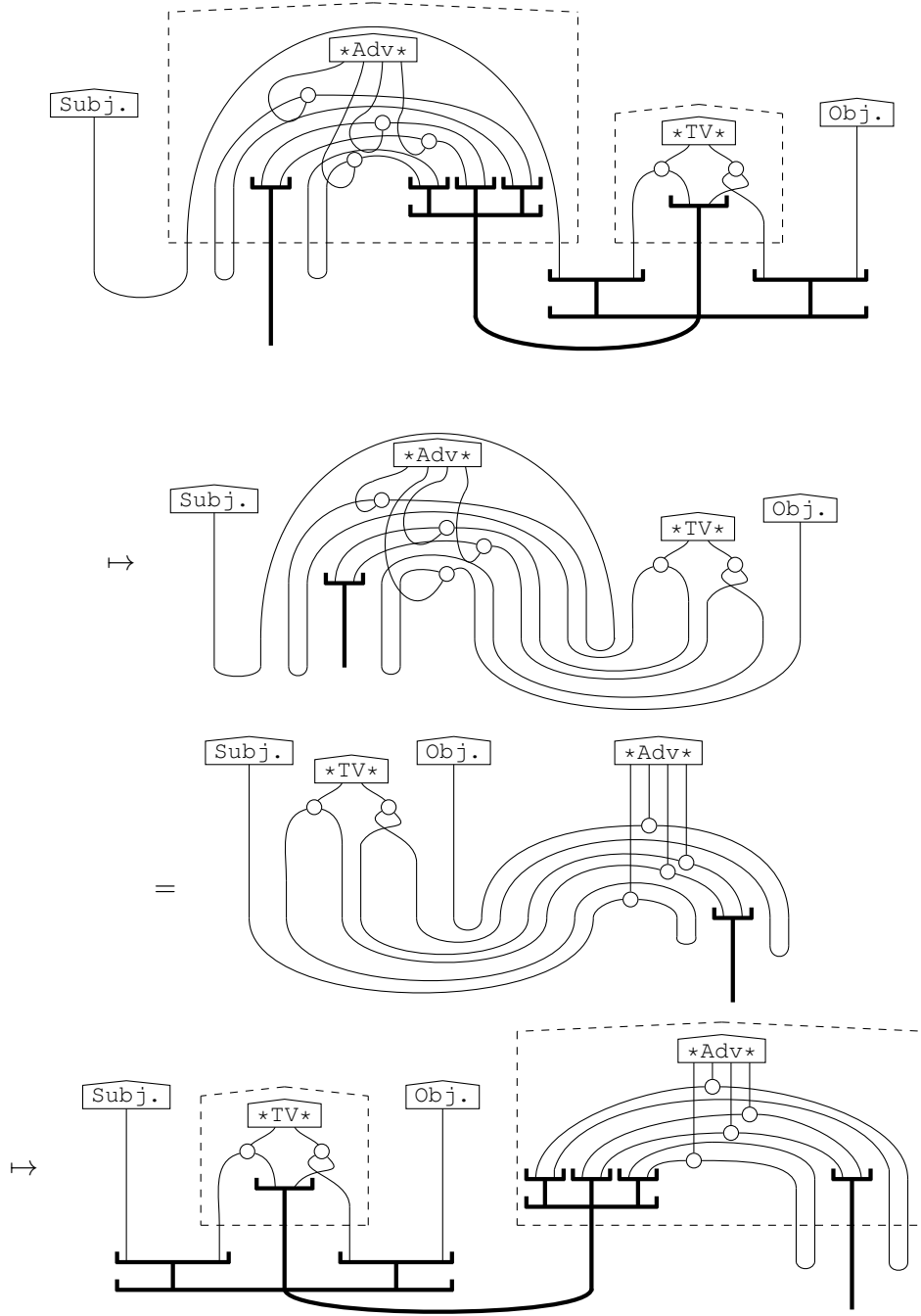

\centering
\input{vder-tvadvequivCOMPI.tikz}
\begin{tikzpicture}
	\begin{pgfonlayer}{nodelayer}
		\node [style=none] (0) at (-11, -1) {$\mapsto$};
		\node [style=langstate] (2) at (-8.5, 0.75) {\!\!\!\!{\tt Subj.}\!\!\!\!};
		\node [style=none] (3) at (-8.5, -3.75) {};
		\node [style=none] (4) at (-8.5, 0.5) {};
		\node [style=langstate] (5) at (7.25, 0.5) {\!\!\!\!{\tt Obj.}\!\!\!\!};
		\node [style=none] (6) at (7.25, -3) {};
		\node [style=none] (7) at (7.25, 0.25) {};
		\node [style=none] (8) at (4.25, -0.25) {};
		\node [style=none] (9) at (3.5, -1) {};
		\node [style=langstate] (10) at (4.5, 0) {\!\!\!\!{\tt *TV*}\!\!\!\!};
		\node [style=white dot] (11) at (3.5, -1) {};
		\node [style=none] (12) at (2.75, -2.75) {};
		\node [style=none] (13) at (4.25, -2) {};
		\node [style=none] (14) at (2.75, -3) {};
		\node [style=none] (15) at (4.75, -2) {};
		\node [style=none] (16) at (4.75, -0.25) {};
		\node [style=white dot] (17) at (5.5, -1) {};
		\node [style=none] (18) at (6.25, -2.5) {};
		\node [style=none] (19) at (5.5, -1) {};
		\node [style=none] (20) at (6.25, -3) {};
		\node [style=none] (21) at (4.75, -3) {};
		\node [style=none] (22) at (4.25, -3) {};
		\node [style=langstate] (23) at (-2.5, 2) {\!\!\!\!{\tt *Adv*}\!\!\!\!};
		\node [style=dot] (24) at (-4, 0.75) {};
		\node [style=none] (25) at (-2.5, 1.25) {};
		\node [style=none] (26) at (-4.25, 1.25) {};
		\node [style=none] (27) at (-7.25, -3.75) {};
		\node [style=none] (28) at (-7.25, -3) {};
		\node [style=none] (29) at (2.25, -3) {};
		\node [style=none] (30) at (2.25, -3) {};
		\node [style=dot] (31) at (-2.25, 0) {};
		\node [style=none] (32) at (1, -1.75) {};
		\node [style=none] (33) at (1.5, -1.75) {};
		\node [style=none] (34) at (-0.25, -1.75) {};
		\node [style=none] (35) at (0.25, -1.75) {};
		\node [style=none] (36) at (-1.5, -1.75) {};
		\node [style=none] (37) at (-1, -1.75) {};
		\node [style=none] (38) at (-6, -1.5) {};
		\node [style=none] (39) at (-6.5, -1.5) {};
		\node [style=none] (40) at (-5, -1.5) {};
		\node [style=none] (41) at (-5.5, -1.25) {};
		\node [style=none] (42) at (-5.5, -1.5) {};
		\node [style=none] (43) at (-4.5, -1.5) {};
		\node [style=none] (44) at (-5.5, -1.5) {};
		\node [style=none] (45) at (-4.5, -1.25) {};
		\node [style=none] (46) at (-4.5, -1.5) {};
		\node [style=none] (47) at (-5, -1.5) {};
		\node [style=none] (48) at (-5, -3.75) {};
		\node [style=none] (49) at (-6.5, -3.75) {};
		\node [style=none] (50) at (-6, -3.75) {};
		\node [style=none] (51) at (-4, -1.5) {};
		\node [style=none] (52) at (-3.5, -1.5) {};
		\node [style=none] (53) at (-4, -3.75) {};
		\node [style=none] (54) at (-3.5, -3.75) {};
		\node [style=none] (55) at (-5.25, -1.5) {};
		\node [style=none] (56) at (-4.75, -1.5) {};
		\node [style=none] (57) at (-3, 1.25) {};
		\node [style=none] (58) at (-1.75, 1) {};
		\node [style=dot] (59) at (-1.25, -0.5) {};
		\node [style=dot] (60) at (-2.25, -1.25) {};
		\node [style=none] (61) at (1.5, -3) {};
		\node [style=none] (62) at (1, -3) {};
		\node [style=none] (63) at (0.25, -3) {};
		\node [style=none] (64) at (-0.25, -3) {};
		\node [style=none] (65) at (-1, -3) {};
		\node [style=none] (66) at (-1.5, -3) {};
		\node [style=none] (67) at (-3.25, 1.75) {};
		\node [style=none] (68) at (-2.75, 1.75) {};
		\node [style=none] (69) at (-2.25, 1.75) {};
		\node [style=none] (70) at (-1.75, 1.75) {};
	\end{pgfonlayer}
	\begin{pgfonlayer}{edgelayer}
		\draw (4.center) to (3.center);
		\draw (7.center) to (6.center);
		\draw [style=swap, in=-75, out=90, looseness=0.75] (9.center) to (8.center);
		\draw [style=swap] (12.center) to (14.center);
		\draw [style=swap, in=90, out=-15] (11) to (13.center);
		\draw [style=swap, in=90, out=-165] (11) to (12.center);
		\draw [style=swap, in=-90, out=90, looseness=0.75] (19.center) to (16.center);
		\draw [style=swap, in=90, out=-150, looseness=1.25] (17) to (18.center);
		\draw [style=swap, in=60, out=-45, looseness=1.75] (17) to (15.center);
		\draw (18.center) to (20.center);
		\draw (28.center) to (27.center);
		\draw (29.center) to (30.center);
		\draw [bend left=90, looseness=2.00] (28.center) to (29.center);
		\draw [in=-165, out=-135, looseness=2.75] (31) to (25.center);
		\draw [in=-165, out=-135, looseness=3.00] (24) to (26.center);
		\draw [style=boldedge] (41.center) to (42.center);
		\draw [style=boldedge] (43.center) to (44.center);
		\draw [style=boldedge] (45.center) to (46.center);
		\draw [style=boldedge] (47.center) to (48.center);
		\draw (38.center) to (50.center);
		\draw (39.center) to (49.center);
		\draw (51.center) to (53.center);
		\draw (52.center) to (54.center);
		\draw [in=0, out=90] (33.center) to (24);
		\draw [in=90, out=180] (24) to (39.center);
		\draw [bend left=270] (32.center) to (38.center);
		\draw [in=0, out=90, looseness=1.25] (35.center) to (31);
		\draw [in=90, out=180] (31) to (55.center);
		\draw [in=30, out=90, looseness=1.25] (36.center) to (60);
		\draw [in=-150, out=-120, looseness=2.00] (60) to (57.center);
		\draw [in=90, out=165] (60) to (52.center);
		\draw [bend left=90] (51.center) to (37.center);
		\draw [in=0, out=90] (34.center) to (59);
		\draw [in=-120, out=-120, looseness=1.50] (59) to (58.center);
		\draw [in=90, out=165, looseness=0.75] (59) to (56.center);
		\draw (33.center) to (61.center);
		\draw (32.center) to (62.center);
		\draw (35.center) to (63.center);
		\draw (34.center) to (64.center);
		\draw (37.center) to (65.center);
		\draw (36.center) to (66.center);
		\draw [bend left=90, looseness=1.50] (30.center) to (61.center);
		\draw (15.center) to (21.center);
		\draw (13.center) to (22.center);
		\draw [bend left=90] (14.center) to (62.center);
		\draw [bend left=90, looseness=0.75] (22.center) to (63.center);
		\draw [bend left=90, looseness=0.75] (21.center) to (64.center);
		\draw [bend left=90, looseness=0.75] (20.center) to (65.center);
		\draw [bend left=90, looseness=0.75] (6.center) to (66.center);
		\draw [bend right=90, looseness=0.75] (3.center) to (27.center);
		\draw [bend left=90, looseness=1.50] (54.center) to (53.center);
		\draw [bend left=90, looseness=1.75] (50.center) to (49.center);
		\draw [in=-105, out=30, looseness=0.75] (26.center) to (67.center);
		\draw [in=-90, out=15] (57.center) to (68.center);
		\draw [in=-75, out=0] (25.center) to (69.center);
		\draw [in=60, out=-90] (70.center) to (58.center);
	\end{pgfonlayer}
\end{tikzpicture}  
\begin{tikzpicture}
	\begin{pgfonlayer}{nodelayer}
		\node [style=none] (0) at (-8, -0.75) {$=$};
		\node [style=langstate] (2) at (-5.75, 3) {\!\!\!\!{\tt Subj.}\!\!\!\!};
		\node [style=none] (3) at (-5.75, -1.5) {};
		\node [style=none] (4) at (-5.75, 2.75) {};
		\node [style=langstate] (5) at (-0.25, 3) {\!\!\!\!{\tt Obj.}\!\!\!\!};
		\node [style=none] (6) at (-0.25, -1.5) {};
		\node [style=none] (7) at (-0.25, 2.75) {};
		\node [style=none] (8) at (-3.25, 2.25) {};
		\node [style=none] (9) at (-4, 1.5) {};
		\node [style=langstate] (10) at (-3, 2.5) {\!\!\!\!{\tt *TV*}\!\!\!\!};
		\node [style=white dot] (11) at (-4, 1.5) {};
		\node [style=none] (12) at (-4.75, -0.5) {};
		\node [style=none] (13) at (-3.25, 0.5) {};
		\node [style=none] (14) at (-4.75, -1.5) {};
		\node [style=none] (15) at (-2.75, 0.5) {};
		\node [style=none] (16) at (-2.75, 2.25) {};
		\node [style=white dot] (17) at (-2, 1.5) {};
		\node [style=none] (18) at (-1.25, 0) {};
		\node [style=none] (19) at (-2, 1.5) {};
		\node [style=none] (20) at (-1.25, -1.5) {};
		\node [style=langstate] (21) at (5.25, 3) {\!\!\!\!{\tt *Adv*}\!\!\!\!};
		\node [style=none] (22) at (3, -1.5) {};
		\node [style=none] (23) at (3.5, -1.5) {};
		\node [style=none] (24) at (6.25, -1.5) {};
		\node [style=dot] (25) at (5.5, -0.25) {};
		\node [style=none] (26) at (5.75, -1.5) {};
		\node [style=dot] (27) at (6, 0) {};
		\node [style=none] (28) at (7.25, -1.5) {};
		\node [style=none] (29) at (6.75, -1.25) {};
		\node [style=none] (30) at (6.75, -1.5) {};
		\node [style=none] (31) at (7.75, -1.5) {};
		\node [style=none] (32) at (6.75, -1.5) {};
		\node [style=none] (33) at (7.75, -1.25) {};
		\node [style=none] (34) at (7.75, -1.5) {};
		\node [style=none] (35) at (7.25, -1.5) {};
		\node [style=none] (36) at (7.25, -3.75) {};
		\node [style=none] (37) at (5.75, -2) {};
		\node [style=none] (38) at (6.25, -2) {};
		\node [style=none] (39) at (4.5, 2.75) {};
		\node [style=none] (40) at (5.5, 2.75) {};
		\node [style=none] (41) at (1.75, -1.5) {};
		\node [style=none] (42) at (2.25, -1.5) {};
		\node [style=none] (43) at (0.5, -1.5) {};
		\node [style=none] (44) at (1, -1.5) {};
		\node [style=none] (45) at (5, 2.75) {};
		\node [style=none] (46) at (6, 2.75) {};
		\node [style=dot] (47) at (4.5, -1) {};
		\node [style=dot] (48) at (5, 1) {};
		\node [style=none] (49) at (8.25, -1.5) {};
		\node [style=none] (50) at (8.75, -1.5) {};
		\node [style=none] (51) at (8.25, -2.25) {};
		\node [style=none] (52) at (8.75, -2.25) {};
		\node [style=none] (53) at (7, -1.5) {};
		\node [style=none] (54) at (7.5, -1.5) {};
		\node [style=none] (55) at (-3.25, -1.5) {};
		\node [style=none] (56) at (-2.75, -1.5) {};
	\end{pgfonlayer}
	\begin{pgfonlayer}{edgelayer}
		\draw (4.center) to (3.center);
		\draw (7.center) to (6.center);
		\draw [style=swap, in=-75, out=90, looseness=0.75] (9.center) to (8.center);
		\draw [style=swap] (12.center) to (14.center);
		\draw [style=swap, in=90, out=-15] (11) to (13.center);
		\draw [style=swap, in=90, out=-165] (11) to (12.center);
		\draw [style=swap, in=-90, out=90, looseness=0.75] (19.center) to (16.center);
		\draw [style=swap, in=90, out=-150, looseness=1.25] (17) to (18.center);
		\draw [style=swap, in=60, out=-45, looseness=1.75] (17) to (15.center);
		\draw (18.center) to (20.center);
		\draw [style=boldedge] (29.center) to (30.center);
		\draw [style=boldedge] (31.center) to (32.center);
		\draw [style=boldedge] (33.center) to (34.center);
		\draw [style=boldedge] (35.center) to (36.center);
		\draw (24.center) to (38.center);
		\draw (26.center) to (37.center);
		\draw (49.center) to (51.center);
		\draw (50.center) to (52.center);
		\draw [in=-180, out=90, looseness=0.75] (23.center) to (47);
		\draw [in=90, out=0, looseness=0.75] (47) to (26.center);
		\draw [bend left=90] (22.center) to (24.center);
		\draw [in=-90, out=90, looseness=0.75] (47) to (39.center);
		\draw [in=165, out=90] (42.center) to (25);
		\draw [in=90, out=-15] (25) to (53.center);
		\draw [in=270, out=90, looseness=0.75] (25) to (40.center);
		\draw [in=-180, out=90] (43.center) to (48);
		\draw [bend left=90] (44.center) to (49.center);
		\draw [in=90, out=0] (48) to (50.center);
		\draw [in=-90, out=90, looseness=0.75] (48) to (45.center);
		\draw [in=165, out=90] (41.center) to (27);
		\draw [in=90, out=-15] (27) to (54.center);
		\draw [in=-90, out=90] (27) to (46.center);
		\draw (13.center) to (55.center);
		\draw (15.center) to (56.center);
		\draw [bend right=90, looseness=1.50] (6.center) to (43.center);
		\draw [bend right=90] (20.center) to (44.center);
		\draw [bend right=90, looseness=0.75] (56.center) to (41.center);
		\draw [bend right=90, looseness=0.75] (55.center) to (42.center);
		\draw [bend right=90, looseness=0.75] (14.center) to (22.center);
		\draw [bend right=90, looseness=0.75] (3.center) to (23.center);
		\draw [bend left=105, looseness=1.50] (38.center) to (37.center);
		\draw [bend left=90, looseness=2.00] (52.center) to (51.center);
	\end{pgfonlayer}
\end{tikzpicture}
\input{vder-tvadvequivCOMPV.tikz}
\caption{
\textbf{We relate:} {\tt Alice washes Fido gently} 
{\bf to:} {\tt Alice gently washes Fido}
}
\end{figure*}

\begin{figure*}
\centering
\input{vtype-posspronI.tikz}
\begin{tikzpicture}
	\begin{pgfonlayer}{nodelayer}
		\node [style=white dot] (0) at (-10, 1.5) {};
		\node [style=none] (1) at (-10, -1.5) {};
		\node [style=none] (2) at (-7.75, -1) {};
		\node [style=none] (3) at (-10, -5.5) {};
		\node [style=none] (4) at (-9, -0.5) {};
		\node [style=none] (5) at (-8.5, -0.5) {};
		\node [style=white dot] (6) at (-9, 0.25) {};
		\node [style=none] (7) at (-11, -1.5) {};
		\node [style=none] (8) at (-8.5, -0.5) {};
		\node [style=white dot] (9) at (-8.5, 0.25) {};
		\node [style=none] (10) at (-11, -2.25) {};
		\node [style=none] (11) at (4.75, 1.25) {};
		\node [style=none] (12) at (4.75, 0.5) {};
		\node [style=langstate] (13) at (5.75, 1.5) {\!\!\!\! \ {\tt *TV*} \ \!\!\!\!};
		\node [style=white dot] (14) at (4.75, 0.5) {};
		\node [style=none] (15) at (4.25, -1) {};
		\node [style=none] (16) at (5.5, -0.5) {};
		\node [style=none] (17) at (6, -0.5) {};
		\node [style=none] (18) at (6.75, 1.25) {};
		\node [style=white dot] (19) at (6.75, 0.5) {};
		\node [style=none] (20) at (7.25, -1) {};
		\node [style=none] (21) at (6.75, 0.5) {};
		\node [style=none] (22) at (9, -1) {};
		\node [style=none] (23) at (11, -1) {};
		\node [style=none] (24) at (9, 1.25) {};
		\node [style=none] (25) at (11, 0.75) {};
		\node [style=none] (26) at (14.5, -1.5) {};
		\node [style=none] (27) at (14.5, -2.75) {};
		\node [style=none] (28) at (12, -1.5) {};
		\node [style=none] (29) at (12, -2.25) {};
		\node [style=none] (30) at (12, 0.75) {};
		\node [style=none] (31) at (14.5, 1.25) {};
		\node [style=none] (32) at (13.25, -0.5) {};
		\node [style=none] (33) at (13.75, -0.5) {};
		\node [style=none] (34) at (10.25, -0.5) {};
		\node [style=none] (35) at (9.75, -0.5) {};
		\node [style=none] (36) at (13.25, -1) {};
		\node [style=none] (37) at (13.75, -1) {};
		\node [style=none] (38) at (-12, -2.25) {};
		\node [style=none] (39) at (-12, 3.25) {};
		\node [style=none] (40) at (-5.5, 1.25) {};
		\node [style=none] (41) at (-5.5, 0.5) {};
		\node [style=langstate] (42) at (-4.5, 1.5) {\!\!\!\! \ {\tt owns} \ \!\!\!\!};
		\node [style=white dot] (43) at (-5.5, 0.5) {};
		\node [style=none] (44) at (-6, -1) {};
		\node [style=none] (45) at (-4.75, -0.5) {};
		\node [style=none] (46) at (-7.75, -1) {};
		\node [style=none] (47) at (-4.25, -0.5) {};
		\node [style=none] (48) at (-3.5, 1.25) {};
		\node [style=white dot] (49) at (-3.5, 0.5) {};
		\node [style=none] (50) at (-3, -1) {};
		\node [style=none] (51) at (-3.5, 0.5) {};
		\node [style=none] (52) at (-3, -2.25) {};
		\node [style=none] (53) at (-2, -2.25) {};
		\node [style=none] (54) at (-2, 3.25) {};
		\node [style=none] (55) at (3.25, -2.25) {};
		\node [style=none] (56) at (3.25, 3.25) {};
		\node [style=white dot] (57) at (-0.25, 1.5) {};
		\node [style=none] (58) at (-0.25, -1.5) {};
		\node [style=none] (59) at (0.75, -1) {};
		\node [style=none] (60) at (0.75, -2.75) {};
		\node [style=none] (61) at (-0.25, -2.25) {};
		\node [style=none] (62) at (1.5, -1) {};
		\node [style=none] (63) at (2, -1) {};
		\node [style=white dot] (64) at (1.5, 0.25) {};
		\node [style=none] (65) at (-1, -1.5) {};
		\node [style=none] (66) at (2, -1) {};
		\node [style=white dot] (67) at (2, 0.25) {};
		\node [style=none] (68) at (-1, -2.25) {};
		\node [style=none] (69) at (-14.25, -0.5) {$\mapsto$};
		\node [style=none] (70) at (0.75, 5) {\text{(unwrapping wires)}};
		\node [style=langstate] (71) at (-12, 3.5) {\!\!\!\!{\tt possessor}\!\!\!\!};
		\node [style=langstate] (72) at (-2, 3.5) {\!\!\!\!{\tt possessed}\!\!\!\!};
		\node [style=langstate] (73) at (3.25, 3.5) {\!\!\!\!{\tt Obj.}\!\!\!\!};
	\end{pgfonlayer}
	\begin{pgfonlayer}{edgelayer}
		\draw [style=swap, in=90, out=0, looseness=1.25] (0) to (2.center);
		\draw [style=swap, in=90, out=-90, looseness=1.00] (0) to (1.center);
		\draw (3.center) to (1.center);
		\draw (4.center) to (6);
		\draw [in=180, out=90, looseness=1.00] (7.center) to (0);
		\draw (8.center) to (9);
		\draw (7.center) to (10.center);
		\draw [style=swap] (12.center) to (11.center);
		\draw [style=swap, in=90, out=-15, looseness=1.00] (14) to (16.center);
		\draw [style=swap, in=90, out=-165, looseness=1.00] (14) to (15.center);
		\draw [style=swap] (21.center) to (18.center);
		\draw [style=swap, in=90, out=-150, looseness=1.25] (19) to (20.center);
		\draw [style=swap, in=90, out=-45, looseness=1.50] (19) to (17.center);
		\draw (22.center) to (24.center);
		\draw (23.center) to (25.center);
		\draw (26.center) to (27.center);
		\draw (28.center) to (29.center);
		\draw (30.center) to (28.center);
		\draw (31.center) to (26.center);
		\draw [bend right=90, looseness=1.00] (32.center) to (34.center);
		\draw [bend right=90, looseness=1.00] (33.center) to (35.center);
		\draw (33.center) to (37.center);
		\draw (32.center) to (36.center);
		\draw [style=swap] (41.center) to (40.center);
		\draw [style=swap, bend left=90, looseness=1.25] (44.center) to (46.center);
		\draw [style=swap, in=90, out=-15, looseness=1.00] (43) to (45.center);
		\draw [style=swap, in=90, out=-165, looseness=1.00] (43) to (44.center);
		\draw [style=swap] (51.center) to (48.center);
		\draw [style=swap, in=90, out=-150, looseness=1.25] (49) to (50.center);
		\draw [style=swap, in=90, out=-45, looseness=1.50] (49) to (47.center);
		\draw (50.center) to (52.center);
		\draw (39.center) to (38.center);
		\draw (54.center) to (53.center);
		\draw (56.center) to (55.center);
		\draw [bend right=90, looseness=0.75] (55.center) to (29.center);
		\draw [style=swap, in=90, out=0, looseness=1.00] (57) to (59.center);
		\draw [style=swap, in=90, out=-90, looseness=1.00] (57) to (58.center);
		\draw (61.center) to (58.center);
		\draw (62.center) to (64);
		\draw [in=180, out=90, looseness=0.75] (65.center) to (57);
		\draw (66.center) to (67);
		\draw (65.center) to (68.center);
		\draw [bend right=90, looseness=1.25] (53.center) to (68.center);
		\draw [bend left=90, looseness=1.00] (61.center) to (52.center);
		\draw [bend right=90, looseness=1.00] (38.center) to (10.center);
		\draw [bend right=90, looseness=0.75] (60.center) to (27.center);
		\draw (59.center) to (60.center);
		\draw [bend left=90, looseness=1.00] (36.center) to (66.center);
		\draw [bend left=90, looseness=1.00] (37.center) to (62.center);
		\draw [bend right=90, looseness=1.25] (15.center) to (23.center);
		\draw [bend left=90, looseness=1.25] (22.center) to (20.center);
		\draw [bend left=90, looseness=1.25] (35.center) to (17.center);
		\draw [bend left=90, looseness=1.25] (34.center) to (16.center);
		\draw [bend left=90, looseness=1.25] (45.center) to (8.center);
		\draw [bend left=90, looseness=1.25] (47.center) to (4.center);
		\draw [in=90, out=90, looseness=1.25] (30.center) to (24.center);
		\draw [in=90, out=90, looseness=1.00] (25.center) to (31.center);
	\end{pgfonlayer}
\end{tikzpicture}
\begin{tikzpicture}
	\begin{pgfonlayer}{nodelayer}
		\node [style=white dot] (0) at (-6, -0.25) {};
		\node [style=none] (1) at (-5.5, -1.75) {};
		\node [style=none] (2) at (2.75, 2) {};
		\node [style=none] (3) at (2.75, 1.25) {};
		\node [style=langstate] (4) at (3.75, 2.25) {\!\!\!\! \ {\tt *TV*} \ \!\!\!\!};
		\node [style=white dot] (5) at (2.75, 1.25) {};
		\node [style=none] (6) at (2.25, -2.25) {};
		\node [style=none] (7) at (4.75, 2) {};
		\node [style=white dot] (8) at (4.75, 1.25) {};
		\node [style=none] (9) at (5.5, -0.75) {};
		\node [style=none] (10) at (4.75, 1.25) {};
		\node [style=none] (11) at (-7.5, -3.25) {};
		\node [style=none] (12) at (-4, 1.25) {};
		\node [style=none] (13) at (-2.25, 1.25) {};
		\node [style=none] (14) at (1, -1.75) {};
		\node [style=none] (15) at (6.75, -0.25) {};
		\node [style=white dot] (16) at (-1, 0.25) {};
		\node [style=none] (17) at (2.25, -2.25) {};
		\node [style=none] (18) at (-7.5, 3) {};
		\node [style=none] (19) at (1, 3) {};
		\node [style=none] (20) at (3.5, -1) {};
		\node [style=none] (21) at (4, -1.25) {};
		\node [style=white dot] (22) at (-2.5, -1.25) {};
		\node [style=white dot] (23) at (-3, -1.25) {};
		\node [style=none] (24) at (0, -1.75) {};
		\node [style=none] (25) at (0, -1.75) {};
		\node [style=none] (26) at (-6.5, -1.75) {};
		\node [style=none] (27) at (-6.5, -3.25) {};
		\node [style=none] (28) at (-3.75, -1.25) {};
		\node [style=none] (29) at (-4.75, -1.25) {};
		\node [style=none] (30) at (6.75, 3) {};
		\node [style=none] (31) at (-9.75, -0.5) {$=$};
		\node [style=none] (32) at (0, 4.5) {\text{(dragging wires into place)}};
		\node [style=langstate] (33) at (-7.5, 3.25) {\!\!\!\!{\tt possessor}\!\!\!\!};
		\node [style=langstate] (34) at (1, 3.25) {\!\!\!\!{\tt possessed}\!\!\!\!};
		\node [style=langstate] (35) at (6.75, 3.25) {\!\!\!\!{\tt Obj.}\!\!\!\!};
		\node [style=none] (36) at (-5.5, -5.25) {};
		\node [style=langstate] (37) at (-3.25, 1.5) {\!\!\!\! \ {\tt owns} \ \!\!\!\!};
	\end{pgfonlayer}
	\begin{pgfonlayer}{edgelayer}
		\draw [style=swap, in=90, out=-30] (0) to (1.center);
		\draw [style=swap] (3.center) to (2.center);
		\draw [style=swap, in=60, out=-165] (5) to (6.center);
		\draw [style=swap] (10.center) to (7.center);
		\draw [style=swap, in=105, out=-150, looseness=1.25] (8) to (9.center);
		\draw [in=-120, out=-150, looseness=1.50] (16) to (17.center);
		\draw [in=90, out=-90, looseness=0.75] (12.center) to (0);
		\draw [in=90, out=-90] (18.center) to (11.center);
		\draw [in=105, out=-90] (13.center) to (16);
		\draw (19.center) to (14.center);
		\draw [in=90, out=-30, looseness=1.25] (10.center) to (21.center);
		\draw [in=90, out=-15] (5) to (20.center);
		\draw [in=90, out=-30, looseness=1.25] (16) to (24.center);
		\draw [bend left=75] (14.center) to (25.center);
		\draw [in=90, out=-150] (0) to (26.center);
		\draw (26.center) to (27.center);
		\draw [bend right=90, looseness=1.25] (11.center) to (27.center);
		\draw [in=-90, out=-75, looseness=1.25] (22) to (20.center);
		\draw [in=-75, out=-90, looseness=1.25] (21.center) to (23);
		\draw [bend right=90, looseness=2.50] (28.center) to (29.center);
		\draw (30.center) to (15.center);
		\draw [bend left=90, looseness=1.25] (9.center) to (28.center);
		\draw [in=-90, out=-90, looseness=1.25] (29.center) to (15.center);
		\draw [style=swap] (1.center) to (36.center);
	\end{pgfonlayer}
\end{tikzpicture}
\begin{tikzpicture}
	\begin{pgfonlayer}{nodelayer}
		\node [style=none] (0) at (-5, 2) {};
		\node [style=langstate] (1) at (-4, 2.25) {\!\!\!\! \ {\tt owns} \ \!\!\!\!};
		\node [style=none] (2) at (-3, 2) {};
		\node [style=white dot] (3) at (-6.5, 0.75) {};
		\node [style=white dot] (4) at (-2, 0.75) {};
		\node [style=none] (5) at (-4, -1.25) {};
		\node [style=none] (6) at (-4, -1.5) {};
		\node [style=none] (7) at (-2, -1.5) {};
		\node [style=none] (8) at (-4, -1.5) {};
		\node [style=none] (9) at (-2, -1.25) {};
		\node [style=none] (10) at (-2, -1.5) {};
		\node [style=none] (11) at (-3, -1) {};
		\node [style=none] (12) at (-3, -2.25) {};
		\node [style=none] (13) at (-3, -1) {};
		\node [style=none] (14) at (-3.5, -0.75) {};
		\node [style=none] (15) at (-3.5, -1) {};
		\node [style=none] (16) at (-2.5, -1) {};
		\node [style=none] (17) at (-3.5, -1) {};
		\node [style=none] (18) at (-2.5, -0.75) {};
		\node [style=none] (19) at (-2.5, -1) {};
		\node [style=none] (20) at (-2.75, -1) {};
		\node [style=white dot] (21) at (-2.75, 0) {};
		\node [style=white dot] (22) at (-3.25, 0) {};
		\node [style=none] (23) at (-3.25, -1) {};
		\node [style=none] (24) at (-2.25, -1.5) {};
		\node [style=none] (25) at (-3.75, -1.5) {};
		\node [style=none] (26) at (-5, -2.25) {};
		\node [style=none] (27) at (-3.75, -0.5) {};
		\node [style=none] (28) at (-5, -0.5) {};
		\node [style=none] (29) at (-1, -2.25) {};
		\node [style=none] (30) at (-7, -2) {};
		\node [style=none] (31) at (-6, -2) {};
		\node [style=none] (32) at (-0.5, -2) {};
		\node [style=none] (33) at (-7.5, -2) {};
		\node [style=none] (34) at (-7.5, 3) {};
		\node [style=none] (35) at (-4, 3.25) {};
		\node [style=none] (36) at (-0.5, 3) {};
		\node [style=none] (37) at (-7, -2.25) {};
		\node [style=none] (38) at (-6, -2.25) {};
		\node [style=none] (39) at (-9, 3.75) {};
		\node [style=none] (40) at (-9, -2.25) {};
		\node [style=none] (41) at (-9, 3.75) {};
		\node [style=none] (42) at (1, -2.25) {};
		\node [style=none] (43) at (1, 3.75) {};
		\node [style=none] (44) at (8.5, -2.25) {};
		\node [style=none] (45) at (8.5, 3.75) {};
		\node [style=none] (46) at (3.75, 2) {};
		\node [style=none] (47) at (3.75, 1.25) {};
		\node [style=langstate] (48) at (4.75, 2.25) {\!\!\!\! \ {\tt *TV*} \ \!\!\!\!};
		\node [style=white dot] (49) at (3.75, 1.25) {};
		\node [style=none] (50) at (3, -0.75) {};
		\node [style=none] (51) at (4.5, 0.25) {};
		\node [style=none] (52) at (3, -1.75) {};
		\node [style=none] (53) at (5, 0.25) {};
		\node [style=none] (54) at (5.75, 2) {};
		\node [style=white dot] (55) at (5.75, 1.25) {};
		\node [style=none] (56) at (6.5, -0.25) {};
		\node [style=none] (57) at (5.75, 1.25) {};
		\node [style=none] (58) at (6.5, -1.75) {};
		\node [style=none] (59) at (7, -0.75) {};
		\node [style=none] (60) at (2.5, -0.75) {};
		\node [style=none] (61) at (2.5, 2.75) {};
		\node [style=none] (62) at (4.75, 3) {};
		\node [style=none] (63) at (7, 2.75) {};
		\node [style=none] (64) at (4, 0.5) {};
		\node [style=none] (65) at (4, 0.25) {};
		\node [style=none] (66) at (5.5, 0.25) {};
		\node [style=none] (67) at (4, 0.25) {};
		\node [style=none] (68) at (5.5, 0.5) {};
		\node [style=none] (69) at (5.5, 0.25) {};
		\node [style=none] (70) at (4.75, 0.25) {};
		\node [style=none] (71) at (4.75, -1.75) {};
		\node [style=none] (72) at (2.5, -1.5) {};
		\node [style=none] (73) at (2.5, -1.75) {};
		\node [style=none] (74) at (7, -1.75) {};
		\node [style=none] (75) at (2.5, -1.75) {};
		\node [style=none] (76) at (7, -1.5) {};
		\node [style=none] (77) at (7, -1.75) {};
		\node [style=none] (78) at (4.75, -1.75) {};
		\node [style=none] (79) at (4.75, -2.25) {};
		\node [style=none] (80) at (4.75, -2.25) {};
		\node [style=none] (81) at (-6, -4.5) {};
		\node [style=none] (82) at (4.75, -2.25) {};
		\node [style=none] (83) at (-3, -2.25) {};
		\node [style=none] (84) at (-4, 4.25) {\small\texttt{\underline{whose}}};
		\node [style=none] (85) at (0, 5.5) {\text{(recovering a pregroup proof with bracketing)}};
		\node [style=none] (86) at (-12, 0) {$\mapsto$};
		\node [style=langstate] (87) at (-9, 4) {\!\!\!\!{\tt possessor}\!\!\!\!};
		\node [style=langstate] (88) at (1, 4) {\!\!\!\!{\tt possessed}\!\!\!\!};
		\node [style=langstate] (89) at (8.5, 4) {\!\!\!\!{\tt Obj.}\!\!\!\!};
	\end{pgfonlayer}
	\begin{pgfonlayer}{edgelayer}
		\draw [in=90, out=-90] (2.center) to (4);
		\draw [in=90, out=-90, looseness=0.75] (0.center) to (3);
		\draw [style=boldedge] (5.center) to (6.center);
		\draw [style=boldedge] (7.center) to (8.center);
		\draw [style=boldedge] (9.center) to (10.center);
		\draw [style=boldedge] (11.center) to (12.center);
		\draw [style=boldedge] (14.center) to (15.center);
		\draw [style=boldedge] (16.center) to (17.center);
		\draw [style=boldedge] (18.center) to (19.center);
		\draw (21) to (20.center);
		\draw (22) to (23.center);
		\draw [in=-120, out=90, looseness=0.75] (24.center) to (4);
		\draw [in=-90, out=90, looseness=1.25] (25.center) to (27.center);
		\draw [bend right=90] (27.center) to (28.center);
		\draw (28.center) to (26.center);
		\draw [in=-15, out=90, looseness=0.75] (29.center) to (4);
		\draw [in=90, out=-45, looseness=0.75] (3) to (31.center);
		\draw [in=90, out=-135, looseness=0.75] (3) to (30.center);
		\draw [style=swap, dashed] (34.center) to (33.center);
		\draw [style=swap, dashed] (33.center) to (32.center);
		\draw [style=swap, dashed] (32.center) to (36.center);
		\draw [style=swap, dashed] (35.center) to (34.center);
		\draw [style=swap, dashed] (35.center) to (36.center);
		\draw (38.center) to (31.center);
		\draw (37.center) to (30.center);
		\draw [style=swap] (47.center) to (46.center);
		\draw [style=swap] (50.center) to (52.center);
		\draw [style=swap, in=90, out=-15] (49) to (51.center);
		\draw [style=swap, in=90, out=-165] (49) to (50.center);
		\draw [style=swap] (57.center) to (54.center);
		\draw [style=swap, in=90, out=-150, looseness=1.25] (55) to (56.center);
		\draw [style=swap, in=60, out=-45, looseness=1.75] (55) to (53.center);
		\draw [style=swap, dashed] (61.center) to (60.center);
		\draw [style=swap, dashed] (60.center) to (59.center);
		\draw [style=swap, dashed] (59.center) to (63.center);
		\draw [style=swap, dashed] (62.center) to (61.center);
		\draw [style=swap, dashed] (62.center) to (63.center);
		\draw [style=boldedge] (64.center) to (65.center);
		\draw [style=boldedge] (66.center) to (67.center);
		\draw [style=boldedge] (68.center) to (69.center);
		\draw [style=boldedge] (70.center) to (71.center);
		\draw (56.center) to (58.center);
		\draw [style=boldedge] (72.center) to (73.center);
		\draw [style=boldedge] (74.center) to (75.center);
		\draw [style=boldedge] (76.center) to (77.center);
		\draw [style=boldedge] (78.center) to (79.center);
		\draw (43.center) to (42.center);
		\draw (41.center) to (40.center);
		\draw [bend right=90, looseness=1.25] (40.center) to (37.center);
		\draw (38.center) to (81.center);
		\draw [bend right=90, looseness=0.75] (29.center) to (42.center);
		\draw [style=boldedge, bend left=90, looseness=0.50] (82.center) to (83.center);
		\draw [bend left=90, looseness=0.50] (44.center) to (26.center);
		\draw (45.center) to (44.center);
	\end{pgfonlayer}
\end{tikzpicture}

\caption{
\textbf{From:} {\tt author that owns book that John (was) entertain(s) -ed (by)} {\bf we derive a possessive relative pronoun:} {\tt author \underline{whose} book entertained John}
}
\end{figure*}

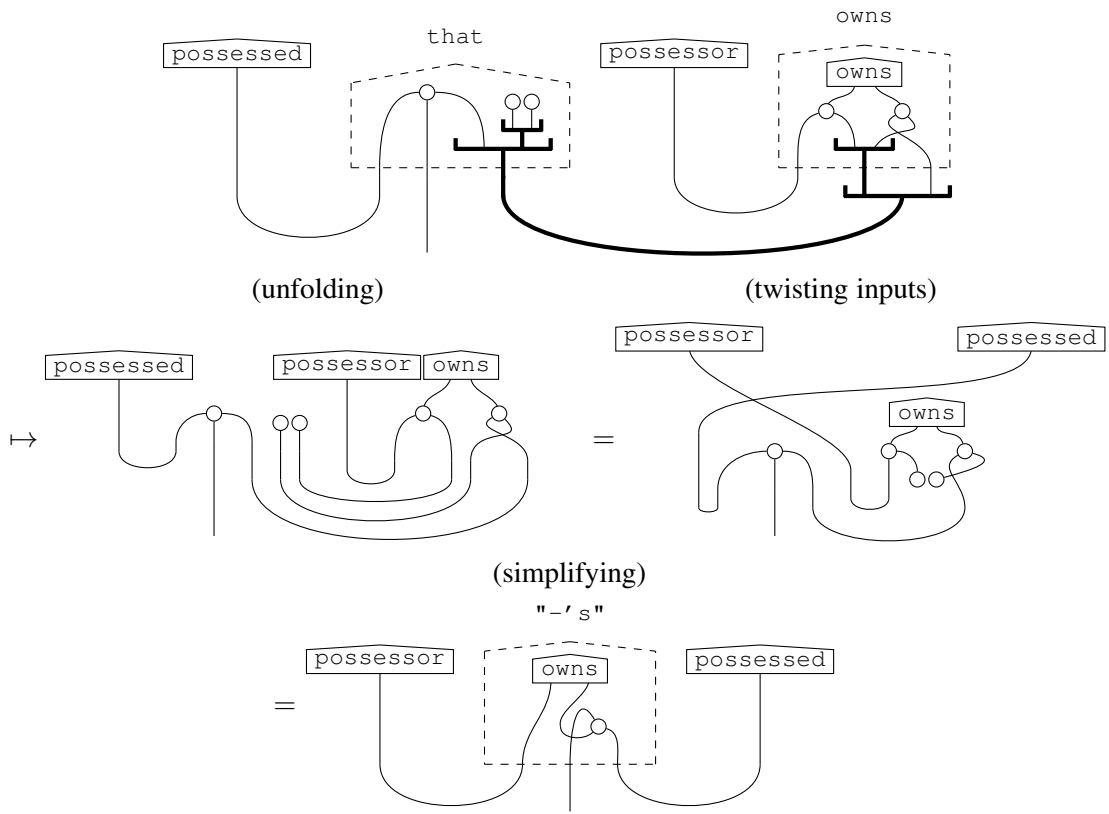
\begin{figure*}
\centering
\begin{tikzpicture}
	\begin{pgfonlayer}{nodelayer}
		\node [style=none] (0) at (-9.25, -1.25) {};
		\node [style=none] (1) at (-9.25, 2.25) {};
		\node [style=none] (2) at (2.25, -0.75) {};
		\node [style=none] (3) at (2.25, 2.25) {};
		\node [style=none] (4) at (7, 1.75) {};
		\node [style=none] (5) at (6.25, 1) {};
		\node [style=langstate] (6) at (7.25, 2) {\!\!\!\!{\tt owns}\!\!\!\!};
		\node [style=white dot] (7) at (6.25, 1) {};
		\node [style=none] (8) at (5.5, -0.5) {};
		\node [style=none] (9) at (7, 0) {};
		\node [style=none] (10) at (5.5, -0.75) {};
		\node [style=none] (11) at (7.5, 0) {};
		\node [style=none] (12) at (7.5, 1.75) {};
		\node [style=white dot] (13) at (8.25, 1) {};
		\node [style=none] (14) at (9, -0.5) {};
		\node [style=none] (15) at (8.25, 1) {};
		\node [style=none] (16) at (9, -1.25) {};
		\node [style=none] (17) at (9.5, -0.5) {};
		\node [style=none] (18) at (5, -0.5) {};
		\node [style=none] (19) at (5, 2.5) {};
		\node [style=none] (20) at (7.25, 2.75) {};
		\node [style=none] (21) at (9.5, 2.5) {};
		\node [style=none] (22) at (6.5, 0.25) {};
		\node [style=none] (23) at (6.5, 0) {};
		\node [style=none] (24) at (8, 0) {};
		\node [style=none] (25) at (6.5, 0) {};
		\node [style=none] (26) at (8, 0.25) {};
		\node [style=none] (27) at (8, 0) {};
		\node [style=none] (28) at (7.25, 0) {};
		\node [style=none] (29) at (7.25, -1.25) {};
		\node [style=none] (30) at (7.25, 3.5) {\small\tt owns};
		\node [style=none] (31) at (5.5, -0.75) {};
		\node [style=none] (32) at (2.25, -0.75) {};
		\node [style=none] (33) at (6.75, -1) {};
		\node [style=none] (34) at (6.75, -1.25) {};
		\node [style=none] (35) at (9.5, -1.25) {};
		\node [style=none] (36) at (6.75, -1.25) {};
		\node [style=none] (37) at (9.5, -1) {};
		\node [style=none] (38) at (9.5, -1.25) {};
		\node [style=none] (39) at (8.25, -1.25) {};
		\node [style=none] (40) at (8.25, -1.25) {};
		\node [style=white dot] (41) at (-4.25, 1.5) {};
		\node [style=none] (42) at (-4.25, -0.5) {};
		\node [style=none] (43) at (-2.75, 0) {};
		\node [style=none] (44) at (-0.5, -0.5) {};
		\node [style=none] (45) at (-6.25, -0.5) {};
		\node [style=none] (46) at (-6.25, 1.75) {};
		\node [style=none] (47) at (-3.5, 2.25) {};
		\node [style=none] (48) at (-0.5, 1.75) {};
		\node [style=none] (49) at (-3.5, 3) {\small\tt that};
		\node [style=none] (50) at (-3.5, 0.25) {};
		\node [style=none] (51) at (-3.5, 0) {};
		\node [style=none] (52) at (-1, 0) {};
		\node [style=none] (53) at (-3.5, 0) {};
		\node [style=none] (54) at (-1, 0.25) {};
		\node [style=none] (55) at (-1, 0) {};
		\node [style=none] (56) at (-2.25, 0) {};
		\node [style=none] (57) at (-2.25, -0.75) {};
		\node [style=none] (58) at (-2.75, 0) {};
		\node [style=none] (59) at (-1.75, 0) {};
		\node [style=none] (60) at (-2.25, 0) {};
		\node [style=none] (61) at (-1.75, 0.5) {};
		\node [style=none] (62) at (-1.75, 0) {};
		\node [style=none] (63) at (-4.25, -0.75) {};
		\node [style=none] (64) at (-2.25, 0.75) {};
		\node [style=none] (65) at (-2.25, 0.5) {};
		\node [style=none] (66) at (-1.25, 0.5) {};
		\node [style=none] (67) at (-2.25, 0.5) {};
		\node [style=none] (68) at (-1.25, 0.75) {};
		\node [style=none] (69) at (-1.25, 0.5) {};
		\node [style=none] (70) at (-2, 0.5) {};
		\node [style=none] (71) at (-1.5, 0.5) {};
		\node [style=white dot] (72) at (-2, 1.25) {};
		\node [style=none] (73) at (-5.5, -0.5) {};
		\node [style=none] (74) at (-1.5, 0.5) {};
		\node [style=white dot] (75) at (-1.5, 1.25) {};
		\node [style=none] (76) at (-5.5, -1.25) {};
		\node [style=none] (77) at (-5.5, -1.25) {};
		\node [style=none] (78) at (-9.25, -1.25) {};
		\node [style=none] (79) at (8.25, -1.25) {};
		\node [style=none] (80) at (-2.25, -1.25) {};
		\node [style=none] (81) at (-2.25, -0.75) {};
		\node [style=none] (82) at (-2.25, -1.25) {};
		\node [style=none] (83) at (-4.25, -2.75) {};
		\node [style=none] (84) at (-4.25, -0.75) {};
		\node [style=langstate] (86) at (2.25, 2.5) {\!\!\!\!{\tt possessor}\!\!\!\!};
		\node [style=langstate] (87) at (-9.25, 2.5) {\!\!\!\!{\tt possessed}\!\!\!\!};
		\node [style=none] (88) at (-9.25, -1.25) {};
		\node [style=none] (89) at (2.25, -0.75) {};
	\end{pgfonlayer}
	\begin{pgfonlayer}{edgelayer}
		\draw [style=swap, in=-90, out=90, looseness=1.25] (5.center) to (4.center);
		\draw [style=swap] (8.center) to (10.center);
		\draw [style=swap, in=90, out=-15, looseness=1.00] (7) to (9.center);
		\draw [style=swap, in=90, out=-165, looseness=1.00] (7) to (8.center);
		\draw [style=swap, in=-75, out=90, looseness=1.00] (15.center) to (12.center);
		\draw [style=swap, in=90, out=-150, looseness=1.25] (13) to (14.center);
		\draw [style=swap, in=60, out=-45, looseness=1.75] (13) to (11.center);
		\draw [style=swap, dashed] (19.center) to (18.center);
		\draw [style=swap, dashed] (18.center) to (17.center);
		\draw [style=swap, dashed] (17.center) to (21.center);
		\draw [style=swap, dashed] (20.center) to (19.center);
		\draw [style=swap, dashed] (20.center) to (21.center);
		\draw [style=boldedge] (22.center) to (23.center);
		\draw [style=boldedge] (24.center) to (25.center);
		\draw [style=boldedge] (26.center) to (27.center);
		\draw [style=boldedge] (28.center) to (29.center);
		\draw (14.center) to (16.center);
		\draw [style=boldedge] (33.center) to (34.center);
		\draw [style=boldedge] (35.center) to (36.center);
		\draw [style=boldedge] (37.center) to (38.center);
		\draw [style=swap, in=90, out=0, looseness=1.25] (41) to (43.center);
		\draw [style=swap, in=90, out=-90, looseness=1.00] (41) to (42.center);
		\draw [style=swap, dashed] (46.center) to (45.center);
		\draw [style=swap, dashed] (45.center) to (44.center);
		\draw [style=swap, dashed] (44.center) to (48.center);
		\draw [style=swap, dashed] (47.center) to (46.center);
		\draw [style=swap, dashed] (47.center) to (48.center);
		\draw [style=boldedge] (50.center) to (51.center);
		\draw [style=boldedge] (52.center) to (53.center);
		\draw [style=boldedge] (54.center) to (55.center);
		\draw [style=boldedge] (56.center) to (57.center);
		\draw [style=boldedge] (61.center) to (62.center);
		\draw (63.center) to (42.center);
		\draw [style=boldedge] (64.center) to (65.center);
		\draw [style=boldedge] (66.center) to (67.center);
		\draw [style=boldedge] (68.center) to (69.center);
		\draw (70.center) to (72);
		\draw [in=180, out=90, looseness=1.00] (73.center) to (41);
		\draw (74.center) to (75);
		\draw (73.center) to (76.center);
		\draw [style=boldedge, bend left=90, looseness=0.50] (79.center) to (80.center);
		\draw [style=boldedge] (81.center) to (82.center);
		\draw (84.center) to (83.center);
		\draw [bend left=90, looseness=1.00] (77.center) to (78.center);
		\draw (88.center) to (1.center);
		\draw (89.center) to (3.center);
		\draw [bend left=90, looseness=1.00] (31.center) to (89.center);
	\end{pgfonlayer}
\end{tikzpicture}
\begin{tikzpicture}
	\begin{pgfonlayer}{nodelayer}
		\node [style=none] (0) at (-11.5, -0.5) {};
		\node [style=none] (1) at (-11.5, 1.5) {};
		\node [style=none] (2) at (-5.5, -1) {};
		\node [style=none] (3) at (-5.5, 1.5) {};
		\node [style=none] (4) at (-2.75, 1.5) {};
		\node [style=none] (5) at (-3.5, 0.5) {};
		\node [style=langstate] (6) at (-2.5, 1.75) {\!\!\!\!{\tt owns}\!\!\!\!};
		\node [style=white dot] (7) at (-3.5, 0.5) {};
		\node [style=none] (8) at (-4.25, -0.75) {};
		\node [style=none] (9) at (-2.75, -0.75) {};
		\node [style=none] (10) at (-4.25, -1) {};
		\node [style=none] (11) at (-2.25, -0.75) {};
		\node [style=none] (12) at (-2.25, 1.5) {};
		\node [style=white dot] (13) at (-1.5, 0.5) {};
		\node [style=none] (14) at (-0.75, -0.75) {};
		\node [style=none] (15) at (-1.5, 0.5) {};
		\node [style=none] (16) at (-0.75, -1.25) {};
		\node [style=none] (17) at (-4.25, -1) {};
		\node [style=none] (18) at (-5.5, -1) {};
		\node [style=white dot] (19) at (-9, 0.5) {};
		\node [style=none] (20) at (-9, -0.25) {};
		\node [style=none] (21) at (-8, -0.25) {};
		\node [style=none] (22) at (-8, -0.25) {};
		\node [style=none] (23) at (-9, -2.75) {};
		\node [style=none] (24) at (-7.25, -0.5) {};
		\node [style=none] (25) at (-6.75, -0.5) {};
		\node [style=white dot] (26) at (-7.25, 0.25) {};
		\node [style=none] (27) at (-10, -0.25) {};
		\node [style=none] (28) at (-6.75, -0.5) {};
		\node [style=white dot] (29) at (-6.75, 0.25) {};
		\node [style=none] (30) at (-10, -0.5) {};
		\node [style=none] (31) at (-10, -0.5) {};
		\node [style=none] (32) at (-11.5, -0.5) {};
		\node [style=none] (33) at (-8, -1.25) {};
		\node [style=none] (34) at (-7.25, -1.25) {};
		\node [style=none] (35) at (-2.25, -1.25) {};
		\node [style=none] (36) at (-2.75, -1.25) {};
		\node [style=none] (37) at (-6.75, -1.25) {};
		\node [style=none] (38) at (3.75, -2) {};
		\node [style=none] (39) at (3.75, 0) {};
		\node [style=none] (40) at (7.75, -1.75) {};
		\node [style=none] (41) at (9.5, 0.25) {};
		\node [style=langstate] (42) at (9.75, 0.5) {\!\!\!\!{\tt owns}\!\!\!\!};
		\node [style=none] (43) at (8.75, -1) {};
		\node [style=none] (44) at (8.75, -1.75) {};
		\node [style=none] (45) at (10, 0.25) {};
		\node [style=none] (46) at (10.75, -2) {};
		\node [style=none] (47) at (10.75, -2) {};
		\node [style=none] (48) at (8.75, -1.75) {};
		\node [style=none] (49) at (7.75, -1.75) {};
		\node [style=white dot] (50) at (5.75, -0.5) {};
		\node [style=none] (51) at (5.75, -1.5) {};
		\node [style=none] (52) at (6.75, -1.5) {};
		\node [style=none] (53) at (6.75, -1.5) {};
		\node [style=none] (54) at (5.75, -2.75) {};
		\node [style=none] (55) at (4.25, -1.5) {};
		\node [style=none] (56) at (4.25, -2) {};
		\node [style=none] (57) at (4.25, -2) {};
		\node [style=none] (58) at (6.75, -2) {};
		\node [style=white dot] (59) at (10, -1.25) {};
		\node [style=white dot] (60) at (9.5, -1.25) {};
		\node [style=white dot] (61) at (8.75, -0.5) {};
		\node [style=white dot] (62) at (10.75, -0.5) {};
		\node [style=none] (63) at (3.75, 0) {};
		\node [style=none] (64) at (12.5, 2.25) {};
		\node [style=none] (65) at (7.75, -1) {};
		\node [style=none] (66) at (3.5, 2.25) {};
		\node [style=none] (67) at (-14, -0.25) {$\mapsto$};
		\node [style=none] (68) at (1.25, -0.25) {$=$};
		\node [style=none] (69) at (-6.25, 3.75) {\text{(unfolding)}};
		\node [style=none] (70) at (7.5, 3.75) {\text{(twisting inputs)}};
		\node [style=langstate] (71) at (-11.5, 1.75) {\!\!\!\!{\tt possessed}\!\!\!\!};
		\node [style=langstate] (72) at (3.5, 2.5) {\!\!\!\!{\tt possessor}\!\!\!\!};
		\node [style=langstate] (73) at (-5.5, 1.75) {\!\!\!\!{\tt possessor}\!\!\!\!};
		\node [style=langstate] (74) at (12.5, 2.5) {\!\!\!\!{\tt possessed}\!\!\!\!};
	\end{pgfonlayer}
	\begin{pgfonlayer}{edgelayer}
		\draw [style=swap, in=-105, out=90, looseness=1.25] (5.center) to (4.center);
		\draw [style=swap] (8.center) to (10.center);
		\draw [style=swap, in=90, out=-15] (7) to (9.center);
		\draw [style=swap, in=90, out=-165] (7) to (8.center);
		\draw [style=swap, in=-90, out=90] (15.center) to (12.center);
		\draw [style=swap, in=90, out=-150, looseness=1.25] (13) to (14.center);
		\draw [style=swap, in=90, out=-45, looseness=1.75] (13) to (11.center);
		\draw (14.center) to (16.center);
		\draw [style=swap, in=90, out=0, looseness=1.25] (19) to (21.center);
		\draw [style=swap, in=90, out=-90] (19) to (20.center);
		\draw (23.center) to (20.center);
		\draw (24.center) to (26);
		\draw [in=180, out=90] (27.center) to (19);
		\draw (28.center) to (29);
		\draw (27.center) to (30.center);
		\draw [bend left=90, looseness=0.75] (16.center) to (33.center);
		\draw (33.center) to (22.center);
		\draw (11.center) to (35.center);
		\draw (9.center) to (36.center);
		\draw (24.center) to (34.center);
		\draw (28.center) to (37.center);
		\draw [bend left=90, looseness=0.50] (36.center) to (37.center);
		\draw [bend left=90, looseness=0.75] (35.center) to (34.center);
		\draw [style=swap] (43.center) to (44.center);
		\draw [style=swap, in=90, out=0, looseness=1.25] (50) to (52.center);
		\draw [style=swap, in=90, out=-90] (50) to (51.center);
		\draw (54.center) to (51.center);
		\draw [in=180, out=90] (55.center) to (50);
		\draw (55.center) to (56.center);
		\draw [bend left=90, looseness=0.75] (47.center) to (58.center);
		\draw (58.center) to (53.center);
		\draw [in=75, out=-75] (41.center) to (61);
		\draw [in=90, out=0] (61) to (60);
		\draw [in=105, out=-90, looseness=1.25] (45.center) to (62);
		\draw [in=-15, out=15, looseness=2.00] (59) to (62);
		\draw [in=90, out=-150, looseness=1.25] (62) to (46.center);
		\draw (61) to (43.center);
		\draw (65.center) to (49.center);
		\draw (1.center) to (32.center);
		\draw [bend right=90] (32.center) to (31.center);
		\draw (3.center) to (18.center);
		\draw [in=90, out=-90, looseness=0.50] (66.center) to (65.center);
		\draw [in=-90, out=90, looseness=0.50] (63.center) to (64.center);
		\draw [bend right=90, looseness=0.75] (38.center) to (57.center);
		\draw (63.center) to (38.center);
		\draw [bend right=90] (49.center) to (48.center);
		\draw [bend right=90] (18.center) to (17.center);
	\end{pgfonlayer}
\end{tikzpicture}
\begin{tikzpicture}
	\begin{pgfonlayer}{nodelayer}
		\node [style=none] (0) at (-0.25, -0.75) {};
		\node [style=langstate] (2) at (0, 0.75) {\!\!\!\!{\tt owns}\!\!\!\!};
		\node [style=white dot] (3) at (0.75, -0.75) {};
		\node [style=none] (4) at (0, -1.75) {};
		\node [style=none] (5) at (1.25, -1.75) {};
		\node [style=none] (6) at (0, -3) {};
		\node [style=none] (7) at (2.25, -1.75) {};
		\node [style=none] (8) at (-2.25, -1.75) {};
		\node [style=none] (9) at (-2.25, 1.25) {};
		\node [style=none] (10) at (0, 1.5) {};
		\node [style=none] (11) at (2.25, 1.25) {};
		\node [style=none] (13) at (1.25, -1.75) {};
		\node [style=none] (14) at (0, 2.25) {\small\tt "-'s"};
		\node [style=none] (15) at (-0.5, 0.5) {};
		\node [style=none] (16) at (-1.25, -1.75) {};
		\node [style=none] (17) at (-1.25, -1.75) {};
		\node [style=none] (18) at (-5, -1.75) {};
		\node [style=none] (19) at (-1.25, -1.75) {};
		\node [style=none] (20) at (-5, -1.75) {};
		\node [style=none] (21) at (-5, 0.75) {};
		\node [style=none] (22) at (5, -1.75) {};
		\node [style=none] (23) at (1.25, -1.75) {};
		\node [style=none] (24) at (5, -1.75) {};
		\node [style=none] (25) at (5, 0.75) {};
		\node [style=none] (26) at (-7.5, -0.25) {$=$};
		\node [style=langstate] (27) at (5, 1) {\!\!\!\!{\tt possessed}\!\!\!\!};
		\node [style=langstate] (28) at (-5, 1) {\!\!\!\!{\tt possessor}\!\!\!\!};
		\node [style=none] (29) at (0, 3.25) {\text{(simplifying)}};
		\node [style=none] (30) at (0.5, 0.5) {};
	\end{pgfonlayer}
	\begin{pgfonlayer}{edgelayer}
		\draw [style=swap] (4.center) to (6.center);
		\draw [style=swap, in=90, out=-15] (3) to (5.center);
		\draw [style=swap, dashed] (9.center) to (8.center);
		\draw [style=swap, dashed] (8.center) to (7.center);
		\draw [style=swap, dashed] (7.center) to (11.center);
		\draw [style=swap, dashed] (10.center) to (9.center);
		\draw [style=swap, dashed] (10.center) to (11.center);
		\draw [in=90, out=-90] (15.center) to (16.center);
		\draw (21.center) to (20.center);
		\draw [bend right=90] (20.center) to (19.center);
		\draw [bend right=90] (23.center) to (24.center);
		\draw (25.center) to (24.center);
		\draw [in=-90, out=-135] (3) to (0.center);
		\draw [in=90, out=120, looseness=1.75] (3) to (4.center);
		\draw [in=-90, out=105, looseness=0.75] (0.center) to (30.center);
	\end{pgfonlayer}
\end{tikzpicture}
\caption{
\textbf{From:} {\tt (possessed) that (possessor) owns} {\bf we derive the possessive modifier:} {\tt (possessor) \underline{'s} (possessed)} 
}
\end{figure*}

\end{document}